\documentclass{article}
\usepackage{fullpage,cite}
\usepackage{mysty,bm}
\usepackage[capitalise]{cleveref}
\usepackage{amssymb}

\newcommand{\argmin}{\operatorname{argmin}}
\newcommand{\rank}{\operatorname{rank}}

\newcommand{\mcal}{\mathcal}
\newcommand{\scr}{\mathscr}
\newcommand{\bb}{\mathbb}

\newcommand{\mbf}{\mathbf}
\newcommand{\bnm}{\big\|}
\newcommand{\tsp}{\widetilde{\scr{P}}_{\mathbb{T}_l}}
\newcommand{\wtd}{\widetilde}
\newcommand{\wht}{\widehat}

\graphicspath{{./figure/}}

\begin{document}

\title{Fast and Provable Tensor-Train Format Tensor Completion via Precondtioned Riemannian Gradient Descent}
\author{
Fengmiao Bian\thanks{Department of Mathematics, The Hong Kong University of Science and Technology, Hong Kong, China.
 E-mail: mafmbian@ust.hk} \and
		Jian-Feng Cai\thanks{Department of Mathematics, The Hong Kong University of Science and Technology, Hong Kong, China. E-mail: jfcai@ust.hk} \and
		Xiaoqun Zhang\thanks{School of Mathematical Sciences and Institute of Natural Sciences, Shanghai Jiao Tong University, Shanghai China. E-mail: xqzhang@sjtu.edu.cn} \and
  Yuanwei Zhang\thanks{School of Mathematical Sciences, Shanghai Jiao Tong University, Shanghai China. E-mail: sjtuzyw@sjtu.edu.cn}
		}
\date{}
\maketitle

\begin{abstract}
Low-rank tensor completion aims to recover a tensor from partially observed entries, and it is widely applicable in fields such as quantum computing and image processing. Due to the significant advantages of the tensor train (TT) format in handling structured high-order tensors, this paper investigates the low-rank tensor completion problem based on the TT-format. We proposed a preconditioned Riemannian gradient descent algorithm (PRGD) to solve low TT-rank tensor completion and establish its linear convergence. Experimental results on both simulated and real datasets demonstrate the effectiveness of the PRGD algorithm. On the simulated dataset, the PRGD algorithm reduced the computation time by two orders of magnitude compared to existing classical algorithms. In practical applications such as hyperspectral image completion and quantum state tomography, the PRGD algorithm significantly reduced the number of iterations, thereby substantially reducing the computational time.
\end{abstract}

\begin{keywords}
Low-rank tensor completion, tensor train decomposition, preconditioned Riemannian gradient descent, quantum state tomography, hyperspectral image completion
\end{keywords}

\begin{AMS}
{15A23 $\cdot$ 15A83 $\cdot$ 68Q25 $\cdot$ 90C26 $\cdot$ 90C53}
\end{AMS}

\section{Introduction}
Tensors, with their inherent ability to express complex interactions within high-dimensional data, have become indispensable tools in the fields of data science and engineering. They are widely used, for instance, in recommendation systems (such as movie ratings, user-product-time tensors \cite{HK, DGGG}), image processing (such as hyperspectral images, wavelength-image tensors \cite{foster2022colour}), and financial mathematics (such as stock prices, time scale-price tensors \cite{GKS20}). However, the storage and acquisition of high-dimensional tensor data present significant challenges. Consequently, a core issue is how to accurately recover the target tensor from incomplete observational data. Recovering a target tensor directly from only a subset of its entries, without imposing any structural assumptions, is an ill-posed inverse problem.  To address this challenge, a low-rank structure is commonly assumed, effectively constraining the data to lie in an underlying low-dimensional space, rendering the low-rank tensor completion problem tractable. There are many existing algorithms for solving tensor completion problems. Formally, tensor completion is merely an extension of matrix completion, but the multilinear nature of tensors introduces significant computational challenges. In addressing low-rank matrix completion, since directly solving the non-convex problem of finding the lowest rank is NP-hard, a classic approach is to relax the rank into the nuclear norm of the matrix. However, this convex relaxation method is not feasible for low-rank tensor completion, primarily because the nuclear norm of tensors, although also a convex function, is usually NP-hard to compute \cite{HL2013}. This computational difficulty exists in many tensor-related convex functions, such as tensor operator norms. Consequently, convex relaxation techniques for tensor completion are often computationally infeasible. Even so, some advanced results from matrix completion can be applied to tensor completion by disregarding the inherent tensor structure and reformulating the tensor completion problem as a matrix completion problem. This allows for the application of matrix completion algorithms, such as \cite{CR2009, CT2010, CCS2010, R2011, XY2021, CFMY, CCZ, SL2016}. While these methods are easy to implement, they result in high computational costs due to the transformation of tensors into large-sized matrices. Therefore, the effective approach to tensor completion is to directly exploit the low-rank tensor structure itself.

Compared to dealing with full-size tensors, tensor decomposition \cite{KB} allows us to exploit the low-rank structure inherent in tensor completion problems, thereby reducing the number of parameters in the search space and saving storage space. The rank of a tensor depends on the specific tensor decomposition employed, with common decompositions including CP decomposition \cite{H1927}, Tucker decomposition \cite{T1966}, and Tensor Train (TT) decomposition \cite{oseledets2011tensor}.  TT-format decomposition first appeared in pioneering works in physics, such as quantum dynamics for very large systems \cite{V03, V04, GVWC}. The TT-format decomposition combines the advantages of CP and Tucker decompositions, allowing for more effective handling of higher-order tensors, especially in representing matrix product states in quantum computing.  Moreover, compared to Tucker decomposition, the model complexity of TT-format decomposition grows linearly with the order of the tensor, offering higher storage and computational efficiency. The concept of TT-rank was also introduced by \cite{oseledets2011tensor}. Therefore, this paper focuses on the low-rank tensor completion problem based on TT-format decomposition. Within this framework, the low-rank tensor completion problem can be formulated as follows:
\begin{equation}
\label{model: Low TT}
\min_{\mcal{T}\in \bb{R}^{d_1\times\cdots \times d_m}} f (\mcal{T}) := \frac{1}{2} \Big\langle \mcal{T} - \mcal{T}^{*}, \scr{P}_{\Omega}(\mcal{T} - \mcal{T}^*) \Big\rangle ~~ \text{s.t.} ~~ \textmd{rank}_{tt} (\mcal{T}) = \mbf{r},
\end{equation}
where $\mcal{T}^*$ is the target tensor to be recovered, $\Omega$ is a subset of indices for observed entries, $\textmd{rank}_{tt}(\mcal{T})$ refers to the TT-rank of $\mcal{T}$ which will be introduced later, and  $\scr{P}_{\Omega}$ is the sampling operator defined by 
$$
\scr{P}_{\Omega}(\mcal{T})(i_1, i_2, \dots, i_m)=
\begin{cases}
\mcal{T}(i_1, i_2, \dots, i_n), ~& \textmd{if}~(i_1, i_2, \dots, i_n) \in \Omega,\\
0, ~&\textmd{otherwise.}
\end{cases}
$$
Currently, some convex optimization algorithms are employed to solve low TT-rank tensor completion problems. For instance, \cite{BPTD} utilizes the nuclear norm based on unfolding matrices and demonstrated the outstanding performance of matrix decomposition methods in image restoration problems. Based on this work, \cite{DHJZY} further enhances the effects on image and video restoration by applying total variation regularization. Beyond convex optimization methods, some non-convex algorithms, particularly those involving tensor decomposition, are also used to solve low TT-rank tensor completion problems. For example, \cite{WAS} explores low TT-rank tensor completion problems by combining TT-SVD spectral initialization with alternating minimization techniques. \cite{YZGC} adopts a gradient descent algorithm with random initialization to solve tensor completion problems in TT-format decomposition. \cite{KBDYW} introduces an efficient heuristic initialization method that can solve image and video restoration problems very effectively. However, these works mainly focus on algorithm design and have limited or no theoretical foundation.  Notably, \cite{IMH} proposes a novel relaxation method based on the Schatten norm, providing theoretical guarantees. Additionally, \cite{ZZZW} employs a fast higher-order orthogonal iteration algorithm, determining the statistically optimal convergence rate for TT-format tensor SVD.

Since tensors with a fixed TT rank form a manifold \cite{holtz2012manifolds}, many manifold-based optimization algorithms, particularly those based on Riemannian optimization, can be easily adapted to the low TT-rank tensor completion. Unlike the projected gradient descent algorithm, Riemannian optimization algorithms consider the target tensor as a point on a Riemannian manifold and employ Riemannian gradient descent (RGD) to minimize the objective function.  The application of Riemannian gradient descent algorithms on matrices and tensors has been extensively studied, see \cite{S2016, WCCL2016, WCCL2020, KSV, CLX2022}. Similarly, the low TT-rank tensor completion problem can be reformulated as an unconstrained problem on a Riemannian manifold and solved using the RGD algorithm. Recent \cite{cai2022tensor} have shown that if the initialization is sufficiently good, RGD can achieve linear convergence in low TT-rank tensor completion.  Moreover, when the sample size is nearly optimal, the iteration error produced by the RGD algorithm contracts at a rate less than one. This indicates that RGD is one of the most effective algorithms for solving low TT-rank tensor completion problems. However, as \cite{bian2023preconditioned} illustrates, Riemannian Gradient Descent (RGD) initially requires linearizing the objective function $f(\mcal{T})$ in the tangent space, thus the convergence speed depends on the condition number of the sensing operator in the tangent space, i.e., the effectiveness of RGD highly depends on the metric used on the tangent space. The Preconditioned Riemannian Gradient Descent (PRGD) algorithm \cite{bian2023preconditioned} significantly enhances the performance of RGD by by replacing the standard metric with a simpler, adaptive metric, making it up to ten times faster than RGD in some large-scale low-rank completion problems. Inspired by this result, one naturally wonders whether this preconditioned method could also be applicable to tensor completion problems.

This article effectively answers this question by choosing an asuitable metric for the manifold composed of tensors with fixed TT-rank, which improves the performance of the RGD algorithm. We chose an adaptive, data-driven metric on the tangent space for the RGD algorithm, resulting in a new preconditioned Riemannian gradient descent (PRGD) algorithm. To validate the efficiency of the PRGD algorithm, we conducted numerical experiments on both synthetic and real datasets. The synthetic experiments show that PRGD can enhance RGD performance by two orders of magnitude. Additionally, in practical applications such as hyperspectral image completion and quantum state tomography, the efficacy of PRGD has been confirmed, achieving several tens of times faster than RGD. Moreover, we have provided a convergence guarantee for the PRGD algorithm, demonstrating that PRGD can linearly converge to the target tensor with the same initialization as RGD \cite{cai2022tensor}.

The main contributions of this article are summarized as follows:
\begin{itemize}
\item {\it Data-driven metric and provable theory.} We unfold the gradient $\mathcal{G}_l$ into matrices for each mode and weight the gradient using the 2-norm of the rows of the unfolded matrices. Thus, an adaptive data-driven metric is naturally constructed on the tangent space, and a new preconditioned Riemannian gradient descent (PRGD) algorithm is proposed. This weighting matrix is simple and easy to compute. Theoretically, we demonstrate that under good initialization conditions ($\| \mcal{T}_0 - \mcal{T}^* \|_F = o(1) \|\mcal{T}^* \|_F$), the PRGD algorithm can converge linearly to the target tensor. 
\item {\it Numerical experiments.} To validate the effectiveness of the PRGD algorithm, we conducted experiments on both synthetic and real datasets. On synthetic data, the PRGD algorithm significantly reduces the number of iterations and decreases the computation time by two orders of magnitude compared to the RGD algorithm.  We further applied the PRGD algorithm to practical problems such as hyperspectral image completion and quantum state tomography. Experimental results indicate that the PRGD algorithm also significantly reduces both the number of iterations and computation time, thereby effectively enhancing computational efficiency. This further validates that selecting a suitable metric on the tangent space can enhance the performance of the RGD algorithm, and the adaptive data-driven metric proposed by the PRGD algorithm is indeed an effective metric.
\end{itemize}

The rest of this paper is organized as follows. In section \ref{section: Preliminary}, we recall the basics of TT-format tensors, its decomposition and TT-SVD. In section \ref{sec:algorithm}, we describe our algorithm in detail. In section \ref{sec:convergence},  we present theoretical recovery guarantees for our proposed algorithm. In section \ref{section: numerical experiments}, we demonstrate the efficiency of our PRGD algorithm and highlight its superior performance compared to existing methods. In section \ref{sec:conclusion}, we conclude the paper.

{\bf Notations} In this paper, we use the calligraphic letters $(\mcal{T}, \wht{\mcal{T}})$ to denote tensors, the capital letters $(T, \wht{T})$ to denote the TT components, the mathematical scripts $(\scr{P}, \scr{I})$ to denote operators, the blackboard bold-face letters $(\bb{R}, \bb{M})$ to denote sets, the capital case bold-face letters $(\mbf{G}, \mbf{I})$ to denote matrices, and the lower case bold-face letters $(\mbf{x}, \mbf{y})$ to denote vectors. For a positive integer $d$, we denote $[d]:=\{1, \dots, d\}$. The Frobenius norm of tensors or matrices is denoted by $\|\cdot\|_F$. The notations $C_m, C_{m,1}, \dots$ represent constants that may depend on the tensor order $m$. For tensor size $\mbf{d} = (d_1, \dots, d_m)$ and rank $\mbf{r} = (r_1, \dots, r_{m-1})$, we denote $\bar{d}:=\max_{i=1}^m d_i, \bar{r}:=\max_{i=1}^{m-1} r_i$, and $\underline{r}:=\min_{i=1}^{m-1} r_i$. Additionally, we define $d^* := d_1 \cdots d_m$ and $r^*:=r_1\cdots r_{m-1}$. Define $\|\mcal{G}_l \|_{\vee} = \max_{i=1, \dots, m}  \max_{k} \| \scr{M}_{i}(\mcal{G}_l)(k, :)\|_2 $.

\section{Problem and Preliminary of TT-format Tensors}
\label{section: Preliminary}

\subsection{Preliminaries}        
We now briefly review some basic operations frequently used in the TT-format tensor representation. These operations are critical in the TT-format preconditioned RGrad (PRGD) algorithm. Interested readers can refer to \cite{oseledets2011tensor, cai2022provable} for more details.
\vspace{0.3cm}

{\bf TT-format.} For tensor $\mcal{T}\in \bb{R}^{d_1\times \cdots \times d_m}$, TT-format rewrites it into a product of $m$ $3$-order tensor. We call those $3$-order tensor components (or cores) and denoted by $T_1, \dots, T_m$, where $T_i\in \bb{R}^{r_{i-1}\times d_{i}\times r_{i}}$ and $r_0 = r_m = 1$, Those cores construct the tensor $\mcal{T}$ in the following manner that for all $(x_1,\dots, x_m)\in [d_1]\times \dots\times[d_m]$
    \begin{equation}
        \label{equ: TT Core}
        \mcal{T}(x_1, \dots, x_m) = T_1(x_1, :)T_2(:, x_2, :)\cdots T_{m}(:, x_m)
    \end{equation}
    where $T_{k}(:, x_k, :)\in \bb{R}^{r_{k-1}\times r_{k}}$ is the $x_k$-th submatrix of $T_k$ with the second index being fixed are ${x_k}$.\\
    
{\bf Separation and TT rank.}
The $i$-th separation of $\mathcal{T}$, denoted by $\mathcal{T}^{\langle i\rangle}$, is a matrix of size $\left(d_1 \cdots d_i\right) \times\left(d_{i+1} \cdots d_m\right)$ and defined by
    $$
    \mathcal{T}^{\langle i\rangle} = \textbf{reshape}(\mathcal{T}, \prod_{k = 1}^i d_k, \prod_{l = i+1}^{m} d_{l}) .
    $$
 Then, $r_i$ is defined by the rank of $\mathcal{T}^{\langle i\rangle}$. The collection $\boldsymbol{r}=\left(r_1, \cdots, r_{m-1}\right)$ is called the TT-rank of $\mathcal{T}$. For convenience, we denote $\rank_{tt}(\mcal{T}) = \mbf{r}$. As proved by Theorem 1 in  \cite{holtz2012manifolds}, the TT rank is well-defined for any tensor. Meanwhile, through the fast TT-SVD algorithm, we can always obtain the TT decomposition $\mcal{T}=[ T_1, \dots, T_m ]$, where $T_i \in \mathbb{R}^{r_{i-1} \times d_i \times r_i}$.\\

{\bf Unfolding and mode-$i$ product.}
\textit{Unfolding} process is to reorder the element of $\mcal{T}\in\bb{R}^{d_1\times\cdots\times d_m}$ into a matrix. We use $\scr{M}_i$ as the mode-$i$ unfolding operator, then $\scr{M}_i (\mcal{T})$ is matrix in $\bb{R}^{d_i\times \prod_{k\neq i} d_k}$. Given matrix $\mbf{X}\in \bb{R}^{d\times d_i}$, the \textit{mode-$i$ product} of $\mcal{T}$ and $\mbf{X}$ is denoted by $\mcal{Y} = \mcal{T}\times_i \mbf{X}$, which can be simply defined through unfolding operator
$$
\mcal{Y} = \mcal{T}\times_i \mbf{X}\Longleftrightarrow \scr{M}_i(\mcal{Y}) = \mbf{X} \scr{M}_i(\mcal{T}).
$$
 
{\bf Tensor condition number.} For TT-format tensor $\mcal{T}$ with TT-rank $(r_1, \dots, r_{m-1})$, the condition number of $\mcal{T}$ is defined by the smallest singular value among all the matrices obtained from separation, that is,
$$
\kappa(\mcal{T}):=\frac{\overline{\sigma}(\mcal{T})}{\underline{\sigma}(\mcal{T})} := \frac{\max_{i=1}^{m-1} 
\sigma_1(\mcal{T}^{\langle i\rangle})}{\min_{i=1}^{m-1} \sigma_{r_i} (\mcal{T}^{\langle i\rangle})},
$$
where $\sigma_{k}(\cdot)$ returns the $k-$th singular value of a matrix.

{\bf Left and right unfoldings.} For any $3$-order tensor $\mathcal{U} \in \mathbb{R}^{p_1 \times p_2 \times p_3}$, the left and right unfolding linear operators $L: \mathbb{R}^{p_1 \times p_2 \times p_3} \rightarrow \mathbb{R}^{\left(p_1 p_2\right) \times p_3}$, $R: \mathbb{R}^{p_1 \times p_2 \times p_3} \rightarrow \mathbb{R}^{p_1 \times\left(p_2 p_3\right)}$ are defined by
    $$
    L(\mathcal{U})(j x, k)=\mathcal{U}(j, x, k), \quad \text { and } \quad R(\mathcal{U})(j, x k)=\mathcal{U}(j, x, k) .
    $$
We say the component $T_i$ is \textit{left-orthogonal} if $L(T_i)^\top L(T_i)$ is an identity matrix. If all TT components $T_1, \dots, T_{m-1}$ in \eqref{equ: TT Core}
are \textit{left-orthogonal}, then such a TT decomposition is called \textit{left-orthogonal decomposition} of $\mcal{T}$.

{\bf Left and right parts of TT-format tensor.} For $\mcal{T} = [T_1,\dots, T_m]$, define the matrix $T^{\leq i}$ of size $(d_1\cdots d_i)\times r_i$ is the $i$-th left part of $\mcal{T}$
    \[
    T^{\leq i}(x_{1}, \dots, x_{i}, :) = T_1(x_1, :)T(:, x_2, :)\cdots T_i(:, x_i, :)\]
    Similarly, we define the matrix $T^{\geq i}$ of size $r_{i-1} \times\left(d_i \cdots d_m\right)$, know as the $i$-th right part of $\mathcal{T}$, column-wisely by
    $$
    T^{\geq i}\left(:, x_i \cdots x_m\right)=T_i\left(:, x_i, :\right) T_{i+1}\left(:, x_{i+1}, :\right) \cdots T_m\left(:, x_m\right)
    $$
    
    By default, we set $T^{\leq 0}=T^{\geq m+1}=[1]$. With these notations, the $i$-th separation of $\mathcal{T}$ can be factorized as
    $$
    \mathcal{T}^{\langle i\rangle}=T^{\leq i} T^{\geq i+1} .
    $$
    
{\bf TT-SVD.} Given an arbitrary $m$-th order tensor $\mcal{X}\in \bb{R}^{d_1\times\cdots\times d_m}$ and TT-rank $\mbf{r} = (r_1, \dots, r_{m-1})$, one can apply TT-SVD algorithm \cite{oseledets2011tensor} to obtain a TT-format approximation of $\mcal{X}$. Here we restate this algorithm as in \cref{alg: TT-SVD}.

\begin{center}    
\begin{algorithm}[H]
\caption{TT-SVD}
\begin{algorithmic}
\REQUIRE an $m$-order tensor $\mcal{X}\in \mathbb{R}^{d_1\times\cdots \times d_m}$ and target TT rank $\mathbf{r} = \left(r_1, r_2, \cdots, r_{m-1}\right)$.\\
\STATE Set $T^{\leq 0} = [1]$.
\FOR{$i = 1,\cdots, m-1$}
\STATE
$L(T_i)\longleftarrow \text{ the top } r_i \text{ left singular vectors of the matrix } (T^{\leq i-1}\otimes \mbf{I}_{d_i})^\top \mcal{X}^{\langle i\rangle} $

\STATE Set $T^{\leq i} =  (T^{\leq i-1}\otimes \mbf{I}_{d_i})L(T_i)$.
\ENDFOR
\STATE $T_m = (T^{\leq m-1})^\top \mcal{X}^{\langle m-1\rangle}$
\ENSURE:  $\operatorname{SVD}_{\mbf{r}}^{tt}(\mcal{X}) = [T_1, \ldots, T_m]\in \bb{M}_{\mbf{r}}^{tt}$.
\end{algorithmic} 
\label{alg: TT-SVD}
\end{algorithm} 
\end{center}

Usually, the output of TT-SVD is not the best low TT rank-$\mbf{r}$ approximation of $\mcal{X}$ and it is NP-hard to compute the best low TT rank-$\mbf{r}$ approximation of an arbitrary tensor\cite{hillar2013most}.

\section{Methodology}
\label{sec:algorithm}
\subsection{Riemannian Gradient Descent Algorithm}
Riemannian gradient descent algorithm is widely verified to be a powerful method for solving the optimization problem on the Riemannian manifold \cite{cai2022provable, bian2023preconditioned, wang2023implicit, zhang2024single}. Regarding the TT-format tensor completion problem in \eqref{model: Low TT}, the set of fixed TT rank tensors $\mathbb{M}_{\mbf{r}}^{tt}:= \{\mcal{T} \in \mathbb{R}^{d_1 \times d_2 \times \dots \times d_m}: \textmd{rank}_{tt} (\mcal{T}) = \mbf{r} \}$ is verified to be an embedding manifold in $\mathbb{R}^{d_1\times d_2\times \cdots\times d_m}$\cite{holtz2012manifolds}, thus one can apply the following RGD update scheme to solve \eqref{model: Low TT}: 
\begin{equation}\label{equ: RGD update}
    \mcal{T}_{l+1} = \scr{R}_l(\mcal{T}_l - \alpha_l \cdot\operatorname{grad}f(\mcal{T}_l))
\end{equation}
where $\alpha_l$ is the $l-$th stepsize, $\operatorname{grad}f(\mcal{T}_l)$ is the Riemannian gradient at $\mcal{T}_l$ on $\mathbb{M}_{\mbf{r}}^{tt}$ and $\scr{R}_l$ is the retraction map from the tangent space $\mathbb{T}_l$ at $\mcal{T}_l$ to $\mathbb{M}_{\mbf{r}}^{tt}$. The explicit formula of \eqref{equ: RGD update} involves the computation of Riemannian gradient $\operatorname{grad} f(\mathcal{T}_l)$, stepsize $\alpha_l$ and retraction operators $\scr{R}_l$:
\begin{itemize}
    \item \textit{Riemannian gradient and tangent space projection.}
    The Riemannian gradient $\operatorname{grad} f(\mcal{T}_l)$ depends on the metric used in tangent space $\mathbb{T}_l$. The classic RGD algorithm utilizes the Euclidean metric in $\mathbb{R}^{d_1\times\cdots\times d_m}$ as the metric in $\mathbb{T}_l$, thus the Riemannian gradient is obtained via projecting the vanilla gradient $\nabla f(\mcal{T}_l)$ onto the tangent space, e.g. $\operatorname{grad} f(\mcal{T}_l) = \scr{P}_{\mathbb{T}_l}(\nabla f(\mcal{T}_l))$. In computation, the tangent space can be parameterized explicitly under the chosen metric and the projection $\mathscr{P}_{\mathbb{T}_l} (\nabla f(\mcal{T}_l))$ can be obtained by fast recursive computations \cite{cai2022provable, steinlechner2016riemannian}. 
    
    \item \textit{Retraction by TT-SVD.}
    The retraction operator $\scr{R}_l$ is a map from tangent space $\mathbb{T}_l$ to manifold $\mathbb{M}_{\mbf{r}}^{tt}$. In general, the update $\mcal{T}_l - \alpha_l \cdot \mathscr{P}_{\mathbb{T}_l} (\nabla f(\mcal{T}_l))$ is not in $\mathbb{M}_{\mbf{r}}^{tt}$. A common choice of retraction is TT-SVD in \cref{alg: TT-SVD}. It is proved that the output of \cref{alg: TT-SVD} is quasi-optimal TT rank-$\mbf{r}$ approximation of the input tensor, with the projection error controllable \cite{oseledets2011tensor}. Moreover, the TT rank of $\mcal{T}_l - \alpha_l \cdot \mathscr{P}_{\mathbb{T}_l} (\nabla f(\mcal{T}_l))$ is no more than $2\mbf{r}$, indicating the computation cost of TT-SVD on $\mcal{T}_l - \alpha_l \cdot \mathscr{P}_{\mathbb{T}_l} (\nabla f(\mcal{T}_l))$ is very low \cite{steinlechner2016riemannian}.
    \item \textit{Choice of $\alpha_l$.} The optimal choice of $\alpha_l$ is the minimizer of the objective function along the geodesic of the search direction, i.e. $\alpha_l = \argmin_{\alpha} f(\scr{R}_l(\mcal{T}_l - \alpha \cdot \operatorname{grad} f(\mcal{T}_l)))$. However, it is hard to compute such a minimizer due to the nonlinearity of $\scr{R}_t$, instead, the linearized case in which minimizing in the tangent space is commonly adopted:
    $$
    \alpha_l = \argmin_\alpha f(\mcal{T}_l - \alpha \cdot \operatorname{grad} f(\mcal{T}_l)) = \frac{\|\scr{P}_{\mathbb{T}_l} \mcal{G}_l\|_F^2}{\|\scr{P}_{\Omega}\scr{P}_{\mathbb{T}_l} \mcal{G}_l\|_F^2}
    $$
\end{itemize}
It is proved that RGD with suitable initialization can converge linearly to the underlying low TT rank tensor under some mild assumptions \cite{cai2022provable}. Moreover, the contraction factor in linear convergence is independent of the condition number of the target low-rank solution, making RGD one of the most efficient algorithms for low-rank matrix/tensor completion or recovery problems.

\subsection{Data-Driven Metric and Precondition Riemannian Gradient Descent}

For unconstrained optimization problems in Hilbert space, it is widely investigated to design suitable metrics to accelerate the gradient descent algorithm. For example, Newton-type algorithms can be regarded as gradient descent in metrics weighted by the objective function's Hessian (or approximated Hessians). Similarly, one can accelerate RGD by altering the Euclidean metric in tangent space $\mathbb{T}_l$. \cite{bian2023preconditioned} uses this idea and proposes a Preconditioned RGD method for the low-rank matrix recovery problem.

\subsubsection{Data-Driven Metric}
We construct our metric from the gradient for each iterate. The gradient contains the observations of $\mcal{T}^*$ and sampling values $\scr{P}_{\Omega}(\mcal{T}_l)$, making the designed metric a \textit{data-driven metric} adapts to the measurements and algorithmic data. Let $\mcal{G}_l$ be the gradient of objective function $f$ in $\mathbb{R}^{d_1\times d_2\times \cdots\times d_m}$. We define the following matrices:
\begin{equation}
    \mbf{G}_{l,i}:= \epsilon\cdot \mbf{I}_{d_i} + \diag(\scr{M}_i(\mcal{G}_l) \scr{M}_i(\mcal{G}_l)^\top),\ \ \text{for}\ i=1,\dots, m.
\end{equation}
where each $\mbf{G}_{l,i}$ is a diagonal matrix in $ \bb{R}^{d_i\times d_i}$ and the $k$-th element of diagonal matrix $\diag(\scr{M}_i(\mcal{G}_l) \scr{M}_i(\mcal{G}_l)^\top)$ is equal to $\|\scr{M}_i(\mcal{G}_l) (k, :)\|_2^2$ where $\scr{M}_i(\mcal{G}_l) (k, :)$ is also the $k$-th slices of tensor $\mcal{G}_l$ in mode-$i$. Then we denote $\scr{W}_l$ as a weight operation and define a new metric in $\bb{R}^{d_1\times \dots\times d_m}$ as follows: for any tensors $\mcal{X}, \mcal{Y}\in \mathbb{R}^{d_1\times \cdots\times d_m}$,
\begin{equation}
    \label{equ: new metric}
    \langle \mcal{X},\mcal{Y} \rangle_{\scr{W}_l} = \langle \scr{W}_l\mcal{X}, \mcal{Y}\rangle = \langle \mcal{X}\times_{i=1}^m \mbf{G}_{l,i}^{\frac{1}{2 m}} , \mcal{Y}\rangle.
\end{equation}
Here $\langle\cdot, \cdot\rangle$ is the Euclidean inner product in $\bb{R}^{d_1\times \dots \times d_m}$. It is easy to verify that $\langle\cdot, \cdot\rangle_{\scr{W}_l}$ is an inner product and we denote the norm induced by $\langle\cdot, \cdot\rangle_{\scr{W}_l}$ as $\bnm \cdot \bnm_{\scr{W}_l}$. The weight operator $\scr{W}_l$ has the following \textit{component weight property}:

\begin{proposition}\label{prop: Component weighted}
For tensor-train format tensor $\mcal{T} = \left[T_1, \dots, T_m\right]$ with $T_i\in \bb{R}^{r_{i-1}\times d_i\times r_i}$, one has 
$$
\scr{W}_l \mcal{T} = \left[T_1\times_2 \mbf{G}_{l,1}^{\frac{1}{2 m}}, \dots, T_m\times_2 \mbf{G}_{l,m}^{\frac{1}{2 m}} \right]
$$
\end{proposition}
\begin{proof}
    Denote $\bar{\mcal{T}} = \left[\bar{T}_1, \dots, \bar{T}_m\right] = \left[T_1\times_2 \mbf{G}_{l,1}^{\frac{1}{2 m}}, \dots, T_m\times_2 \mbf{G}_{l,m}^{\frac{1}{2 m}} \right]$. By definition, for $(x_1,\dots, x_m)\in [d_1]\times \cdots \times [d_m]$,
    $$
    \begin{aligned}
        \scr{W}_l \mcal{T} (x_1,\dots, x_m) &= \mcal{T} (x_1,\dots, x_m) \prod_{k=1}^m\mbf{G}_{l,k}^{\frac{1}{2 m}}(x_k, x_k)\\
       & = T_1(x_1, :)T_2(:, x_2, :)\cdots T_{m}(:, x_m) \prod_{k=1}^m\mbf{G}_{l,k}^{\frac{1}{2 m}}(x_k, x_k)\\
       & = \mbf{G}_{l,1}^{\frac{1}{2 m}}(x_1, x_1) T_1(x_1, :) \cdots \mbf{G}_{l,m}^{\frac{1}{2 m}}(x_m, x_m)T_{m}(:, x_m)\\
       & = \bar{T}_1(x_1, :)\bar{T}_2(:, x_2, :)\cdots \bar{T}_{m}(:, x_m)\\
       & = \bar{\mcal{T}} (x_1,\dots, x_m)
    \end{aligned}
    $$
\end{proof}

\subsubsection{Preconditioned Riemannian Gradient Descent under $\langle\cdot, \cdot\rangle_{\scr{W}_l}$}
Now we derive the Riemannian gradient $\operatorname{\textbf{grad}}f(\mcal{T}_l) $ under new metric $\langle\cdot, \cdot\rangle_{\scr{W}_l}$. Let $\gamma_l(t)$ be a smooth curve in $\bb{M}_{\mbf{r}}^{tt}$ and satisfies $\gamma_l(0) = \mcal{T}_l$. Denote $\widetilde{\scr{P}}_{\bb{T}_l}$ be the orthogonal projection operator onto $\mathbb{T}_l$ under $\langle \cdot, \cdot\rangle_{\scr{W}_l}$, then
\begin{equation}
    \frac{\mathbf{d}}{\mathbf{d}s} f(\gamma_l(t))\big|_{t = 0} = \langle \dot{\gamma_l}(0), \mathcal{G}_l\rangle = \langle \dot{\gamma_l}(0), \widetilde{\scr{P}}_{\mathbb{T}_l}\scr{W}_l^{-1}\mcal{G}_l\rangle_{\scr{W}_l}
\end{equation}
Therefore, the Riemannian gradient $\operatorname{\textbf{grad}} f(\mcal{T}_l)$ under new metric is 
\begin{equation}
\operatorname{\textbf{grad}}f(\mcal{T}_l) = \widetilde{\scr{P}}_{\bb{T}_l}\scr{W}_l^{-1}\mcal{G}_l = \widetilde{\scr{P}}_{\bb{T}_l} (\mcal{G}_l\times_{i=1}^{m}\mbf{G}_{l,i}^{-\frac{1}{2 m}}).
\end{equation}
\begin{figure}[H]
    \centering
    \includegraphics[width = 0.25\linewidth]{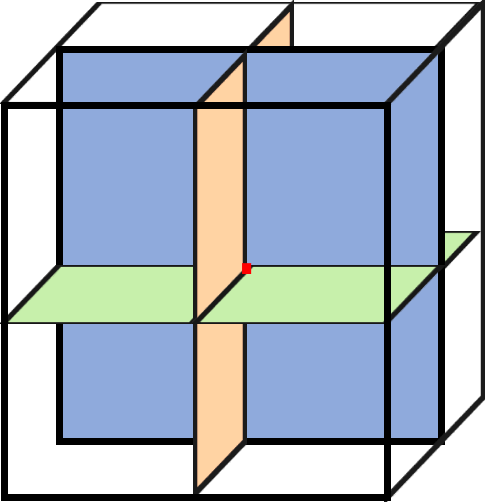}
    \caption{Illustration of scaling operation on $\mcal{G}_l$'s elements.}
    \label{fig: illust-weight}
\end{figure}
Utilizing TT-SVD as the retraction operator, the pseudocode of PRGD is proposed in Algorithm \ref{alg: PRGD}.
\begin{center} 
\begin{algorithm}[H]
\caption{Preconditioned Riemannian Gradient Descent}
\begin{algorithmic} 
\STATE \textbf{Initialization} $\mcal{T}_0\in\bb{M}_{\mbf{r}}^{tt}$, spikiness parameter $\nu$ and maximum iteration numbers $l_{\max}$.\\
\FOR{$l = 0, 1, \dots, l_{\max}$}
\STATE $\mcal{G}_l = \scr{P}_\Omega (\mcal{T}_l - \mcal{T}^*)$.\\
\STATE Compute stepsize $\alpha_l$. \\
\STATE $\mcal{W}_l = \mcal{T}_l - \alpha_l \widetilde{\scr{P}}_{\bb{T}_l}\scr{W}_l^{-1}\mcal{G}_l$. \\
\STATE Set $\overline{\mcal{W}}_l = \operatorname{Trim}_{\xi_l}(\mcal{W}_l)$ with $\xi_l = \frac{10\|\mcal{W}_l\|_F}{9\sqrt{d^*}}\nu$\\
\STATE $\mcal{T}_{l+1} = \operatorname{SVD}_\mbf{r}^{tt}(\overline{\mcal{W}}_l)$
\ENDFOR
\end{algorithmic} 
\label{alg: PRGD}
\end{algorithm}
\end{center}

In Algorithm \ref{alg: PRGD}, we have the same additional procedure $\operatorname{Trim}_{\xi_l}(\mcal{W}_l)$ as RGD in \cite{cai2022provable}. The trimming operator is defined as follows:
\begin{equation}
\label{equ: trim}
\operatorname{Trim}_\zeta(\mathcal{T})=\left\{\begin{array}{l}
\zeta \cdot \operatorname{sign}(\mathcal{T}(\boldsymbol{x})), \quad \text { if }|\mathcal{T}(\boldsymbol{x})| \geq \zeta \\
\mathcal{T}(\boldsymbol{x}), \quad \text { otherwise }
\end{array}\right.
\end{equation}
Due to the same reasons in \cite{cai2022provable}, the trimming step is to guarantee the incoherence property of $\mcal{T}_l$ in each iteration, and it is necessary for our proof of the convergence of PRGD. However, in numerical experiments, the performance of PRGD is almost the same with or without the trimming procedure. So we remark that such a procedure is completely for convenience for technical proof, thus in practice, one can skip this procedure for simplicity.

Regarding the step size $\alpha_l$ in Algorithm \ref{alg: PRGD}, one strategy is computing the steepest stepsize along the search direction in tangent space $\mathbb{T}_l$:
\begin{equation}\label{equ: PRGD stepsize}
    \alpha_l = \argmin_\alpha f(\mcal{T}_l - \alpha \widetilde{\scr{P}}_{\bb{T}_l}\scr{W}_l^{-1}\mcal{G}_l) = \frac{\|\widetilde{\scr{P}}_{\bb{T}_l}\scr{W}_l^{-1}\mcal{G}_t\|_{\scr{W}_l}^2}{\|\scr{P}_{\Omega}\widetilde{\scr{P}}_{\bb{T}_l}\scr{W}_l^{-1}\mcal{G}_l\|_F^2}
\end{equation}
however, the adaptive stepsize $\alpha_l$ in \eqref{equ: PRGD stepsize} needs to compute the explicit form of tangent space projection $\widetilde{\scr{P}}_{\bb{T}_l}\scr{W}_l^{-1}\mcal{G}_l$, which involves considerable computation cost in each iteration. 

Another strategy to avoid this computational cost is keeping the stepsize $\alpha_l$ constant during the iteration. Interestingly, the numerical experiments in Section \ref{section: numerical experiments} show that the performance of the optimal constant stepsize strategy is much better than the adaptive stepsize strategy in practice. This might be confusing because the adaptive stepsize in \eqref{equ: PRGD stepsize} is steepest along the search direction in the tangent space. Here we remark that this might be because the precondition metric sharps the local curvature of the manifold, making the retraction error larger for adaptive stepsize, thus resulting in a worse performance.

\subsection{Tangent Space Parameterization and Computation of $\widetilde{\scr{P}}_\bb{T}$}\label{subsec: Tg Para and Tp Comp}

{\bf New Left Orthogonal Decomposition of $\mcal{T}\in \bb{M}_{\bf{r}}^{tt}$.} Given operator $\scr{W}_l$ and corresponding inner product $\langle\cdot, \cdot\rangle_{\scr{W}_l}$, we need to suitably parameterize the tangent space $\bb{T}_l$ such that we can obtain the close form of $\wtd{\scr{P}}_{\bb{T}_l}$. Before that, we first define the \textit{new left orthogonal} tensor-train decomposition as follows:
\begin{definition}[New left orthogonal decomposition]
    Given $\mcal{T}\in \bb{M}_{\mbf{r}}^{tt}$, we call $\mcal{T} = [\wtd{T}_1, \dots, \wtd{T}_m]$ is a new left orthogonal decomposition if
    $$
    L(\widetilde{T}_i)^\top (\mbf{I}_{r_{i-1}}\otimes \mbf{G}_{l, i}^{\frac{1}{2 m}}) L(\widetilde{T}_i) = \mbf{I}_{r_{i}}, \ \text{for}\ i=1,\dots, m
    $$
    e.g. the left unfolding matrix of $\wtd{T}_i$ is column-wise orthogonal under the inner product weighted by the positive definite matrix $(\mbf{I}_{r_{i-1}}\otimes \mbf{G}_{l, i}^{\frac{1}{2 m}})$.
\end{definition}

Based on the decomposition $\mcal{T} = [\wtd{T}_1, \dots, \wtd{T}_m]$, we denote the corresponding $i$-th left and right parts of $\mcal{T}$ as $\wtd{T}^{\leq i}$ and $\wtd{T}^{\geq i}$ respectively. Then we have the following lemma:

\begin{lemma}\label{lemma: orth left sep}
Let  $\mcal{T} = [\widetilde{T}_1,\dots, \widetilde{T}_m]$ be a new left orthogonal decomposition of $\mcal{T}$. Then $\widetilde{T}^{\leq i}$ is columnwise orthogonal under the inner product weighted by $(\mbf{G}_{l, 1}^{\frac{1}{2 m}}\otimes \mbf{G}_{l, 2}^{\frac{1}{2 m}}\otimes\cdots \otimes\mbf{G}_{l, i}^{\frac{1}{2 m}})$, e.g.,
    \begin{equation}\label{equ: new orthogonal}
    \widetilde{T}^{\leq i\top} (\mbf{G}_{l,1}^{\frac{1}{2 m}}\otimes \mbf{G}_{l,2}^{\frac{1}{2 m}}\otimes\cdots \otimes\mbf{G}_{l,i}^{\frac{1}{2 m}}) \widetilde{T}^{\leq i} = \mbf{I}_{r_i}.
    \end{equation}
\end{lemma}
\begin{proof}
    We prove this lemma through the recursive method.
    
    For $i = 1$, $\widetilde{T}^{\leq 1} = L(\widetilde{T}_1) = \widetilde{T}_1\in \mathbb{R}^{d_1\times r_1}$, then 
    $$
    \widetilde{T}^{\leq 1 \top} \mbf{G}_{l,1}^{\frac{1}{2 m}} \widetilde{T}^{\leq 1} = \widetilde{T}^{\leq 1 \top} (\mbf{I}_{r_0}\otimes \mbf{G}_{l,1}^{\frac{1}{2 m}}) \widetilde{T}^{\leq 1} = \mbf{I}_{r_1}
    $$
    Suppose \eqref{equ: new orthogonal} holds for $i = k-1$, then for $i = k$
    $$
    \begin{aligned}
         \widetilde{T}^{\leq k\top} &(\mbf{G}_{l,1}^{\frac{1}{2 m}}\otimes \mbf{G}_{l,2}^{\frac{1}{2 m}}\otimes\cdots \otimes\mbf{G}_{l,k}^{\frac{1}{2 m}}) \widetilde{T}^{\leq k} \\
         &= L(\widetilde{T}_k)^\top (\widetilde{T}^{\leq k-1 \top}\otimes \mbf{I}_{d_k})  (\mbf{G}_{l,1}^{\frac{1}{2 m}}\otimes \mbf{G}_{l,2}^{\frac{1}{2 m}}\otimes\cdots \otimes\mbf{G}_{l,k}^{\frac{1}{2 m}}) (\widetilde{T}^{\leq k-1}\otimes \mbf{I}_{d_k})L(\widetilde{T}_k)\\
         &= L(\widetilde{T}_k)^\top \left((\widetilde{T}^{\leq k-1 \top}(\mbf{G}_{l,1}^{\frac{1}{2 m}}\otimes \mbf{G}_{l,2}^{\frac{1}{2 m}}\otimes\cdots \otimes\mbf{G}_{l,k-1}^{\frac{1}{2 m}})\widetilde{T}^{\leq k-1}) \otimes \mbf{G}_{l,k}^{\frac{1}{2 m}}\right) L(\widetilde{T}_k)\\
         & = L(\widetilde{T}_k)^\top (\mbf{I}_{r_{k-1}}\otimes \mbf{G}_{l,k}^{\frac{1}{2 m}}) L(\widetilde{T}_k)\\
        & = \mbf{I}_{r_k}
    \end{aligned}
    $$
Thus, \eqref{equ: new orthogonal} holds for $i=1, \dots, m$.
\end{proof}

There is one essential question that remains to be answered: 
$$\textit{Given } \mcal{T}\in \bb{M}_{\mbf{r}}^{tt}, \textit{ how to obtain the new left orthogonal decomposition of } \mcal{T}?
$$
We first set $\wht{\mcal{T}}:=\scr{W}_{l}^{\frac{1}{2}}\mcal{T}$ and obtain the left orthogonal TT decomposition  $\wht{\mcal{T}} = [\wht{T}_1, \dots, \wht{T}_m]$ from the classical procedure, e.g. TT-SVD \cref{alg: TT-SVD}. Then we apply the inversion operation $\scr{W}_l^{-\frac{1}{2}}$ on $\wht{\mcal{T}}$,

\begin{equation}
\label{equ: new left orthogonal}
\mcal{T} = \scr{W}_l^{-\frac{1}{2}}\wht{\mcal{T}} = \left[\wtd{T}_1, \dots, \wtd{T}_m\right]:=\left[\wht{T}_1\times_2 \mathbf{G}_{l, 1}^{-\frac{1}{4 m}}, \dots, \wht{T}_m\times_2 \mathbf{G}_{l, m}^{-\frac{1}{4 m}} \right].
\end{equation}
It is easy to verify that $\wtd{T}_i$ is left orthogonal under $\langle\cdot, \cdot\rangle_{\mbf{I}_{r_{i-1}}\otimes \mbf{G}_{l, i}^{\frac{1}{2 m}}}$.

$$
 L(\widetilde{T}_i)^\top (\mbf{I}_{r_{i-1}}\otimes \mbf{G}_{l,i}^{\frac{1}{2 m}}) L(\widetilde{T}_i) =  L(\wht{T}_i)^\top L(\wht{T}_i) = \mbf{I}_{r_i}
$$

\noindent{\bf Tangent Space Parameterization.} Now we turn to parameterize $\bb{T}_l$ under $\langle\cdot, \cdot\rangle_{\scr{W}_l}$. According to \cite{holtz2012manifolds}, for $\mcal{T} = [\widetilde{T}_1,\dots, \widetilde{T}_m] \in \mathbb{M}_{\bf{r}}^{tt}$ new left orthogonal decomposition. We can parameterize the tangent space at $\mcal{T}$ based on gauge sequence $\mbf{I}_{r_{i-1}}\otimes \mbf{G}_{l,i}^{\frac{1}{2 m}}, i\in [m]$. For any $\mcal{X}\in \bb{T}$, there exist a sequence of tensors $X_1,\dots, X_m$ with $X_i\in \bb{R}^{r_{i-1}\times d_i\times r_i}$ such that 
$$
\mcal{X} = \sum_{i=1}^m \delta \mcal{X}_i\ \ \text{where}\ \ \delta \mcal{X}_i = [\widetilde{T}_1,\dots, \widetilde{T}_{i-1}, X_i, \widetilde{T}_{i+1}, \dots, \widetilde{T}_m]
$$
with $[X_i, \widetilde{T}_i]_{(\mbf{I}_{r_{i-1}}\otimes \mbf{G}_{l,i}^{\frac{1}{2 m}})}:= L(X_i)^T (\mbf{I}_{r_{i-1}}\otimes \mbf{G}_{l,i}^{\frac{1}{2 m}}) L_(\widetilde{T}_i)  = \mbf{0}, i\in [m]$. Then for an arbitrary tensor $\mcal{A}\in \bb{R}^{d_1\times \dots\times d_m}$, the orthogonal projection of $\mcal{A}$ onto $\bb{T}$ under $\langle\cdot, \cdot\rangle_{\scr{W}_l}$ is denoted by $\widetilde{\scr{P}}_{\bb{T}}(\mcal{A})$:

\begin{equation}\label{projection}
\widetilde{\scr{P}}_{\bb{T}}(\mcal{A}) = \sum_{i=1}^m \delta \mcal{A}_i\ \text{where}\ \delta \mcal{A}_i = [\widetilde{T}_1, \dots, \widetilde{T}_{i-1}, A_i, \widetilde{T}_{i+1}, \dots, \widetilde{T}_m]
\end{equation}
for some tensor $A_i$ of size $r_{i-1}\times d_i \times r_i$ satisfying $[\widetilde{T}_i, A_i]_{(\mbf{I}_{r_{i-1}}\otimes \mbf{G}_{l,i}^{\frac{1}{2 m}})} = \mbf{0}$. We have the following lemma to indicate that $\delta \mcal{A}_i$'s are orthogonal to each other under precondition metric.
\begin{lemma}\label{lemma: component orthogonal}
    For an arbitrary $\mcal{A}$ of size $d_1\times \dots \times d_m$, the components of $\delta \mcal{A}_i$'s of $\widetilde{\scr{P}}_{\bb{T}}(\mcal{A})$ satisfy $\langle\delta \mcal{A}_i , \delta \mcal{A}_j\rangle_{\scr{W}_l} = 0$ for all $1\leq i\neq j\leq m$
\end{lemma}
\begin{proof}
    Denote $\wht{T}_i = \wtd{T}_i\times_2 \mbf{G}_{l,i}^{\frac{1}{4 m}}, \wht{A}_i = \wtd{A}_i\times_2 \mbf{G}_{l,i}^{\frac{1}{4 m}}$ for all $i\in [m]$. Then without loss of generalization, assume $i\leq j$ we have 
    $$
    \begin{aligned}
    \langle \delta\mcal{A}_i, \delta\mcal{A}_j\rangle_{\scr{W}_l} & =  \langle \delta\mcal{A}_i\times_{i = 1}^{m} \mbf{G}_{l,i}^{\frac{1}{4 m}}, \delta\mcal{A}_j\times_{i = 1}^{m} \mbf{G}_{l,i}^{\frac{1}{4 m}}\rangle\\
    & = \langle (\wht{T}^{\leq i-1}\otimes \mbf{I}_{d_i}) L(\wht{A}_i) \wht{T}^{\geq i+1}, (\wht{T}^{\leq i-1}\otimes \mbf{I}_{d_i})L(\wht{T}_i) (\wht{\delta A_j})^{\geq i+1} \rangle\\
    & = \langle L(\wht{A}_i) \wht{T}^{\geq i+1}, L(\wht{T}_i) (\wht{\delta A}_j)^{\geq i+1}\rangle\\
    & = 0
    \end{aligned}
    $$
    where $\left(\delta\mcal{A}_j\times_{i = 1}^{m} \mbf{G}_{l,i}^{\frac{1}{4 m}}\right)^{\langle i \rangle} = (\wht{T}^{\leq i-1}\otimes \mbf{I}_{d_i})L(\wht{T}_i) (\wht{\delta A_j})^{\geq i+1}$ for some $(\wht{\delta A_j})^{\geq i+1}$ for simplicity. The third equality is due to $\wht{T}^{\leq i-1} = (\mbf{G}_{l,1}^{\frac{1}{4 m}}\otimes \mbf{G}_{l,2}^{\frac{1}{4 m}}\otimes \cdots \otimes\mbf{G}_{l,i-1}^{\frac{1}{4 m}}) \wtd{T}^{\leq i-1}$ and Lemma \ref{lemma: orth left sep}. The last equality is due to the fact that $L(\wht{A}_i)^\top L(\wht{T}_i)$ is an all-zeros matrix.
\end{proof}

\vspace{0.3cm}
\noindent\textbf{The close form of $\widetilde{\scr{P}}_\bb{T}$.} 
\begin{lemma}\label{lemma: close of P_tilde}
    For tensor $\mcal{T}\in \mathbb{M}_{\mbf{r}}^{tt}$, denote $\widetilde{\scr{P}}_\bb{T}$ as the tangent space projection operator of $\mcal{T}$ under $\langle \cdot, \cdot\rangle_{\scr{W}_l}$. Let $\wht{\mcal{T}}=\scr{W}^{\frac{1}{2}}_l(\mcal{T})$ and denote $\wht{\bb{T}}$ be the tangent space of $\wht{\mcal{T}}$, then
    $$
    \widetilde{\scr{P}}_\bb{T} = \scr{W}_l^{-\frac{1}{2}} \scr{P}_{\wht{\bb{T}}}\scr{W}_l^{\frac{1}{2}}.
    $$
\end{lemma}

\begin{proof}
    Due to the orthogonality property of Lemma \ref{lemma: component orthogonal}, determining $\delta \mcal{A}_i$ is equivalent to solve the following individual optimization problem 

\begin{equation}\label{equ: component minimal}
    \min _{A_i}\left\|\mcal{A}-\delta \mcal{A}_i\right\|_{\scr{W}_l}, \ \text { s.t. } \delta \mcal{A}_i=\left[\widetilde{T}_1, \dots, A_i, \dots, \widetilde{T}_m\right] \text { and } [\widetilde{T}_i, A_i]_{(\mbf{I}_{r_{i-1}}\otimes \mbf{G}_{l,i}^{\frac{1}{2 m}})} = \mbf{0}
\end{equation}

By the definition of $\|\cdot\|_{\scr{W}_l}$, denote $\wht{\mcal{A}} = \scr{W}_l^{\frac{1}{2}}(\mcal{A})$ and $\wht{\mcal{T}} = \scr{W}_l^{\frac{1}{2}}(\mcal{T})$, then $[\wht{T}_1, \dots, \wht{T}_m]:=[T_1\times_2 \mbf{G}_{l, 1}^{\frac{1}{4m}}, \dots, T_m\times_2 \mbf{G}_{l, m}^{\frac{1}{4m}}]$ is the left orthogonal decomposition of $\wht{\mcal{T}}$ under $\langle\cdot, \cdot\rangle$. So \eqref{equ: component minimal} is equivalent to 


\begin{equation}\label{equ: transformed component minimal}
    \min _{\widehat{A}_i}\left\|\widehat{\mcal{A}}-\widehat{\delta \mcal{A}}_i\right\|_{\mbf{F}}, \ \text { s.t. } \widehat{\delta \mcal{A}_i}=\left[\widehat{T}_1, \dots, \widehat{A}_i, \dots, \widehat{T}_m\right] \text { and } [\widehat{T}_i, \widehat{A}_i] = \mbf{0}
\end{equation}
here $\wht{A}_i = A_i\times \mbf{G}_{l, i}^{\frac{1}{4m}}$ and this procedure is actually computing $\scr{P}_{\wht{\bb{T}}}\scr{W}_l^{\frac{1}{2}}(\mcal{A})$. Referring to \cite{lubich2015time, cai2022provable}, we get the following close form solution of \eqref{equ: transformed component minimal}:
\begin{equation}\label{equ: tsp in nm}
L\left(\widehat{A}_i\right)=
\left(\mbf{I}_{r_{i-1}d_i}- L\left(\widehat{T}_i\right) L\left(\widehat{T}_i\right)^{\top}\right)\left(\widehat{T}^{\leq i-1} \otimes \mbf{I}_{d_i}\right)^{\top} \widehat{\mcal{A}}^{\langle i\rangle}\left(\widehat{T}^{\geq i+1}\right)^{\top}\left(\widehat{T}^{\geq i+1}\left(\widehat{T}^{\geq i+1}\right)^{\top}\right)^{-1}
\end{equation}
for $i \in[m-1]$  and $L\left(\widehat{A}_m\right) = \left(\widehat{T}^{\leq m-1} \otimes \mbf{I}_{d_m}\right)^{\top} \widehat{\mcal{A}}^{\langle m\rangle}.$ Then one can obtain $A_i = \widehat{A}_i\times_2 \mbf{G}_{l, i}^{-\frac{1}{4 m}}$ for $i =1, \dots, m$. Conclude all above, we get:
$$
    \widetilde{\scr{P}}_\bb{T}(\mcal{A}) = \scr{W}_l^{-\frac{1}{2}} \scr{P}_{\wht{\bb{T}}}\scr{W}_l^{\frac{1}{2}}(\mcal{A}).
$$
\end{proof}
\begin{remark}
We claim that the PRGD iterate on $\mcal{T}_l$ is almost the RGD iterate on $\wht{\mcal{T}}_l$. Plugging Lemma \ref{lemma: close of P_tilde} into the iterate scheme of PRGD gives that $\mcal{T}_{l+1} = \scr{R}_l(\scr{W}_l^{-\frac{1}{2}}(\widehat{\mcal{T}}_l - \alpha_l \cdot \scr{P}_{\wht{\bb{T}}_l} \wht{\mcal{G}_l}))$, where $\wht{\mcal{G}}_l:=\scr{W}_l^{-\frac{1}{2}}\mcal{G}_l$ is the gradient of the objective function on $\wht{\mcal{T}}_l$.
Moreover, if one chooses $\scr{R}_l = \scr{W}_l^{-\frac{1}{2}}\operatorname{SVD}_{\mbf{r}}^{tt}\scr{W}_l^{\frac{1}{2}}$ as retraction operator, then the iterate scheme of PRGD on $\mcal{T}_l$ became $\wht{\mcal{T}}_{l+1} = \scr{W}_{l+1}^{\frac{1}{2}}\scr{W}_l^{-\frac{1}{2}}\operatorname{SVD}_{\mbf{r}}^{tt} (\widehat{\mcal{T}}_l - \alpha_l \cdot \scr{P}_{\wht{\bb{T}}_l} \wht{\mcal{G}_l})$ where $\wht{\mcal{T}}_{l+1}:= \scr{W}_{l+1}^{\frac{1}{2}} (\mcal{T}_{l+1})$, indicating that it is almost the RGD iterate on $\wht{\mcal{T}}_l$.
\end{remark}

\begin{remark}[Computational Complexity of PRGD]\label{rmk: Computational Complexity}
Lemma \ref{lemma: close of P_tilde} indicates that, compared with RGD, PRGD involves two additional operations $\scr{W}_l^{\frac{1}{2}}$ and $\scr{W}_l^{-\frac{1}{2}}$. However, we claim the computational cost of these two operations is insignificant. For the first operation $\scr{W}_l^{-\frac{1}{2}} (\mcal{G}_l)$, since $\mcal{G}_l = \scr{P}_{\Omega} (\mcal{T}_l - \mcal{T}^*)$ is sparse tensor and $\mbf{G}_{l, i}$ in $\scr{W}_l$ is diagonal matrix, this computational cost is $O(m |\Omega|)$. For the second operation, since $\wtd{\scr{P}}_{\bb{T}_l}\scr{W}_l^{-\frac{1}{2}}(\mcal{G}_l)$ can be represented by sum of $m$ rank-$\mbf{r}$ tensor. Thus with Proposition \ref{prop: Component weighted}, the computational cost with this part is $O(m \bar{d} \bar{r}^2)$. As for the retraction part, the operation $\scr{W}_l$ doesn't change the TT-rank of the given tensor, so $\mcal{W}_l$ has the rank at most $2 \mbf{r}$, making the computation cost of retraction the same with RGD.
\end{remark}

\section{Recovery Guarantee}
\label{sec:convergence}
In this section, we will present the convergence guarantee of the PRGD (i.e., Algorithm \ref{alg: PRGD}) and provide the proof of the main theorem. We first present the main theorem on exact matrix recovery and related lemmas.

\begin{theorem}\label{lemma:convergence-guarantee}
Assume that $\mcal{T}^*$ is of size $d_1 \times d_2 \times \cdots \times d_m$ with a TT-rank $\mbf{r} = (r_1, \dots, r_m)$ whose spkiness is bounded by $Spiki(\mcal{T}^*) \leq \nu$ and the initilization $\mcal{T}_0$ satifies 
$$
\|\mcal{T}_0 - \mcal{T}^*\|_F \leq \frac{\underline{\sigma}}{Cm \kappa_0 \bar{r}^{1/2}} \quad \textmd{and} \quad Incoh(\mcal{T}_0) \leq 2 \kappa_0^2 \nu
$$
where $C > 0$ is a sufficiently large absolute constant. Set the stepsize $\alpha_l = 1.001 \epsilon_l^{\frac{1}{2}} p^{-1}$. Then, there exists an absolute constant $C_m > 0$ depending only on $m$ such that if the sample size $n$ satisfies 
$$
n \geq C_m \rho_l^{2m+6} \bar{d} \log^{2m+4}(\bar{d}) \kappa_0^{8} \mu^{m} (r^*)^2 r_{m-1}^2 \sum_{i=1}^{m-1} r_i^{-1} + C_m \rho_l^{m+3} \sqrt{d^*} \log^{m+2}(\bar{d}) \kappa_0^4 \mu^{\frac{m}{2}} r^* r_{m-1} \sum_{i=1}^{m-1} r_i^{-\frac{1}{2}},
$$
then with probability at least $1 - (m+4)\bar{d}^{-m}$, the sequence $\{ \mcal{T}_l \}_{l=1}^{\infty}$ generated by Algorithm \ref{alg: PRGD} satisfy 
$$
 \| \mcal{T}_l - \mcal{T}^* \|_F^2 \leq 0.3574 \|\mcal{T}_{l-1} - \mcal{T}^*\|_F^2
$$
for all $l = 1, 2, \cdots$.
\end{theorem}

\subsection{Key Lemmas}
For the recovery guarantee of exact tensor, the weighted inner product $\langle \cdot, \cdot \rangle_{\scr{W}_l}$ defined in \eqref{equ: new metric} and some lemmas introduced by this inner product play a crucial role. Here, we first present a lemma to demonstrate the equivalence between the weighted norm $\| \cdot \|_{\scr{W}_l}$  and the Frobinus norm $\| \cdot \|_F$. 

\begin{lemma}\label{norm-equivalence}
For any $\mcal{Z} \in \bb{R}^{d_1 \times d_2 \times \cdots \times d_m}$ and $l \in \bb{N}$, we have
\begin{equation}
\label{eq:equivalence} 
\nu_l \| \mcal{Z} \|_{F}^2 \leq \| \mcal{Z} \|_{\scr{W}_l}^2 \leq \mu_l \| \mcal{Z} \|_{F}^2,
\end{equation}
where $\nu_l = \epsilon_l^{\frac{1}{2}}$ and $\mu_l = (\epsilon_l + \| \mcal{G}_l \|_{\vee}^2 )^{\frac{1}{2}}$.
\end{lemma}

\begin{proof}
According to the definition of $\langle\cdot, \cdot\rangle_{\scr{W}_l}$ in \eqref{equ: new metric}, we have
$$
\begin{aligned}
\| \mcal{Z} \|_{\scr{W}_l}^2 &= \langle \mcal{Z}\times_{i=1}^m \mbf{G}_{l,i}^{\frac{1}{4 m}}, \mcal{Z}\times_{i=1}^m \mbf{G}_{l,i}^{\frac{1}{4 m}} \rangle
\\
&= \| \mcal{Z} \times_{i=1}^m \mbf{G}_{l,i}^{\frac{1}{4m}} \|_{F}^2,
\end{aligned}
$$
where $\mbf{G}_{i, l}:= \epsilon_l \mbf{I}_{d_i} + \diag(\scr{M}_i(\mcal{G}_l) \scr{M}_i(\mcal{G}_l)^\top)$ for $i = 1, \dots, m$. Since $\mbf{G}_{i,l}$ is a diagonal matrix and each diagonal element satisfies
$$
\epsilon_l < | \mbf{G}_{i,l} (j, k) | < \epsilon_l + \| \mcal{G}_l \|_{\vee}^2,
$$
then we have
$$
\epsilon_l^{\frac{1}{2}} \| \mcal{Z} \|_{F}^2 \leq \| \mcal{Z} \|_{\scr{W}_l}^2 \leq (\epsilon_l + \| \mcal{G}_l \|_{\vee}^2 )^{\frac{1}{2}} \| \mcal{Z} \|_{F}^2. 
$$
Denote $\nu_l = \epsilon_l^{\frac{1}{2}} $ and $\mu_l = (\epsilon_l + \| \mcal{G}_l \|_{\vee}^2 )^{\frac{1}{2}}$, it can be rewritten as
$$
\nu_l \| \mcal{Z} \|_{F}^2 \leq \| \mcal{Z} \|_{\scr{W}_l}^2 \leq \mu_l \| \mcal{Z} \|_{F}^2.
$$
This completes the proof.
\end{proof}

Next, we present some lemmas required by the proof of the main theorem. Their proofs are provided in the Appendix \ref{appendixA}.


\begin{lemma}\label{lemma:bound of tildet}
Let $\mcal{T}$ be a tensor of rank $(r_1, \dots, r_{m-1})$ satisfying $Incoh(\mcal{T}) \leq \sqrt{\mu}$. Consider its left orthogonal decomposition $\mcal{T} = \left[\widetilde{T}_1, \dots, \widetilde{T}_m\right]$ generated by \eqref{equ: new left orthogonal}. Then the following bounds hold
        $$
    \max_{x_i}\|\wtd{T}_i (:, x_i, :)\|_F\leq\left\{\begin{array}{cc}
        \nu_l^{-\frac{1}{2m}}  \rho_l \sqrt{\frac{ \mu r_i}{d_i}} & i\in [m-1], \\
         \nu_l^{-\frac{1}{2m}} \rho_l  \mu_l^{\frac{1}{2}}\sigma_{\max} (\mcal{T})\sqrt{  \frac{\mu r_{m-1}}{d_m}} & i = m,
    \end{array}\right.
 $$
    and

    $$
    \quad
    \|\wtd{T}_i\|_F \leq \left\{\begin{array}{cc}
       \nu_l^{-\frac{1}{2m}} \sqrt{r_i}  &  i\in [m-1],\\
         \nu_l^{-\frac{1}{2m}} \mu_l^{\frac{1}{2}}\sqrt{\underline{r}} \sigma_{\max}(\mcal{T}) &  i = m.
    \end{array}\right. 
    $$
\end{lemma}

\begin{lemma}\label{lemma: tilde delta norm}
Let $\mcal{T}$ and $\mcal{T}^*$ be tensors of rank $(r_1, r_2, \dots, r_{m-1})$ with $Incoh(\mcal{T}) \leq \sqrt{\mu}$. Consider the left orthogonal decomposition $\mcal{T} =[\wtd{T}_1, \wtd{T}_2, \dots, \wtd{T}_m]$ and $\mcal{T}^* = [\wtd{T}_1^*, \wtd{T}_2^*, \dots, \wtd{T}_m^*]$, respectively, generated by \eqref{equ: new left orthogonal}. Denote $\wtd{\Delta}_i = \wtd{T}_i - \wtd{T}_i^*$, then the following bounds hold
$$
\|\wtd{\Delta}_i\|_F\leq \left\{\begin{array}{ll}
    20 \rho_l^{\frac{3}{2}} \nu_l^{-\frac{1}{2m}} \sigma_{\min}^{-1}(\mcal{T}^*) \sqrt{r_i} \kappa_0^2 \|\mcal{T} - \mcal{T}^*\|_F, & i\in [m-1], \\
    3 \rho_l^{\frac{1}{2}} \nu_l^{-\frac{1}{2m}} \mu_l^{\frac{1}{2}} \kappa_0 \sqrt{r_{m-1}} \|\mcal{T} - \mcal{T}^*\|_F, & i=m,
\end{array}\right.
$$
and 
$$
\max_{x_i}\|\wtd{\Delta}_i(:, x_i, :)\|_F\leq \left\{\begin{array}{ll}
 2 \rho_l \nu_l^{-\frac{1}{2m}}\sqrt{\mu r_i d_i^{-1}}, & i\in [m-1], \\
   2 \rho_l \nu_l^{-\frac{1}{2m}} \mu_l^{\frac{1}{2}}\sqrt{ \mu r_{m-1} d_m^{-1}} \sigma_{\max}(\mcal{T}^*), & i=m.
\end{array}\right.
$$
\end{lemma}
\subsection{Proof of Main Theorem}
We begin by providing the proof of Theorem \ref{lemma:convergence-guarantee}.
\begin{proof}[Proof of Theorem \ref{lemma:convergence-guarantee}]
By Algorithm \ref{alg: PRGD} and the equivalence between the Frobenius norm and the $\| \cdot \|_{\mcal{W}_l}$ norm, we have
$$
\begin{aligned}
\|\mcal{T}_{l+1} - \mcal{T}^*\|^2_{F} &= \|\operatorname{SVD}_\mbf{r}^{tt}(\mcal{W}_l) - \mcal{T}^*\|_{F}^2\\
&\leq \|\mcal{W}_l - \mcal{T}^*\|_{F}^{2} + \frac{600m}{\underline{\sigma}} \|\mcal{W}_l - \mcal{T}^*\|_{F}^{3}\\
&\leq \frac{1}{\nu_l}\|\mcal{W}_l - \mcal{T}^*\|_{\scr{W}_l}^2+ \frac{1}{\nu_l^{\frac{3}{2}}} \frac{600m}{\underline{\sigma}} \|\mcal{W}_l - \mcal{T}^*\|_{\scr{W}_l}^{3},
\end{aligned}
$$
where the first inequality is obtained from Lemma \ref{bound-ttsvd}. Next, we consider the estimation of  $\|\mcal{W}_l - \mcal{T}^*\|_{\scr{W}_l}^2$.
\begin{equation}\label{eq:i123}
\begin{aligned}
    \|\mcal{W}_l - \mcal{T}^*\|_{\scr{W}_l}^2 &= \|\mcal{T}_l - \alpha_l \tsp\scr{W}_l^{-1}\scr{P}_{\Omega}(\mcal{T}_l - \mcal{T}^*) - \mcal{T}^*\|_{\scr{W}_l}^2\\
    &= \| \scr{W}_l^{\frac{1}{2}}(\mcal{T}_l - \mcal{T}^*) - \alpha_l \scr{W}_l^{\frac{1}{2}} \tsp  \scr{W}_l^{-1} \scr{P}_{\Omega}(\mcal{T}_l - \mcal{T}^*) \|_F^2\\
    &= \| \scr{W}_l^{\frac{1}{2}}(\mcal{T}_l - \mcal{T}^*) - \alpha_l \scr{P}_{\wht{\bb{T}}_l} \scr{W}_l^{-\frac{1}{2}} \scr{P}_{\Omega}(\mcal{T}_l - \mcal{T}^*) \|_F^2\\
    &=\|\scr{W}_l^{\frac{1}{2}}(\mcal{T}_l - \mcal{T}^*) - \alpha_l p \epsilon_l^{-\frac{1}{2}} \scr{P}_{\wht{\bb{T}}_l}\scr{W}_l^{\frac{1}{2}}(\mcal{T}_l - \mcal{T}^*) - \alpha_l \scr{P}_{\wht{\bb{T}}_l}(\scr{W}_l^{-\frac{1}{2}} \scr{P}_{\Omega}\scr{W}_l^{-\frac{1}{2}} - p \epsilon_l^{-\frac{1}{2}} \scr{I})\scr{W}_l^{\frac{1}{2}}(\mcal{T}_l - \mcal{T}^*) \|_F^2\\
    &= \|\scr{W}_l^{\frac{1}{2}}(\mcal{T}_l - \mcal{T}^*) - \alpha_l p \epsilon_l^{-\frac{1}{2}} \scr{P}_{\wht{\bb{T}}_l}\scr{W}_l^{\frac{1}{2}}(\mcal{T}_l - \mcal{T}^*)\|_F^2\\
    &\quad -2\alpha_l \left\langle\scr{W}_l^{\frac{1}{2}}(\mcal{T}_l - \mcal{T}^*) - \alpha_l p \epsilon_l^{-\frac{1}{2}} \scr{P}_{\wht{\bb{T}}_l}\scr{W}_l^{\frac{1}{2}}(\mcal{T}_l - \mcal{T}^*), \scr{P}_{\wht{\bb{T}}_l}(\scr{W}_l^{-\frac{1}{2}} \scr{P}_{\Omega}\scr{W}_l^{-\frac{1}{2}} - p \epsilon_l^{-\frac{1}{2}} \scr{I})\scr{W}_l^{\frac{1}{2}}(\mcal{T}_l - \mcal{T}^*)   \right\rangle\\
    &\quad + \alpha_l^2 \| \scr{P}_{\wht{\bb{T}}_l}(\scr{W}_l^{-\frac{1}{2}} \scr{P}_{\Omega}\scr{W}_l^{-\frac{1}{2}} - p \epsilon_l^{-\frac{1}{2}} \scr{I})\scr{W}_l^{\frac{1}{2}}(\mcal{T}_l - \mcal{T}^*)  \|_F^2\\
    &= \underbrace{\|\scr{W}_l^{\frac{1}{2}}(\mcal{T}_l - \mcal{T}^*) - \alpha_l p \epsilon_l^{-\frac{1}{2}} \scr{P}_{\wht{\bb{T}}_l}\scr{W}_l^{\frac{1}{2}}(\mcal{T}_l - \mcal{T}^*)\|_F^2}_{I_1} \\
    &\quad + (2\alpha_l^2 p \epsilon_l^{-\frac{1}{2}} - 2 \alpha_l) \underbrace{ \left\langle \scr{P}_{\wht{\bb{T}}_l}\scr{W}_l^{\frac{1}{2}}(\mcal{T}_l - \mcal{T}^*), \scr{P}_{\wht{\bb{T}}_l}(\scr{W}_l^{-\frac{1}{2}} \scr{P}_{\Omega}\scr{W}_l^{-\frac{1}{2}} - p \epsilon_l^{-\frac{1}{2}} \scr{I})\scr{W}_l^{\frac{1}{2}}(\mcal{T}_l - \mcal{T}^*) \right\rangle }_{I_2}\\
    &\quad +  \alpha_l^2 \underbrace{ \| \scr{P}_{\wht{\bb{T}}_l}(\scr{W}_l^{-\frac{1}{2}} \scr{P}_{\Omega}\scr{W}_l^{-\frac{1}{2}} - p \epsilon_l^{-\frac{1}{2}} \scr{I})\scr{W}_l^{\frac{1}{2}}(\mcal{T}_l - \mcal{T}^*) \|_F^2 }_{I_3}.
\end{aligned}
\end{equation}
Using the estimation \eqref{eq:I1}, \eqref{eq:I2} and \eqref{eq:I3} in Appendix \ref{appendixB},  we obtain
$$
\begin{aligned}
\|\mcal{W}_l - \mcal{T}^*\|_{\scr{W}_l}^2 &\leq \left[\big( (1-\eta_l)^2 + (2- \eta_l)\eta_l  \frac{400}{600000}\big) \rho_l \nu_l +   0.71  (2\alpha_l^2 p \epsilon_l^{-\frac{1}{2}} - 2 \alpha_l)  p  + 0.354 \alpha_l^2 p^2\nu_l^{-1} \right] \|\mcal{T}_l - \mcal{T}^*\|_F^2\\
&\leq 0.3571 \epsilon_l^{\frac{1}{2}} \|\mathcal{T}_l - \mcal{T}^*\|_F^2.
\end{aligned}
$$
where we choose $\alpha_l = 1.001 \epsilon_l^{\frac{1}{2}} p^{-1}$, then we know $\eta_l =  \alpha_l p \epsilon_l^{-\frac{1}{2}} = 1.001$. Therefore, we have
\begin{equation}\label{eq:T-T*-norm}
\begin{aligned}
\|\mcal{T}_{l+1} - \mcal{T}^*\|^2_{F}
&\leq \frac{1}{\nu_l}\|\mcal{W}_l - \mcal{T}^*\|_{\scr{W}_l}^2 + \frac{1}{\nu_l^{\frac{3}{2}}} \frac{600m}{\underline{\sigma}} \|\mcal{W}_l - \mcal{T}^*\|_{\scr{W}_l}^{3}\\
&\leq 0.3571 \|\mathcal{T}_l - \mcal{T}^*\|_F^2 + \frac{600m}{\underline{\sigma}} 0.3571^{\frac{3}{2}}\|\mathcal{T}_l - \mcal{T}^*\|_F^3\\
&\leq 0.3571(1 + 0.001\sqrt{0.3571}) \|\mcal{T}_l - \mcal{T}^*\|_F^2\\
&\leq 0.3574 \|\mcal{T}_l - \mcal{T}^*\|_F^2,
\end{aligned}
\end{equation}
where the second inequality follows from the fact that $\| \mcal{T}_l - \mcal{T}^* \|_F \leq \frac{\underline{\sigma}}{600000 m}$ due to the initialization.
\end{proof}

\section{Numerical Experiments}\label{section: numerical experiments}
In this section, we compare our proposed PRGD algorithm with the RGD algorithm \cite{cai2022provable} for the tensor completion problem in the TT format. For a full comparison, we implement two step-size selection strategies, the adaptive step size and optimal constant step size, for both of these two algorithms. 

We set different oversampling(OS) ratios to investigate the recovery ability of the tested algorithms. The OS ratio is the ratio of the number of samples to the dimension of the constraints. For $m$-order tensor $\mcal{T}^*\in \mathbb{R}^{d_1\times \dots\times d_m}$ with TT-rank $(1, r_1,\dots, r_{m-1}, 1)$, the OS ratio is 
$$
\operatorname{OS}: = \frac{|\Omega|}{\operatorname{dim}(\bb{M}_{\mbf{r}}^{tt})} = \frac{|\Omega|}{\sum_{k=1}^m r_{k-1} d_k r_k - \sum_{k=1}^{m-1} r_k^2}.
$$

\subsection{Synthetic Data}

We first test those algorithms on synthetic tensors with size $d\times d\times d$ and TT rank $(r, r)$. The ground truth low TT-rank tensor $\mcal{T}^*$ is constructed by the TT format, where the entries of each component are generated from a uniform distribution in $[0, 1]$. Given sampling size $n = \operatorname{OS}\times \operatorname{dim}(\bb{M}_{\mbf{r}}^{tt})$, the observed entries set $\Omega$ is obtained by sampling uniformly at random. We compared the CPU times (in seconds) and the iteration numbers of different algorithms to obtain the relative error $(\|\mcal{T}_l - \mcal{T}^*\|_F)/\|\mcal{T}^*\|_F\leq 10^{-4}$. For each experiment, We conduct five random trials and report the average results.

\textbf{Sensitivity to the Size of Tensor.} Firstly, we investigate the sensitivity of PRGD to changes in tensor size. Here $\operatorname{OS} = 7$, rank $r= 5$ are fixed and the tensor size $d$ is varied from $100$ to $600$. The results are shown in \cref{subfig: Synthetic size CPU} and \cref{subfig: Synthetic size iteration}.

\begin{figure}[H]
	\centering
	\subfigure[{\tt CPU-time(s) $r=5$, $m=3$, and $\operatorname{OS}=7$}]{ \includegraphics[width=0.43\linewidth]{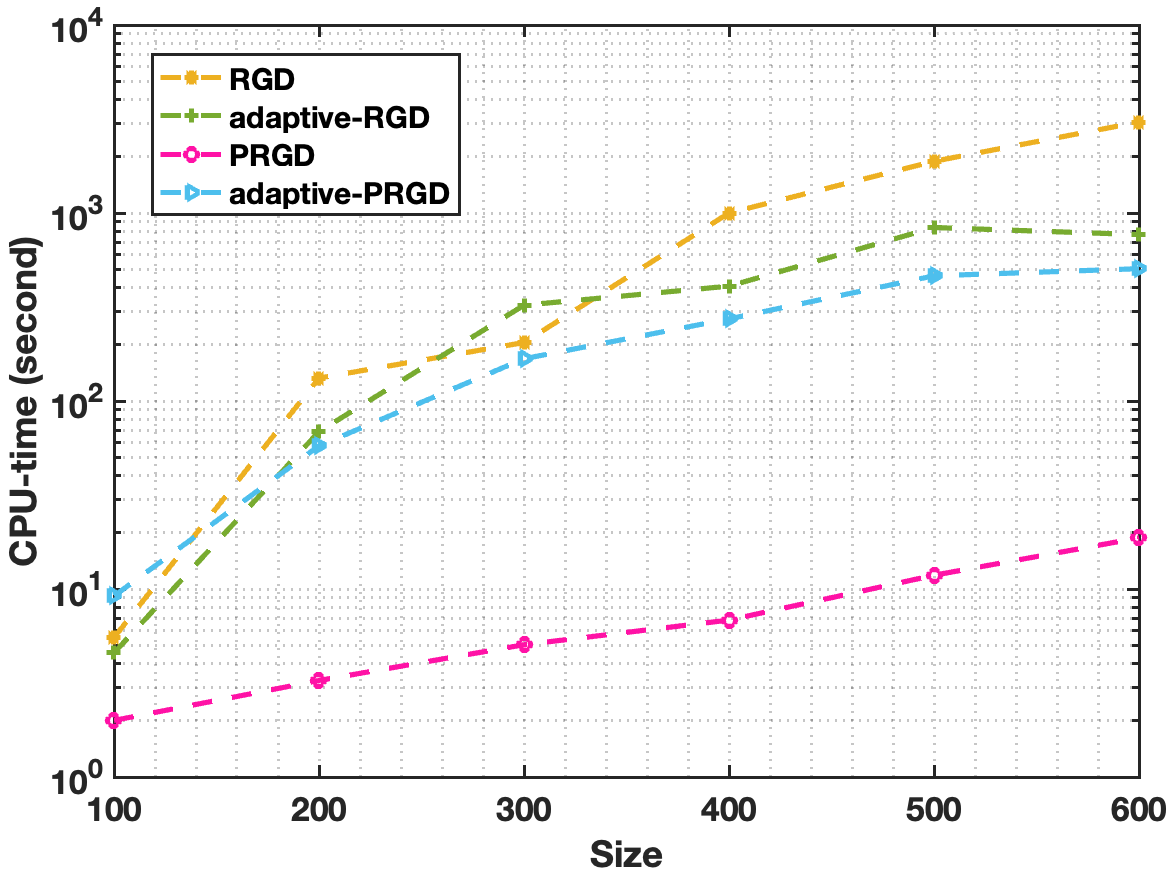} 
    \label{subfig: Synthetic size CPU}} 		
	\subfigure[{\tt $\sharp$iteration $r=5$, $m=3$, and $\operatorname{OS}=7$}]{ \includegraphics[width=0.43\linewidth]{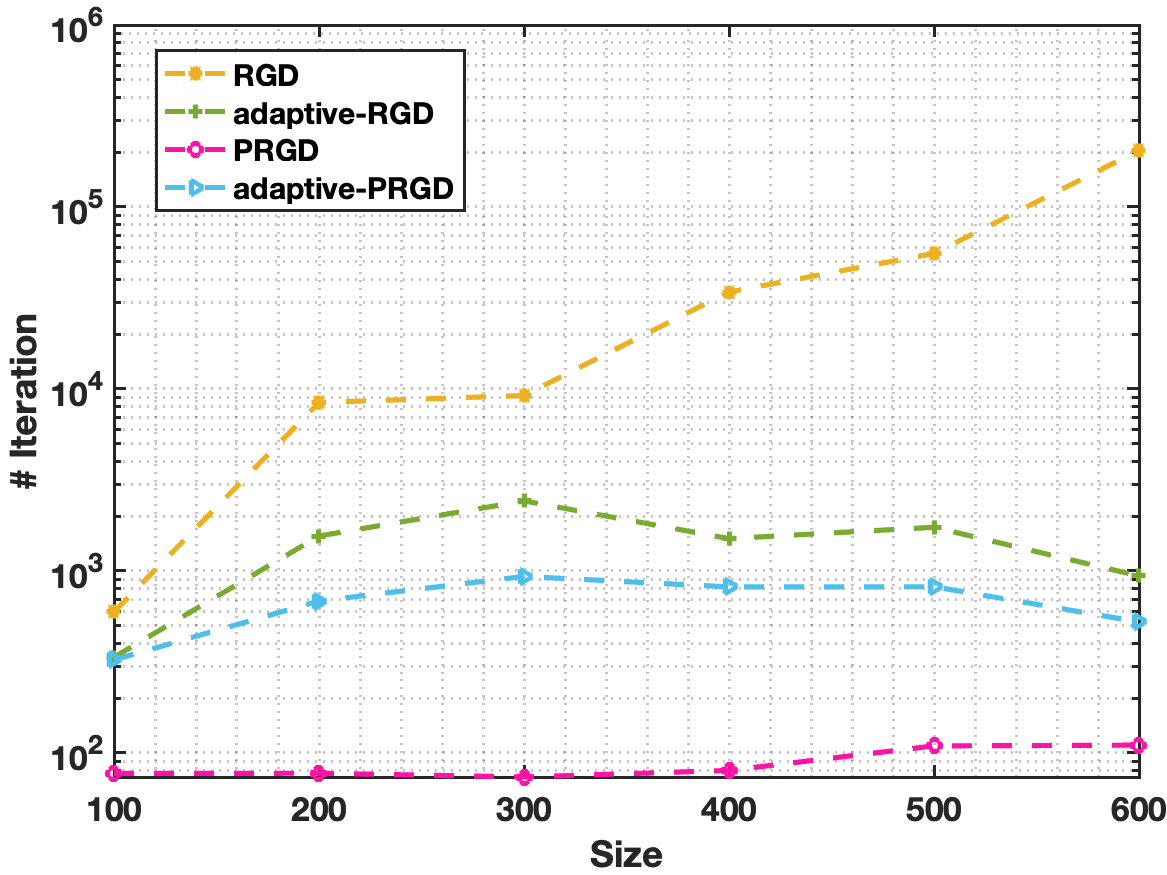}\label{subfig: Synthetic size iteration}} 		\\[-2mm]	
	\caption{ Results of CPU-time (in seconds) and the number of iterations for TT-rank tensor completion on simulated data.}
    \label{res-size}
\end{figure}
\cref{subfig: Synthetic size iteration} shows that PRGD generally requires fewer iterations than RGD, and PRGD with optimal constant step size is much faster than the other algorithms. Moreover, as analyzed in \cref{subsec: Tg Para and Tp Comp}, the additional operations related to data-driven metric in PRGD are much smaller than the dominant part, compared with RGD. As a result, we can observe from \cref{subfig: Synthetic size CPU} that the total computation time of PRGD with optimal constant step size is significantly less than the other algorithms. As for adaptive step size cases, since more computation is needed for step size calculation, RGD and PRGD with adaptive step size usually take more computation time. Notably, when tensor size $d = 600$, the computation time of PRGD with optimal constant step size is three orders of magnitude less than RGD with optimal constant step size.

\textbf{Sensitivity to the Oversampling Ratio.} Secondly, we investigate the sensitivity of PRGD to changes in oversampling ratio $\operatorname{OS}$. We remark that $\operatorname{OS}$ is a suitable criterion to evaluate the difficulty of the low-rank tensor completion problem. A smaller $\operatorname{OS}$ implies a more challenging completion problem since there are fewer observed entries available. We fix tensor size $d = 300$, rank $r = 5$ and vary the $\operatorname{OS}$ from $4$ to $12$. The results are shown in \cref{subfig: Synthetic OS CPU} and \cref{subfig: Synthetic OS iteration}.

\begin{figure}[H]
	\centering
	\subfigure[{\tt CPU-time(s) $r=5$, $m=3$, and $d=300$}]{ \includegraphics[width=0.43\linewidth]{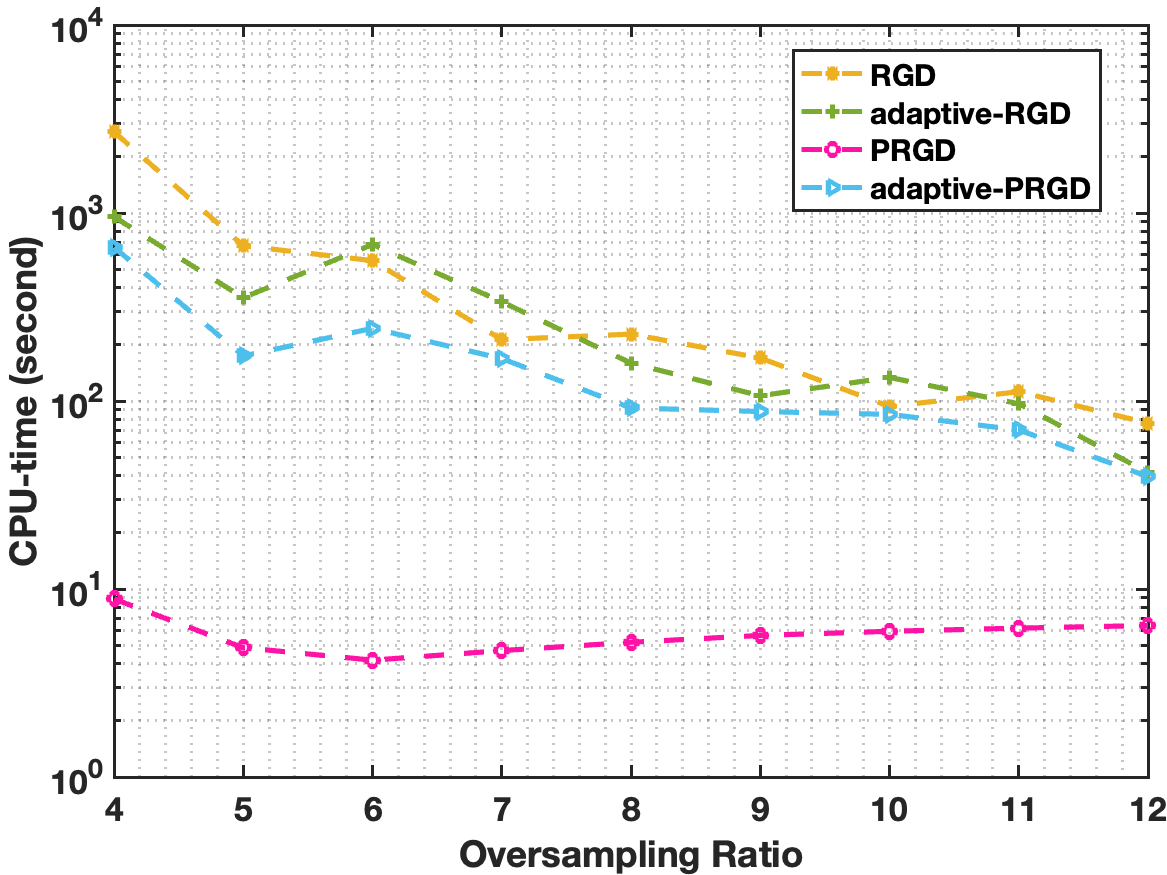} \label{subfig: Synthetic OS CPU}} 		
	\subfigure[{\tt $\sharp$iteration $r=5$, $m=3$, and $d=300$}]{ \includegraphics[width=0.43\linewidth]{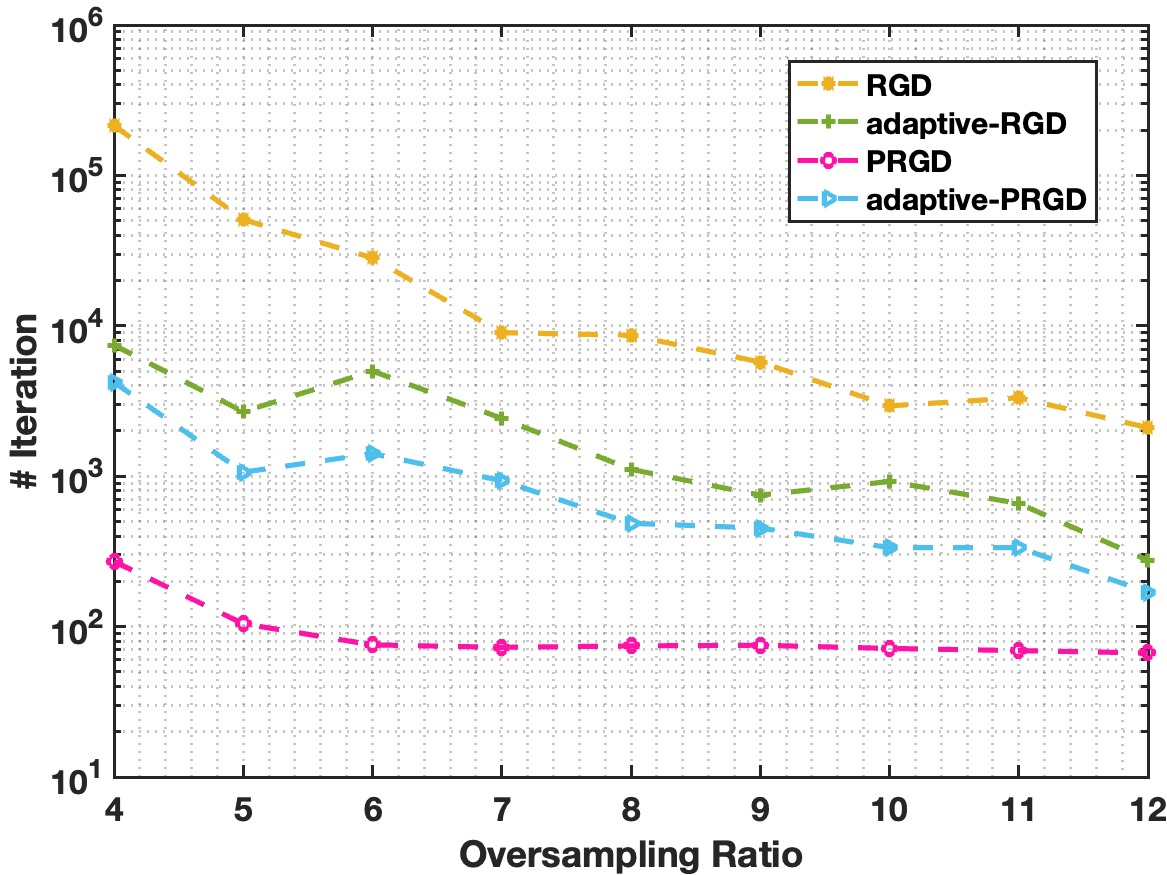}\label{subfig: Synthetic OS iteration} } 		\\[-2mm]	
	\caption{ Results of CPU-time (in seconds) and the number of iterations for TT-rank tensor completion on simulated data.}
    \label{res-os}
\end{figure}

\cref{subfig: Synthetic OS CPU} and \cref{subfig: Synthetic OS iteration} again confirm that PRGD outperforms the other algorithms. Especially when the OS is smaller, i.e., the problem is more difficult, PRGD with optimal constant step size significantly outperforms the other methods. For instance, when $\operatorname{OS} = 4$, PRGD with optimal constant step size is at around three orders of magnitude faster than other algorithms.

\textbf{Sensitivity to the TT-Rank.} Thirdly, we investigate the performance of PRGD under different TT ranks. We fix tensor size $d = 300$, $\operatorname{OS} = 7$ and varying the rank $r$ from $4$ to $10$. The results are shown in \cref{subfig: Synthetic rank CPU} and \cref{subfig: Synthetic rank iteration}.
\begin{figure}[H]
	\centering
	\subfigure[{\tt CPU-time(s) $\operatorname{OS}=7$, $m=3$, and $d=300$}]{ \includegraphics[width=0.43\linewidth]{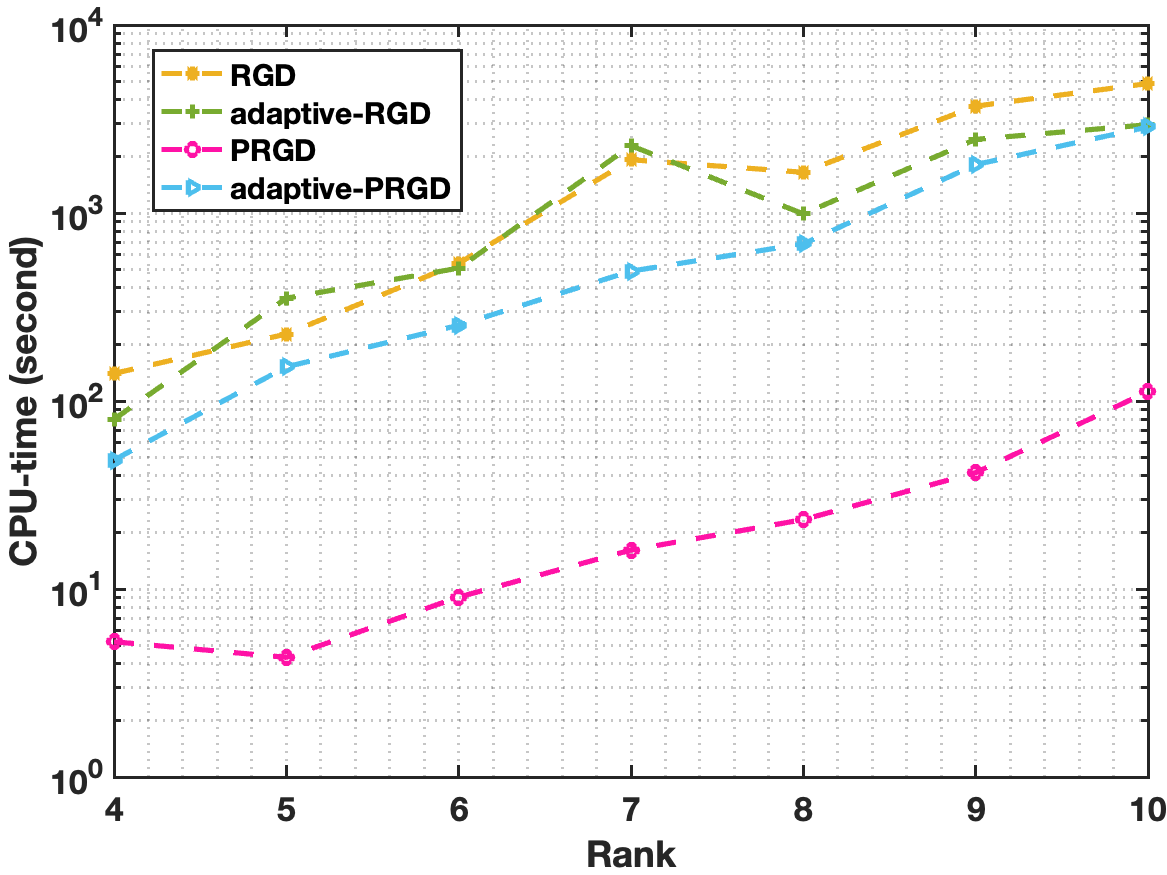} \label{subfig: Synthetic rank CPU}} 		
	\subfigure[{\tt $\sharp$iteration $\operatorname{OS}=7$, $m=3$, and $d=300$}]{ \includegraphics[width=0.43\linewidth]{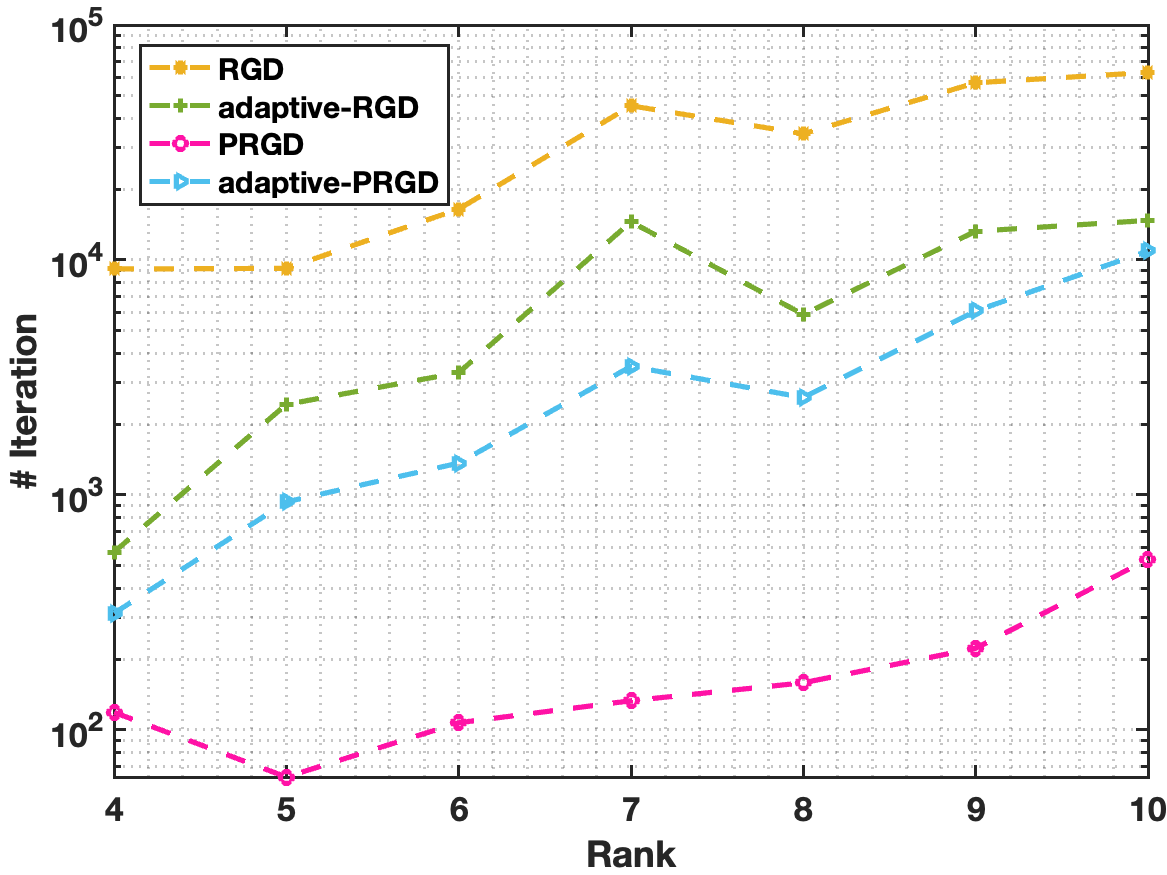} \label{subfig: Synthetic rank iteration}} 		\\[-2mm]	
	\caption{ Results of CPU-time (in seconds) and the number of iterations for TT-rank tensor completion on simulated data.}
    \label{res-rank}
\end{figure}

\cref{subfig: Synthetic rank CPU} and \cref{subfig: Synthetic rank iteration} show that, as the rank increases, the number of iterations and computation time required by all algorithms increase noticeably. While PRGD is consistently better than RGD and PRGD with optimal constant step size is better than RGD at around ten times.

\textbf{Robustness to Additive Noise.}
We consider a synthetic noisy tensor 
$$
\widehat{\mcal{T}}^{*} = \mcal{T}^{*} + \delta \cdot \mcal{E}
$$

where $\mcal{T}^*$ is the underlying low TT rank tensor, $\mcal{E}$ is a tensor with its entries being sampled from $\mcal{N}(0, 1)$. $\delta = \sigma \cdot \frac{\|\scr{P}_{\Omega}(\mcal{T}^*)\|_F}{\|\mcal{E}\|_F}$ represents the noise level. We stop the algorithm when $\|\mcal{T}_{t+1} - \mcal{T}_t\|_F\leq 10^{-5}\cdot\max(1, \|\mcal{T}_t\|_F)$. The robustness of the tested algorithm can be indicated by whether the relative error of stop point $\|\mcal{T}_l - \mcal{T}^*\|_F/\|\mcal{T}^*\|_F$ less than $\sigma$ or not.

We test all algorithms on tensor with $d = 100$, $r = 4$, $\operatorname{OS}=6$ and vary $\sigma = [1\times 10^{-4}, 5\times 10^{-4}, 1\times 10^{-3}, 5\times 10^{-3}, 1\times 10^{-2}, 5\times 10^{-2}, 1\times 10^{-1}]$. The results are shown in \cref{subfig: Robust_Iteration} and \ref{subfig: Robust_reler}. From \cref{subfig: Robust_reler}, it is shown that all tested algorithms are robust to the additive noise. The recovery error is proportional to the noise level. From \cref{subfig: Robust_Iteration}, it is shown that PRGD with optimal step size is around ten times faster than RGD algorithms.

\begin{figure}[H]
	\centering
	\subfigure[{\tt Relative error}]{ \includegraphics[width=0.43\linewidth]{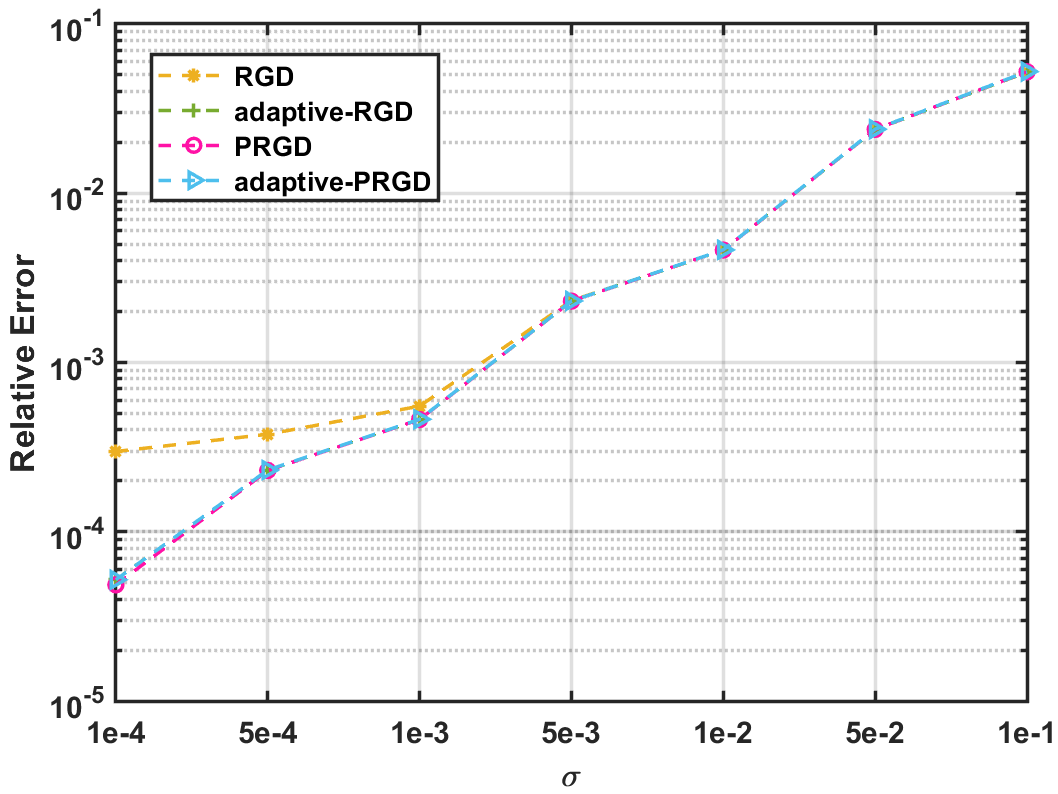}\label{subfig: Robust_Iteration}} 		
	\subfigure[{\tt $\sharp$iteration}]{ \includegraphics[width=0.43\linewidth]{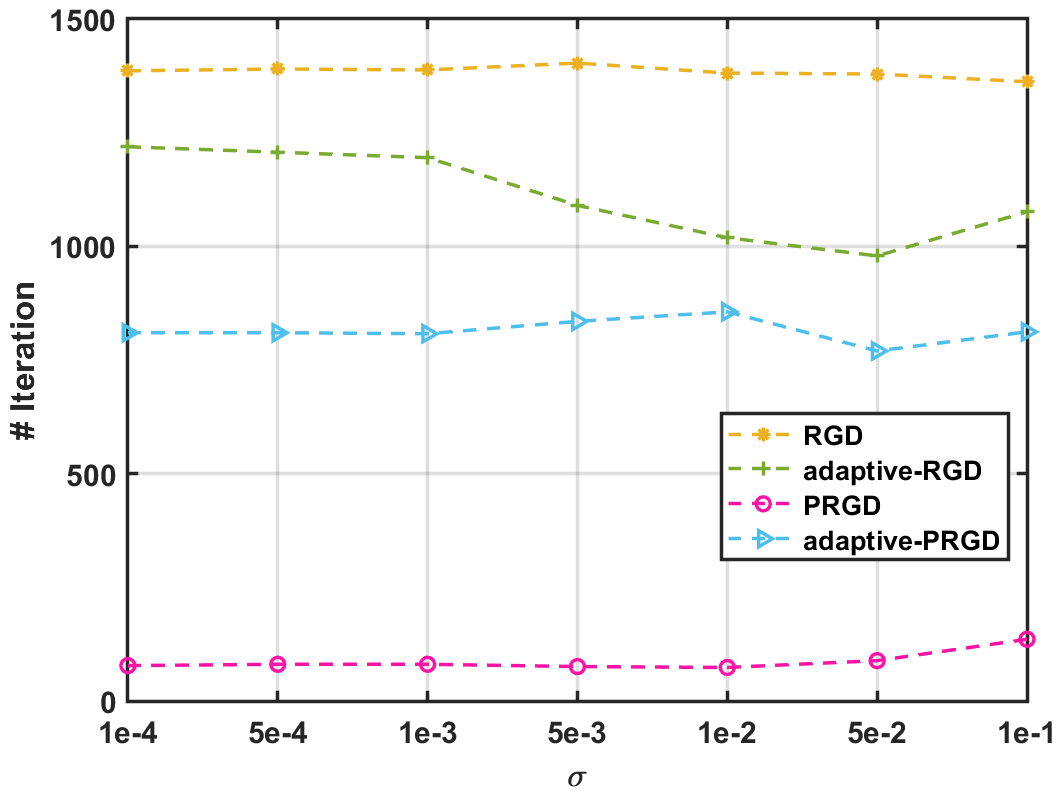}\label{subfig: Robust_reler} } 	\\[-2mm]	
	\caption{ Results of relative error and the number of iterations for TT-rank tensor completion on different noise levels.}
    \label{robust}
\end{figure}


\subsection{Hyperspectral Images}
We consider the low TT rank tensor completion on hyperspectral images where the dataset is from ``50 reduced hyperspectral reflectance images" \cite{foster2022colour}. The hyperspectral images from this dataset are three-order tensors with size $\mbf{d} = (d_1,d_2,d_3) = (250, 329, 33)$ where $(250, 329)$ is the size of images and $33$ represents the wavelength values from $[400, 410, \dots, 720]$. We select two images from this dataset as ground truths and conduct the low-rank tensor completion experiments on each image.

We use peak signal-to-noise-ratio (PSNR) to evaluate the image completion performance, defined by 
$$
\operatorname{PSNR}:=10\log_{10}\left(n_1 n_2 n_3\frac{\max(\mcal{T^*})^2}{\|\mcal{T-T^*}\|_F^2}\right)
$$
where $\max (\mcal{T}^*)$ denotes the highest pixel value of $\mcal{T}^*$. We set the TT rank $(1, 30, 10, 1)$ and vary $\operatorname{OS} = [4, 5, 6]$ to investigate the performance of PRGD under different sample levels. The results for completing two hyperspectral images are shown in \cref{fig: hsi}. 
\begin{figure}[H]
    \centering
    \includegraphics[width = 0.95\linewidth]{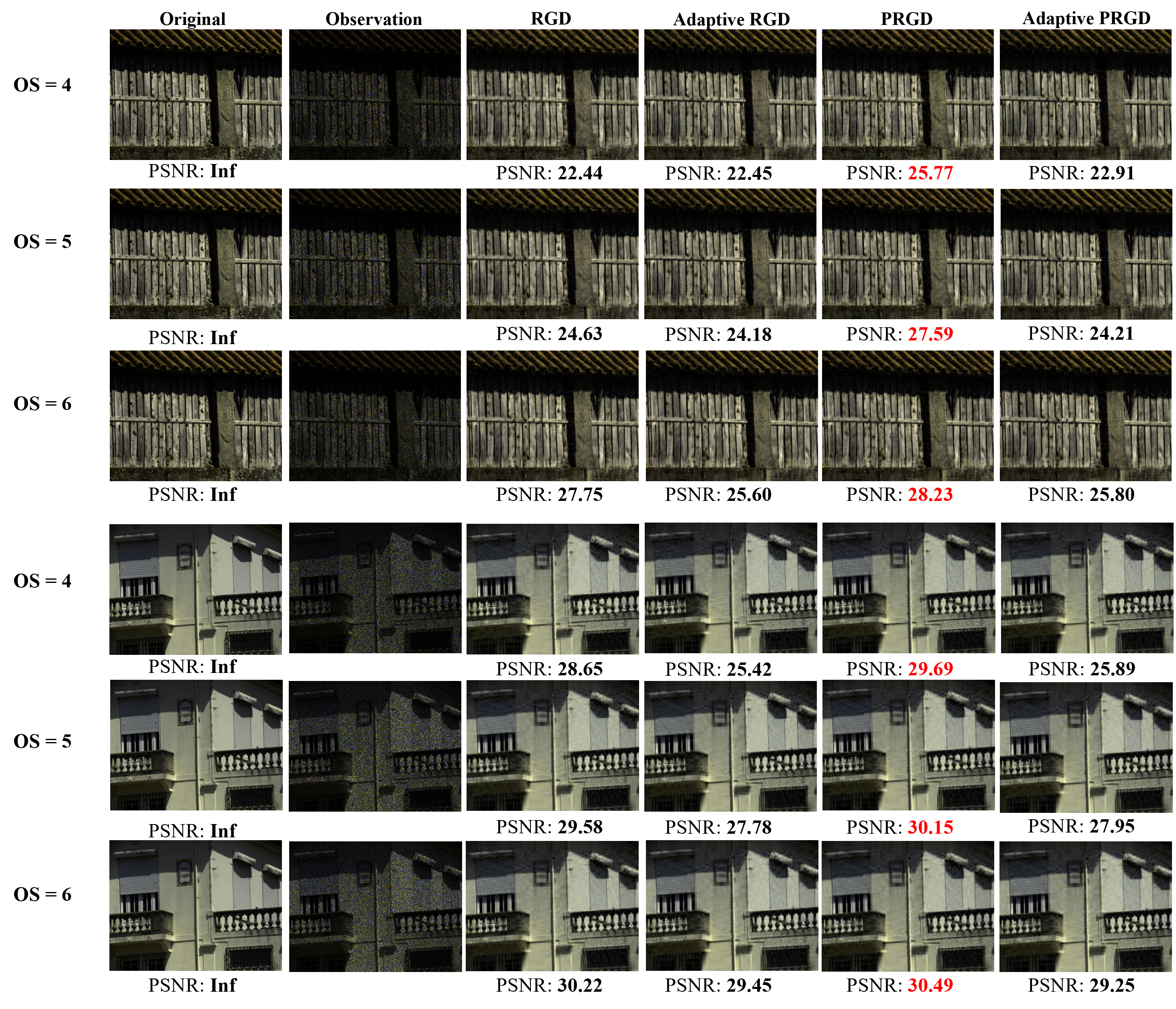}
    \caption{RGB representations of hyperspectral images.}
    \label{fig: hsi}
\end{figure}
It is shown that PRGD consistently outperforms the other algorithms across oversampling levels of $4,5,6$, especially when $\operatorname{OS}$ is small. For $\operatorname{OS} = 4$, PRGD with optimal constant step size reaches PSNR values of $25.77$ on the first image, the PSNR value is around 3 higher than other algorithms. The results indicate that PRGD is more effective at improving image reconstruction, particularly with low sampling levels.


\subsection{Quantum State Tomography}\label{subsec: qst}

We first formulate the quantum 
state tomography (QST) for matrix product operators (MPOs) under Pauli-measure as a low-TT-rank tensor completion problem, then test the PRGD algorithm on this completion problem. In a $n$-qubits quantum system, a quantum state can be fully described by a density matrix 
$\rho\in \mathbb{C}^{2^n\times 2^n}$, which is positive semidefinite and has unit trace. Specifically, such a density matrix is called a matrix product operator (MPO) if its elements can be expressed as the following matrix products:
$$
\rho(i_1\cdots i_n, j_n\cdots j_n) = X_1(i_1, j_1, :)X_2(:, i_2, j_2, :)\cdots X_n(:, i_n, j_n)
$$
where $X_k\in \mathbb{C}^{r_{k-1}\times 2\times 2\times  r_k}$ and $\mathbf{r} = (1, r_1, \dots, r_{n-1}, 1)$ is called the bond dimension of $\rho$, here we denote $\rho = [X_1, \dots, X_n]$. The task of QST is to recover the density matrix $\rho$ from a few measurements. A typical method is the Pauli measurement, which is the inner product of the Pauli basis and density matrix. We denote the Pauli basis set as $\{W_\mbf{a} := \sigma_{a_1}\otimes \sigma_{a_2}\otimes \cdots\otimes \sigma_{a_n}, \mathbf{a} = (a_1, \dots, a_n)\in [4]^n\}$, where

$$\sigma_1 = \left[\begin{array}{cc}
    1 & 0 \\
    0 & 1
\end{array}\right], \sigma_2 = \left[\begin{array}{cc}
    0 & 1 \\
    1 & 0
\end{array}\right],  \sigma_3 = \left[\begin{array}{cc}
    0 & -i \\
    i & 0
\end{array}\right], \sigma_4 = \left[\begin{array}{cc}
    1 & 0 \\
    0 & -1
\end{array}\right]$$
are $2\times 2$ Pauli matrices. All of the Pauli measurements construct a real tensor $\mcal{T}^*$ where $\mathcal{T}^*(\mbf{a}):= \langle \rho, W_\mbf{a}\rangle = \operatorname{Tr}(\rho W_\mbf{a})$. It is easy to verify that
$$
\mathcal{T}^*(\mbf{a}) = Y_1(a_1, :)Y_2(:, a_2, :)\cdots Y_n (:, a_n)
$$
where $Y_k(:, \alpha_k, :) = \sum_{i_k, j_k} X_k(:, i_k, j_k, :) \sigma_{\alpha_k}(i_k, j_k)$. It indicates that $\mcal{T}^*$ is an order-$n$ tensor with TT rank $\mbf{r}$. In QST, one needs to sample some measure matrices from $\{W_{\mbf{a}}\}$ and reconstruct $\rho$ from corresponding measurements. Denote $\Omega$ as the indices of sampled Pauli basis, then the measurement it is equal to $\scr{P}_{\Omega} (\mcal{T}^*)$. Thus we can conduct tensor completion of $\mcal{T}^*$ and obtain the reconstructed $\rho$ immediately after completion since $\{W_{\mbf{a}}\}$ are the orthogonal basis.

Regarding the generation of target MPO, we use the same method as in \cite{verstraete2004matrix,qin2024quantum}. We first generate a random matrix product state $u\in \mathbb{C}^{2^n\times 1}$ where $u=[U_1, \dots, U_n], U_k\in \mathbb{C}^{s_{k-1}\times 2\times s_k}$. Then the density matrix $\rho = u u^\dagger$ is an MPO with $X_k(:, i_k, j_k, :) = U_k(:, i_k, :)\otimes \overline{U_k(:, j_k, :)}$ and $r_k = s_k^2$. After that, we compute the corresponding measurement tensor $\mcal{T}^*$ and conduct the completion tasks. In the experiment, we set the qubit number $n=10$, the TT rank $r_i = 4, i=1, \dots, n-1$. The size of tensor $\mcal{T}^*$ grows exponentially with $n$, while the dimension of the manifold grows polynomially. Thus we need to set the $\operatorname{OS}$ ratio marginally large to make the problem solvable. We set $\operatorname{OS} = 200$ and the corresponding sampling ratio is around $7.6\%$. The numerical results are shown in \cref{subfig: QST_Iteration} and \cref{subfig: QST_CPU time}.

\begin{figure}[H]
    \centering 		
    \subfigure[{\tt $\sharp$iteration} $r = 4, n = 10, \operatorname{OS} = 200$]{ \includegraphics[width=0.43\linewidth]{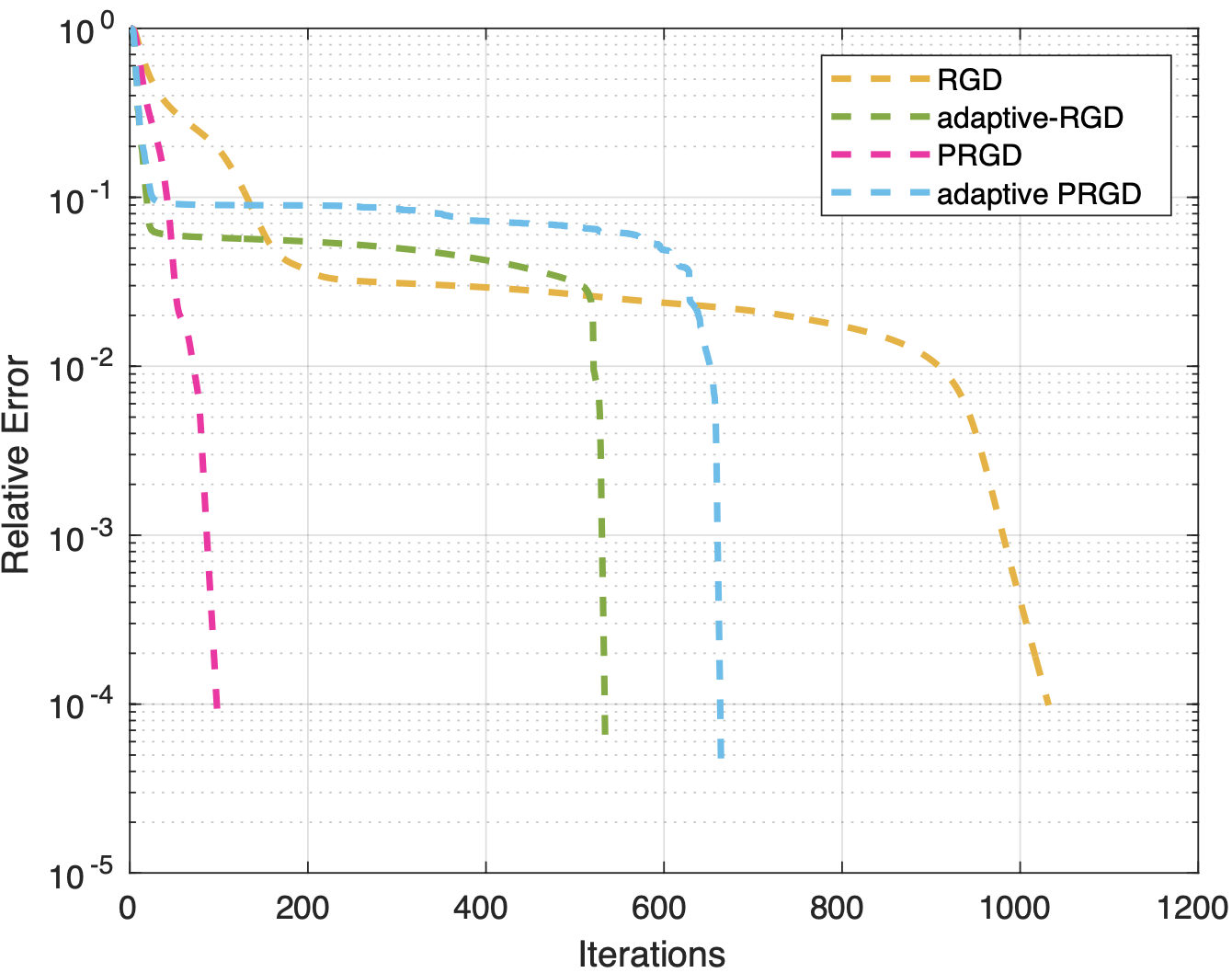}\label{subfig: QST_Iteration} } 		
    \subfigure[{\tt CPU-time (s)} $r = 4, n = 10, \operatorname{OS} = 200$]{ \includegraphics[width=0.43\linewidth]{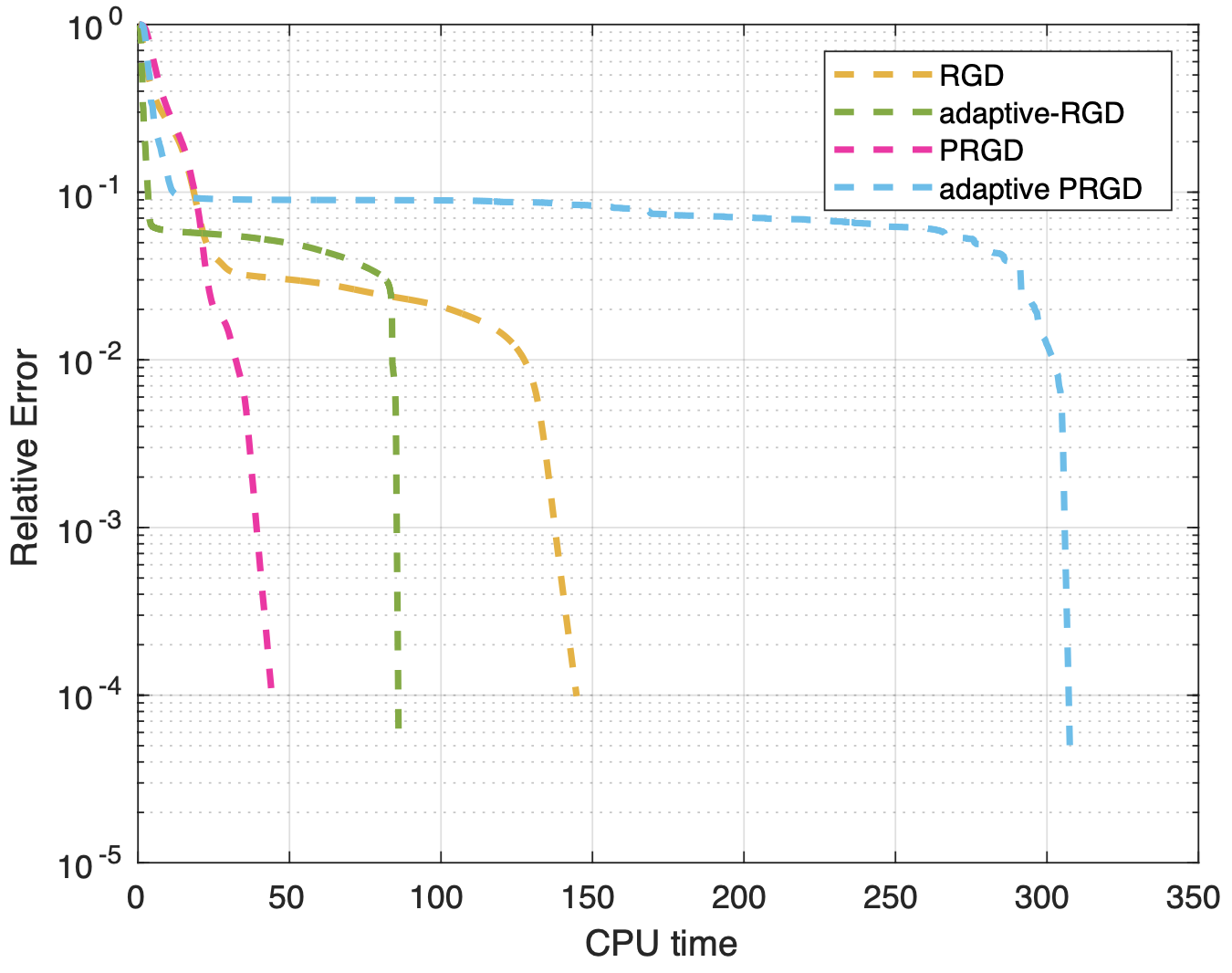}\label{subfig: QST_CPU time}} \\[-2mm]
    \label{fig: QST}
\end{figure}

As shown in \cref{subfig: QST_Iteration}, PRGD achieves convergence approximately $5$ and $10$ times faster than RGD with adaptive step size and optimal constant stepsize, respectively. However, when considering CPU time, as presented in \cref{subfig: QST_CPU time}, the advantage of PRGD with an optimal constant step size diminishes, and the adaptive step size version of PRGD becomes the slowest. This performance difference arises from the increased computational cost of the preconditioning step in PRGD, which grows with the tensor order, thereby making both the iterative process and adaptive step sizes computationally more expensive.

\section{Conclusion}
\label{sec:conclusion}
In this paper, we proposed a preconditioned Riemannian gradient descent algorithm (PRGD) for solving the low TT-rank tensor completion problem. This algorithm utilizes an adaptive data-driven preconditioner, which is computationally simple and efficient and can be seen as an extension of the preconditioner in matrix case \cite{bian2023preconditioned}. Theoretical analysis indicates that PRGD can converge linearly to the target tensor.  Notably, the convergence rate is independent of the condition number of the target tensor, meaning that the algorithm maintains efficient convergence performance even if the target tensor has a high condition number. To validate the effectiveness of the PRGD algorithm,  we conducted a series of numerical experiments on both synthetic and real datasets, including applications such as hyperspectral image completion and quantum state tomography. The experimental results show that compared to the Riemannian gradient descent algorithm \cite{cai2022provable}, the PRGD algorithm significantly improves computational efficiency, reducing computation time by two orders of magnitude. This fully substantiates the effectiveness of the proposed data-driven preconditioner in enhancing the performance of RGD for tensor completion.

\appendix
\section{Proofs of Lemma \ref{lemma:bound of tildet} and Lemma \ref{lemma: tilde delta norm}}
\label{appendixA}
Before providing the proof of these lemmas, we will review some lemmas from \cite{cai2022provable} for readability, as they will also be used in the following proof. We first introduce several events whose randomness stems from the sampling set $\Omega$, which will be useful in the proof. 
$$
\begin{aligned}
\bm{\mcal{E}}_1 &=\big\{\|\frac{d^*}{n} \scr{P}_{\bb{T}^*} \scr{P}_{\Omega} \scr{P}_{\bb{T}^*} - \scr{P}_{\bb{T}^*} \| \leq \frac{1}{2} \big\},\\
\bm{\mcal{E}}_2 &= \big\{ \max_{x \in [d^*]} \sum_{i=1}^n \scr{I}(\omega_i = x) \leq 2m \log(\bar{d})\big\},\\
\bm{\mcal{E}}_3 &= \big\{\| (\scr{P}_{\Omega} - \frac{n}{d^*} \scr{I})(\mcal{J}) \| \leq C_m \left( \sqrt{\frac{n \bar{d}}{d^*}} +1 \right)\log^{m+2}(\bar{d}) \big\},
\end{aligned}
$$
where $\mcal{J} \in \bb{R}^{d_1 \times \cdots \times d_m}$ is the tensor with all its entries one and $\scr{I}$ is the identity operator from $\bb{R}^{d_1 \times \cdots \times d_m}$ to $\bb{R}^{d_1 \times \cdots \times d_m}$. From Lemmas 31, 32, and 33 in \cite{cai2022provable}, we know that  $\bm{\mcal{E}}_1, \bm{\mcal{E}}_2$ and $\bm{\mcal{E}}_3$ all hold with probability exceeding $1 - \bar{d}^{-m}$. Consider the empirical process:
$$
\beta_n(\gamma_1, \gamma_2):= \sup_{\mcal{A}\in \mathbb{K}_{\gamma_1, \gamma_2}} \bigg|\langle \mcal{P}_{\Omega}\mcal{A}, \mcal{A} - \frac{n}{d^*}\| \mcal{A}\|_F^2 \bigg|,
$$
where $\mathbb{K}_{\gamma_1, \gamma_2}=\{\mcal{A} \in \mathbb{R}^{d_1 \times \dots \times d_m}: \|\mcal{A}\|_F \leq 1, \| \mcal{A}\|_{l_{\infty}} \leq \gamma_1, \|\mcal{A}\|_* \leq \gamma_2\}$. Denote
$$
\bm{\mcal{E}}_4 = \big\{\beta_n(\gamma_1, \gamma_2) \lesssim_{m} \gamma_1\gamma_2 \big(\sqrt{\frac{n\bar{d}}{d^*}} + 1) \log^{m+2}(\bar{d}) + \gamma_1 \sqrt{\frac{n\log(\bar{d})}{d^*}} + \gamma_1^2 \log(\bar{d})\big\},
$$
for all $\gamma_1 \in [1/d^*, 1]$ and $\gamma_2 \in [1, \bar{d}^{(m-1)/2}]$. From \cite{cai2022provable}, we know that the event $\bm{\mcal{E}}_4$ holds with probability exceeding $1 - (m+4) \bar{d}^{-m}$. Denote
$$
\bm{\mcal{E}}_5^{i} = \big\{\|\scr{P}^{(i)}(\scr{P}_{\Omega} - \frac{n}{d^*} \scr{I}) \scr{P}^{(i)} \| \leq C_m \sqrt{\frac{\mu^2 \bar{r}^2 \bar{d} n \log(\bar{d})}{(d^*)^2}}\big\} 
$$
for all $i \in [m]$. Here, $\scr{P}^{(i)}$ denotes a projection operator that is independent of the chosen left orthogonal representation of $\mcal{T}^{*}$. Defining $\bm{\mcal{E}}_5 = \cap_{i=1}^{m} \bm{\mcal{E}}_5^{i}$, it follows that $\bm{\mcal{E}}_5$ holds with a probability exceeding $1 - m \bar{d}^{-m}$ as  Lemma 11 in \cite{cai2022provable}. Here we denote $\bm{\mcal{E}} = \bm{\mcal{E}}_1 \cap \bm{\mcal{E}}_2 \cap \bm{\mcal{E}}_3 \cap \bm{\mcal{E}}_4 \cap \bm{\mcal{E}}_5$, then $\bm{\mcal{E}}$ holds with probability exceeding $1 - (m+4) \bar{d}^{-m}$. Here we assume that $\bm{\mcal{E}}$ holds throughout this paper. Next, we present the main theorem on exact matrix recovery and related lemmas.

\begin{lemma}[Corollary 12 in \cite{cai2022provable}]\label{coro12}
For any $i \in [m]$, set
$$
\mcal{A} =[T_1^*, \dots, T_{i-1}^*, A, T_{i+1}^*, \dots, T_m^*], ~~~ \mcal{B} =[T_1^*, \dots, T_{i-1}^*, B, T_{i+1}^*, \dots, T_m^*]
$$
for arbitrary $A, B \in \bb{R}^{r_{i-1} \times d_i \times r_i}$. Then under the event $\bm{\mcal{E}}_5$, we have
$$
\langle \mcal{A},~(\scr{P}_{\Omega} - \frac{n}{d^*}\scr{I}) \mcal{B} \rangle \leq C_{m} \sqrt{\frac{\mu^2 \bar{r}^2 \bar{d} n \log(\bar{d})}{(d^*)^2}} \|\mcal{A}\|_F\| \mcal{B} \|_F.
$$
\end{lemma}
\begin{lemma}[Lemma 26 in \cite{cai2022provable}]\label{bound-ttsvd}
Let $\mcal{T}^* \in \bb{M}_{r}^{tt}$ be an $m$-way tensor and define $\underline{\sigma}:= \min_{i=1}^{m-1} \sigma_{\min}(\mcal{T}^*)^{\langle i \rangle}$. And we denote the tensor $\mcal{T} = \mcal{T}^* + \mcal{D}$. Then suppose $\underline{\sigma} \geq C_m \| \mcal{D} \|_F$ for some constant $C_m \geq 500m$ depending only on $m$, we have
$$
\| \operatorname{SVD}_{\mathbf{r}}^{tt}(\mcal{T}) - \mcal{T}^* \|_F^2 \leq \| \mcal{D} \|_{F}^2 + \frac{600m \| \mcal{D} \|_F^3}{\underline{\sigma}}.
$$
\end{lemma}
\begin{lemma}[Lemma 13 in \cite{cai2022provable}]\label{Lema13}
Suppose that $\Omega$ is the set sampled uniformly with replacement with size $| \Omega | =n $. Then under the event $\bm{\mcal{E}}_3$, we have for any tensors $\mcal{A}, \mcal{B}$ with TT rank $(r_1, r_2, \dots, r_{m-1})$,
$$
\begin{aligned}
&|\langle (\scr{P}_{\Omega} - \frac{n}{d^*}\scr{I}) \mcal{A}, \mcal{B} \rangle|\\ 
&\leq C_m \big( \sqrt{\frac{n\bar{d}}{d^*}} +1 \big) \log^{m+2}(\bar{d}) \cdot \Pi_{i=1}^m
 \big(\max_{x_i} \| A_i (:, x_i, :) \|_F \cdot \| B_i \|_F \wedge \max_{x_i}\|B_i(:, x_i, :)\|_F \|A_i\|_F \big),
 \end{aligned}
$$
where $\mcal{A} =[A_1, \dots, A_m]$ and $\mcal{B}=[B_1, \dots, B_m]$ can be arbitrary decompositions such that $A_i, B_i \in \mathbb{R}^{r_{i-1}\times d_i \times r_{i}}$.
\end{lemma}
\begin{lemma}[Lemma 14 in \cite{cai2022provable}]\label{Lema14}
Let $\mcal{T}$ be a tensor of rank $(r_1, r_2, \dots, r_{m-1})$ such that $Incoh(\mcal{T}) \leq \sqrt{\mu}$, and it has a left orthogonal decomposition $\mcal{T} = [T_1, T_2, \dots, T_m]$ under the standard inner product. Then we have
$$
\max_{x_i} \| T_i (:, x_i, :) \|_F^2 \leq \frac{\mu r_i}{d_i}, \quad \| T_i \|_{F} \leq \sqrt{r_i}, ~~i \in [m-1],
$$
$$
\max_{x_m} \| T_m(:, x_m)\|_F^2 \leq \sigma_{\max}^2 (\mcal{T}) \frac{\mu r_{m-1}}{d_m}, \quad \| T_m \|_F = \| \mcal{T} \|_F \leq \sqrt{\underline{r}} \sigma_{\max}(\mcal{T}).
$$
\end{lemma}

\begin{lemma}[Lemma 27 in \cite{cai2022provable}]\label{Lema27}
Let $\mcal{T}, \mcal{T}^* \in \bb{M}_{r}^{tt}$ be two TT-rank $r$ tensors. Suppose we have $8\|\mcal{T} - \mcal{T}^* \|_F \leq \sigma$, then we have
$$
\| \scr{P}_{\bb{T}}^{\perp} (\mcal{T}^*)\|_F \leq \frac{12\sqrt{2}m \|\mcal{T} - \mcal{T}^* \|_F^2}{\underline{\sigma}},
$$
where $\bb{T}$ is the tangent space at the point $\mcal{T}$.
\end{lemma}
\begin{lemma}\label{lemma: delta norm}
Let $\mcal{T}, \mcal{T}^*$ are tensors of rank $(r_1, r_2, \dots, r_{m-1})$ and satisfy $Incoh(\mcal{T}) \leq \sqrt{\mu}$, and they have the left orthogonal decomposition $\mcal{T} =[T_1, T_2, \dots, T_m]$ and $\mcal{T}^* = [T_1^*, T_2^*, \dots, T_m^*]$ respectively. Denote $\Delta_i = T_i - T_i^*$, then we have
$$
\|\Delta_i\|_F\leq \left\{\begin{array}{ll}
    20\sigma_{\min}^{-1}(\mcal{T}^*) \sqrt{r_i} \kappa_0^2 \|\mcal{T} - \mcal{T}^*\|_F, & i\in [m-1], \\
    3\sqrt{r_{m-1}}\kappa_0\|\mcal{T} - \mcal{T}^*\|_F, & i=m,
\end{array}\right.
$$
and 
$$
\max_{x_i}\|\Delta_i(:, x_i, :)\|_F\leq \left\{\begin{array}{ll}
 2\sqrt{\mu r_i d_i^{-1}}, & i\in [m-1], \\
   2\sqrt{\mu r_{m-1} d_m^{-1}} \sigma_{\max}(\mcal{T}^*), & i=m.
\end{array}\right.
$$
\end{lemma}

\begin{proof}[Proof of Lemma \ref{lemma:bound of tildet}]
From \eqref{equ: new left orthogonal}, we have the decomposition as follows
$$
\wht{\mcal{T}} = \scr{W}_l^{\frac{1}{2}}\mcal{T} = \left[\wht{T}_1, \dots, \wht{T}_m\right] = \left[\wtd{T}_1 \times_{2} \mbf{G}_{l,1}^{\frac{1}{4m}}, \dots, \wtd{T}_m \times_{2} \mbf{G}_{l,m}^{\frac{1}{4m}} \right],
$$
where $\left[\wht{T}_1, \dots, \wht{T}_m\right]$ represents the left orthogonal decomposition obtained using TT-SVD (Algorithm \ref{alg: TT-SVD}). By the definition of spikiness, we have the bound on the spikiness of $\widehat{\mcal{T}}$
$$
\operatorname{Spiki}(\widehat{\mcal{T}})\leq \rho_l^{\frac{1}{2}}\operatorname{Spiki}(\mcal{T}).
$$
Furthermore, given that $\sigma_{\max}(\widehat{\mcal{T}})\leq \mu_l^{\frac{1}{2}}\sigma_{\max}(\mcal{T})$ and $\sigma_{\min}^{-1}(\wht{\mcal{T}}) \leq \nu_l^{-\frac{1}{2}} \sigma_{\min}^{-1}(\mcal{T})$,  we have 
$$
\operatorname{Incoh}(\widehat{\mcal{T}})\leq \rho_l \operatorname{Incoh}(\mcal{T}) = \rho_l \mu.
$$
Since the rank of $\widehat{\mcal{T}}$ is $(r_1, \dots, r_m)$, we can apply Lemma \ref{Lema14} to obtain the following bounds
    $$
    \max_{x_i}\|\widehat{T}_i (:, x_i, :)\|_F\leq\left\{\begin{array}{cc}
       \rho_l  \sqrt{\frac{\mu r_i}{d_i}}, & i\in [m-1], \\
         \rho_l \mu_l^{\frac{1}{2}}\sigma_{\max} (\mcal{T})\sqrt{ \frac{\mu r_{m-1}}{d_m}}, & i = m,
    \end{array}\right.
    $$
    and
    $$
    \|\widehat{T}_i\|_F \leq \left\{\begin{array}{cc}
       \sqrt{r_i},  &  i\in [m-1],\\
         \mu_l^{\frac{1}{2}}\sqrt{\underline{r}} \sigma_{\max}(\mcal{T}), &  i = m.
    \end{array}\right. 
    $$
In addition, $\widetilde{T}_i = \widehat{T}_i\times_2 \mbf{G}_{l, i}^{-\frac{1}{4m}}$, then for $i\in [m]$, we have 
$$
\|\widetilde{T}_i\|_F \leq \nu_l^{-\frac{1}{2m}} \|\widehat{T}_i\|_F \quad and \quad \max_{x_i}\|\widetilde{T}_i(:, x_i, :)\|_F\leq \nu_l^{-\frac{1}{2m}}\max_{x_i}\|\widehat{T}_i(:, x_i, :)\|_F.
$$
Combing these results, one can obtain
    $$
    \max_{x_i}\|\wtd{T}_i (:, x_i, :)\|_F\leq\left\{\begin{array}{cc}
        \nu_l^{-\frac{1}{2m}}  \rho_l \sqrt{\frac{ \mu r_i}{d_i}}, & i\in [m-1], \\
         \nu_l^{-\frac{1}{2m}} \rho_l  \mu_l^{\frac{1}{2}}\sigma_{\max} (\mcal{T})\sqrt{  \frac{\mu r_{m-1}}{d_m}}, & i = m,
    \end{array}\right.
 $$
    and

    $$
    \quad
    \|\wtd{T}_i\|_F \leq \left\{\begin{array}{cc}
       \nu_l^{-\frac{1}{2m}} \sqrt{r_i},  &  i\in [m-1],\\
         \nu_l^{-\frac{1}{2m}} \mu_l^{\frac{1}{2}}\sqrt{\underline{r}} \sigma_{\max}(\mcal{T}), &  i = m.
    \end{array}\right. 
    $$
This completes the proof.
\end{proof}

\begin{proof}[Proof of Lemma \ref{lemma: tilde delta norm}]
By \eqref{equ: new left orthogonal}, we know
    $$
    \wht{\mcal{T}} = \scr{W}_l^{\frac{1}{2}}\mcal{T} = \left[\wht{T}_1, \dots, \wht{T}_m\right] = \left[\wtd{T}_1 \times_{2} \mbf{G}_{l,1}^{\frac{1}{4m}}, \dots, \wtd{T}_m \times_{2} \mbf{G}_{l,m}^{\frac{1}{4m}} \right]
    $$
    and 
    $$
    \wht{\mcal{T}}^* := \scr{W}_l^{\frac{1}{2}}\mcal{T}^* = \left[\wht{T}_1^*, \dots, \wht{T}_m^* \right] = \left[ \wtd{T}_1^* \times_{2} \mbf{G}_{l,1}^{\frac{1}{4m}}, \dots, \wtd{T}_m^* \times_{2} \mbf{G}_{l,1}^{\frac{1}{4m}} \right],
    $$
where $ \left[\wht{T}_1, \dots, \wht{T}_m\right] $ and $\left[\wht{T}_1^*, \dots, \wht{T}_m^* \right]$ are left orthorgnal decomposition of $\mcal{T}$ and $\mcal{T}^*$, respectively, generated by TT-SVD (\cref{alg: TT-SVD}). Let $\wht{\Delta}_i = \wht{T}_i - \wht{T}_i^*$. Then,
$$
\begin{aligned}
\| \wht{\Delta}_i \|_{F} &= \| \wht{T}_i - \wht{T}_i^{*} \|_F\\
&= \| \wtd{T}_i \times \mbf{G}_{l,i}^{\frac{1}{4m}} - \wtd{T}_i^* \times \mbf{G}_{l,i}^{\frac{1}{4m}}\|_{F}\\
&=\| L(\wht{T}_i) - L(\wht{T}_i^*) \|_F.
\end{aligned}
$$
Applying Lemma \ref{lemma: delta norm}, we get
\begin{equation}\label{eq:bound-deltai}
\|\wht{\Delta}_i\|_F\leq \left\{\begin{array}{ll}
    20\sigma_{\min}^{-1}(\wht{\mcal{T}}^*) \sqrt{r_i} \hat{\kappa_0}^2 \|\wht{\mcal{T}} - \wht{\mcal{T}}^*\|_F, & i\in [m-1], \\
    3\sqrt{r_{m-1}} \hat{\kappa}_0\|\wht{\mcal{T}} - \wht{\mcal{T}}^*\|_F, & i=m,
\end{array}\right.
\end{equation}
and 
$$
\max_{x_i}\|\wht{\Delta}_i(:, x_i, :)\|_F\leq \left\{\begin{array}{ll}
 2 \rho_l \sqrt{ \mu r_i d_i^{-1}}, & i\in [m-1], \\
   2 \rho_l \sqrt{\mu r_{m-1} d_m^{-1}} \sigma_{\max}(\wht{\mcal{T}}^*), & i=m.
\end{array}\right.
$$
Since $\| \wht{\mcal{T}} - \wht{\mcal{T}}^{*} \|_F \leq \mu_l^{\frac{1}{2}} \| \mcal{T} - \mcal{T}^* \|_F$, $\sigma_{\max}(\wht{\mcal{T}}) \leq \mu_l^{\frac{1}{2}}\sigma_{\max}(\mcal{T})$, $\| \wtd{\Delta}_i\|_{F} \leq \nu_l^{-\frac{1}{2m}} \| \wht{\Delta}_i\|_{F}$ and $\sigma_{\min}^{-1}(\wht{\mcal{T}}) \leq \nu_l^{-\frac{1}{2}} \sigma_{\min}^{-1}(\mcal{T})$,
we obtain the bounds as follows
$$
\|\wtd{\Delta}_i\|_F\leq \left\{\begin{array}{ll}
    20 \rho_l^{\frac{3}{2}} \nu_l^{-\frac{1}{2m}} \sigma_{\min}^{-1}(\mcal{T}^*) \sqrt{r_i} \kappa_0^2 \|\mcal{T} - \mcal{T}^*\|_F, & i\in [m-1], \\
    3 \rho_l^{\frac{1}{2}} \nu_l^{-\frac{1}{2m}} \mu_l^{\frac{1}{2}} \kappa_0 \sqrt{r_{m-1}} \|\mcal{T} - \mcal{T}^*\|_F, & i=m,
\end{array}\right.
$$
and 
$$
\max_{x_i}\|\wtd{\Delta}_i(:, x_i, :)\|_F\leq \left\{\begin{array}{ll}
 2 \rho_l \nu_l^{-\frac{1}{2m}}\sqrt{\mu r_i d_i^{-1}}, & i\in [m-1], \\
   2 \rho_l \nu_l^{-\frac{1}{2m}} \mu_l^{\frac{1}{2}}\sqrt{ \mu r_{m-1} d_m^{-1}} \sigma_{\max}(\mcal{T}^*), & i=m.
\end{array}\right.
$$
This completes the proof.
\end{proof}

\section{Estimations of $I_1$, $I_2$ and $I_3$ in \eqref{eq:i123}}
\label{appendixB}

\subsection{Estimation of $I_1 = \|\scr{W}_l^{\frac{1}{2}}(\mcal{T} - \mcal{T}^*) - \alpha_l p \epsilon_l^{-\frac{1}{2}} \scr{P}_{\wht{\bb{T}}_l}\scr{W}_l^{\frac{1}{2}}(\mcal{T}- \mcal{T}^*)\|_F^2$}
\label{appB1}

First, we can express $I_1$ in the following form
$$
\begin{aligned}
I_1 &=\|\scr{W}_l^{\frac{1}{2}}(\mcal{T} - \mcal{T}^*) - \alpha_l p \epsilon_l^{-\frac{1}{2}} \scr{P}_{\wht{\bb{T}}_l}\scr{W}_l^{\frac{1}{2}}(\mcal{T} - \mcal{T}^*)\|_F^2\\
&=\|\scr{W}_l^{\frac{1}{2}}(\mcal{T} - \mcal{T}^*) - \eta_l \scr{P}_{\wht{\bb{T}}_l}\scr{W}_l^{\frac{1}{2}}(\mcal{T} - \mcal{T}^*)\|_F^2\\
&=(1 - \eta_l)^2 \| \scr{W}_l^{\frac{1}{2}}(\mcal{T} - \mcal{T}^*) \|_F^2 + (2- \eta_l)\eta_l \|\scr{P}_{\wht{\bb{T}}_l}^{\perp}\scr{W}_l^{\frac{1}{2}}(\mcal{T} - \mcal{T}^*)\|_F^2,
\end{aligned}
$$
where $\eta_l = \alpha_l p \epsilon_l^{-\frac{1}{2}}$. Further, using Lemma \ref{Lema27} and noting the equivalence between $\| \cdot \|_F$ and $\| \cdot \|_{\scr{W}_l}$, we can establish the following bound
\begin{equation}\label{eq:I1}
\begin{aligned}
I_1 &\leq (1-\eta_l)^2 \mu_l \| \mcal{T} - \mcal{T}^* \|_F^2 + (2- \eta_l)\eta_l \cdot \frac{400 m^2 \mu_l \| \mcal{T} - \mcal{T}^*\|_F^4}{\underline{\sigma}^2}\\
&\leq (1-\eta_l)^2 \nu_l \rho_l \| \mcal{T} - \mcal{T}^* \|_F^2 + (2- \eta_l)\eta_l \nu_l \frac{400 m^2 \rho_l}{\underline{\sigma}^2} \| \mcal{T} - \mcal{T}^*\|_F^4\\
&\leq [(1-\eta_l)^2 + (2- \eta_l)\eta_l  \frac{400}{600000}]\rho_l \nu_l \| \mcal{T} - \mcal{T}^*\|_F^2,
\end{aligned}
\end{equation}
where the last inequality holds under the condition $\| \mcal{T} - \mcal{T}^* \|_F \leq \frac{\underline{\sigma}}{600000 m}$.

\subsection{Estimation of $I_2= \left\langle \scr{P}_{\wht{\bb{T}}_l}\scr{W}_l^{\frac{1}{2}}(\mcal{T} - \mcal{T}^*), \scr{P}_{\wht{\bb{T}}_l}(\scr{W}_l^{-\frac{1}{2}} \scr{P}_{\Omega}\scr{W}_l^{-\frac{1}{2}} - p \epsilon_l^{-\frac{1}{2}} \scr{I})\scr{W}_l^{\frac{1}{2}}(\mcal{T} - \mcal{T}^*) \right\rangle$}
\label{appB2}

We begin by expressing $I_2$ in the following form
$$
\begin{aligned}
&\left\langle \scr{P}_{\wht{\bb{T}}_l}\scr{W}_l^{\frac{1}{2}}(\mcal{T} - \mcal{T}^*), \scr{P}_{\wht{\bb{T}}_l}(\scr{W}_l^{-\frac{1}{2}} \scr{P}_{\Omega}\scr{W}_l^{-\frac{1}{2}} - p \epsilon_l^{-\frac{1}{2}} \scr{I})\scr{W}_l^{\frac{1}{2}}(\mcal{T} - \mcal{T}^*) \right\rangle\\
&=\left\langle \scr{P}_{\wht{\bb{T}}_l}\scr{W}_l^{\frac{1}{2}}(\mcal{T} - \mcal{T}^*), \scr{P}_{\wht{\bb{T}}_l}(\scr{W}_l^{-\frac{1}{2}} \scr{P}_{\Omega}\scr{W}_l^{-\frac{1}{2}} - p \scr{W}_l^{-1} + p \scr{W}_l^{-1} - p \epsilon_l^{-\frac{1}{2}} \scr{I})\scr{W}_l^{\frac{1}{2}}(\mcal{T} - \mcal{T}^*) \right\rangle\\
&= \underbrace{\left\langle \scr{P}_{\wht{\bb{T}}_l}\scr{W}_l^{\frac{1}{2}}(\mcal{T} - \mcal{T}^*), \scr{P}_{\wht{\bb{T}}_l} (\scr{P}_{\Omega} - p \scr{I})\scr{W}_l^{-\frac{1}{2}} (\mcal{T} - \mcal{T}^*) \right\rangle}_{I_{21}} + 
 \underbrace{ p\left\langle \scr{P}_{\wht{\bb{T}}_l}\scr{W}_l^{\frac{1}{2}}(\mcal{T} - \mcal{T}^*), ( \scr{W}_l^{-1}  -  \epsilon_l^{-\frac{1}{2}} \scr{I})\scr{W}_l^{\frac{1}{2}}(\mcal{T} - \mcal{T}^*) \right\rangle}_{I_{22}}.
\end{aligned}
$$
This decomposition allows us to bound $I_{21}$ and $I_{22}$ separately. First, we focus on $I_{21}$. Employing the variational representation of Frobenius norm, we can express it as
$$
\begin{aligned}
&\left\langle \scr{P}_{\wht{\bb{T}}_l}\scr{W}_l^{\frac{1}{2}}(\mcal{T} - \mcal{T}^*), (\scr{P}_{\Omega} - p \scr{I})\scr{W}_l^{-\frac{1}{2}} (\mcal{T} - \mcal{T}^*) \right\rangle\\
&=\left\langle \scr{W}_l^{-\frac{1}{2}} \scr{P}_{\wht{\bb{T}}_l}\scr{W}_l^{\frac{1}{2}} (\mcal{T} - \mcal{T}^*), (\scr{P}_{\Omega} - p \scr{I}) (\mcal{T} - \mcal{T}^*)\right\rangle\\
&=\left\langle \tsp (\mcal{T} - \mcal{T}^*),   (\scr{P}_{\Omega} - p \scr{I}) (\mcal{T} - \mcal{T}^*) \right\rangle.
\end{aligned}
$$
By Lemma \ref{lemma: close of P_tilde}, we have $\wtd{\scr{P}}_{\bb{T}_l} (\mcal{T} - \mcal{T}^*) = \delta \mcal{T}_1 + \dots + \delta \mcal{T}_m$ with $\delta \mcal{T}_i = [\wtd{T}_1, \dots, \wtd{X}_i, \dots, \wtd{T}_m]$. Next, we establish an upper bound for $\|\wtd{X}_i\|_F$.  From \eqref{equ: tsp in nm},we know
$$
L(\wht{X}_i)
= \left\{\begin{array}{cc}
    (\mbf{I} - L(\widehat{T}_i) L(\widehat{T}_i)^\top ) (\widehat{T}^{\leq i-1}\otimes \mbf{I})^\top (\wht{\mcal{T}} - \wht{\mcal{T}}^*)^{\langle i\rangle} (\widehat{T}^{\geq i+1})^\top (\widehat{T}^{\geq i+1} (\widehat{T}^{\geq i+1})^\top )^{-1}, & i\in [m-1], \\
    (\widehat{T}^{\leq m-1}\otimes \mbf{I})^{\top} (\wht{\mcal{T}} - \wht{\mcal{T}}^*)^{\langle m\rangle}, & i = m.
\end{array}\right.
$$
Then, for $i\in [m-1]$ one has
\begin{equation}\label{eq:bound-Xihat}
\begin{aligned}
\|L(\widehat{X}_i)\|_F&=\|(\wht{\mcal{T}} - \wht{\mcal{T}}^*)^{\langle i\rangle} (\widehat{T}^{\geq i+1})^\top (\widehat{T}^{\geq i+1} (\widehat{T}^{\geq i+1})^\top )^{-1} \|_F,\\
&\leq \|(\wht{\mcal{T}}- \wht{\mcal{T}}^*)^{\langle i\rangle}\|_F \cdot \sigma^{-1}_{\min} (\scr{W}_l^{\frac{1}{2}} \mcal{T}),\\
&\leq \|\wht{\mcal{T}} - \wht{\mcal{T}}^* \|_F \cdot \sigma^{-1}_{\min} (\scr{W}_l^{\frac{1}{2}} \mcal{T}),\\
&\leq \sigma^{-1}_{\min} (\scr{W}_l^{\frac{1}{2}} \mcal{T})\|\wht{\mcal{T}} - \wht{\mcal{T}}^* \|_F \leq \rho_l^{\frac{1}{2}} \sigma^{-1}_{\min} (\mcal{T}) \|\mcal{T} - \mcal{T}^* \|_F\leq 2  \rho_l^{\frac{1}{2}} \sigma^{-1}_{\min} (\mcal{T}^*) \|\mcal{T} - \mcal{T}^* \|_F.
\end{aligned}
\end{equation}
For $i=m$, we have $\|L(\widehat{X}_m)\|_F = \|(\wht{\mcal{T}} - \wht{\mcal{T}}^*)^{\langle m\rangle}\|_F \leq \mu_l^{\frac{1}{2}}  \|\mcal{T} - \mcal{T}^* \|_F$. Since $\wtd{X}_i = \wht{X}_i \times_2 G_{l,i}^{-\frac{1}{4m}}$ for $i=1, \dots, m$, then $\| \wtd{X}_i\|_F \leq \nu_l^{-\frac{1}{2m}} \| \wht{X}_i \|_F$. Thus, we arrive at the following bounds
\begin{equation}\label{equ:xhatbound}
\|\wtd{X}_i \|_F \leq \left\{\begin{array}{cc}
 2 \nu_l^{-\frac{1}{2m}}\rho_l^{\frac{1}{2}} \sigma_{\min}^{-1}(\mcal{T}^*)\|\mcal{T} - \mcal{T}^* \|_F, ~~~&i\in [m-1], \\  \nu_{l}^{-\frac{1}{2m}} \mu_l^{\frac{1}{2}}\|\mcal{T} - \mcal{T}^* \|_F, ~~~&i=m.
\end{array}\right.
\end{equation}
To facilitate the subsequent analysis, we embark on a decomposition of $\mcal{T} - \mcal{T}^*$, expressing it as a sum of some tensors
\begin{equation}\label{eq:T-T*decompose}
\begin{aligned}
\mcal{T} - \mcal{T}^* & = [\wtd{T}_1^*, \wtd{T}_2^*, \dots, \wtd{T}_{i-1}^*, \wtd{\Delta}_i , \wtd{T}_{i+1}^*, \dots, \wtd{T}_m^*]\\
&\quad +  \sum_{j=1}^{i-1} ~[\wtd{T}_1^*, \wtd{T}_2^*, \dots, \wtd{T}_{j-1}^*, \wtd{\Delta}_j, \dots, \wtd{T}_{i}, \wtd{T}_{i+1}^*, \dots, \wtd{T}_m^*]\\
&\quad +  \sum_{j=i+1}^{m}  ~[\wtd{T}_1, \wtd{T}_2, \dots,\wtd{T}_{i}, \wtd{T}_{i+1}, \dots, \wtd{T}_{j-1}, \wtd{\Delta}_j, \wtd{T}_{j+1}^*, \dots, \wtd{T}_m^*]\\
& =: \wtd{\mcal{Y}}_{i, i} + \sum_{j=1}^{i-1} \wtd{\mcal{Y}}_{i,j} + \sum_{j=i+1}^{m} \wtd{\mcal{Y}}_{i,j},
 \end{aligned}
\end{equation}
where $\wtd{\Delta}_j = \wtd{T}_j - \wtd{T}_j^*$. Armed with this decomposition, we can now rewrite $\left\langle \tsp (\mcal{T} - \mcal{T}^*),   (\scr{P}_{\Omega} - p \scr{I}) (\mcal{T} - \mcal{T}^*) \right\rangle$ as the following form
$$
\begin{aligned}
\left\langle \tsp (\mcal{T} - \mcal{T}^*),   (\scr{P}_{\Omega} - p \scr{I}) (\mcal{T} - \mcal{T}^*) \right\rangle =\sum_{i=1}^m \left\langle (\scr{P}_\Omega - p \scr{I})(\mcal{T} - \mcal{T}^*), ~~ \delta \mcal{T}_i \right\rangle.
\end{aligned}
$$
Next, we estimate the individual term $\left\langle (\scr{P}_\Omega - p \scr{I})(\mcal{T}_l - \mcal{T}^*), ~~ \delta \mcal{T}_i \right\rangle.$ Leveraging the decomposition in Equation \eqref{eq:T-T*decompose}, we can write
$$
\left\langle (\scr{P}_\Omega - p \scr{I})(\mcal{T}_l - \mcal{T}^*), ~~ \delta \mcal{T}_i \right\rangle = \sum_{j=1}^{m} \left\langle (\scr{P}_\Omega - p \scr{I})\wtd{\mcal{Y}}_{i,j}, ~~ \delta \mcal{T}_i \right\rangle.
$$
Therefore, the key to our estimation lies in estimating each term  $ \left\langle (\scr{P}_\Omega - p \scr{I})\wtd{\mcal{Y}}_{i,j}, ~~ \delta \mcal{T}_i \right\rangle$. To this end, we consider our analysis into two cases.\\

{\it Case 1.} When $i \in [m-1]$, we estimate $\left\langle (\scr{P}_\Omega - p \scr{I})\wtd{\mcal{Y}}_{i,j}, ~~ \delta \mcal{T}_i \right\rangle$. 
\vspace{0.1cm}
\begin{itemize}
\item For $j \neq i,m$, we invoke Lemma \ref{lemma:bound of tildet}, Lemma  \ref{lemma: tilde delta norm} and Equation \eqref{eq:Tdecompose} to obtain the following bound:
$$
\begin{aligned}
&\left(\max_{x_k} \| \wtd{\mcal{Y}}_{i,j} (:, x_k, :)\|_F \| \delta \mcal{T}_i \|_F\right) \wedge \left(\max_{x_k}\| \delta \mcal{T}_i (:, x_k, :)\|_F\| \wtd{\mcal{Y}}_{i,j} \|_F\right) \\
&\leq \left\{\begin{array}{cc}
  \nu_l^{-\frac{1}{m}} \rho_l r_k \sqrt{\frac{\mu}{d_k}}, ~~&k \neq j, i, m,\\
  20  \nu_l^{-\frac{1}{m}} \rho_l^{\frac{5}{2}} \sigma_{\min}^{-1}(\mcal{T}^*)\kappa_0^2 r_k \sqrt{\frac{\mu }{d_k}} \| \mcal{T} - \mcal{T}^* \|_F, ~~&k=j,\\
   2  \nu_l^{-\frac{1}{m}} \rho_l^{\frac{3}{2}}  \sigma_{\min}^{-1}(\mcal{T}^*) \sqrt{\frac{\mu r_k }{d_k}} \|\mcal{T} - \mcal{T}^*\|_F, ~~&k=i,\\
 \nu_l^{-\frac{1}{m}} \mu_l \rho_l \sigma_{\max}^2(\mcal{T})  r_{m-1} \sqrt{\frac{\mu}{d_m}}, ~~&k=m.
\end{array}\right.
\end{aligned}
$$
Synthesizing these bounds using Lemma \ref{Lema13}, we arrive at the following inequality
\begin{equation}\label{I2-ineq:jneqim}
\left| \left\langle (\scr{P}_\Omega - p \scr{I})\wtd{\mcal{Y}}_{i,j}, ~~ \delta \mcal{T}_i \right\rangle \right| \leq C_m \bigg( \sqrt{\frac{n\bar{d}}{d^*}} +1 \bigg) \log^{m+2}(\bar{d}) \bigg(  \rho_l^{m+3} \kappa_0^4 \mu^{\frac{m}{2}}\frac{r^* r_{m-1}}{\sqrt{r_i d^*}}  \bigg) \| \mcal{T} - \mcal{T}^* \|_F^2.
\end{equation}
\item For $j = m$, we proceed similarly, employing Lemma \ref{lemma:bound of tildet}, Lemma  \ref{lemma: tilde delta norm} and Equation \eqref{eq:T-T*decompose} to establish the bound
$$
\begin{aligned}
&\left(\max_{x_k} \| \wtd{\mcal{Y}}_{i,j}(:, x_k, :)\|_F \| \delta \mcal{T}_i \|_F\right) \wedge \left(\max_{x_k}\| \delta \mcal{X}_i (:, x_k, :)\|_F\| \wtd{\mcal{Y}}_{i,j} \|_F\right) \\
&\leq \left\{\begin{array}{cc}
 \nu_l^{-\frac{1}{m}} \rho_l r_k  \sqrt{\frac{\mu }{d_k}}, ~~&k \neq  i, m,\\
2 \nu_l^{-\frac{1}{m}} \rho_l^{\frac{3}{2}} \sigma_{\min}^{-1}(\mcal{T}^*) \sqrt{\frac{\mu r_k }{d_k}} \| \mcal{T} - \mcal{T}^* \|_F, ~~&k =i,\\
 3 \nu_l^{-\frac{1}{m}} \rho_l^{\frac{3}{2}} \kappa_0 \sigma_{\max}(\mcal{T})  r_{m-1} \mu_l \sqrt{\frac{\mu }{d_m}} \| \mcal{T} - \mcal{T}^* \|_F, ~~&k=m.
\end{array}\right.
\end{aligned}
$$
Finally, invoking Lemma \ref{Lema13} once more, we deduce the following inequality
\begin{equation}\label{I2-ineq:j=m}
\left| \langle (\scr{P}_\Omega - p \scr{I})\wtd{\mcal{Y}}_{i,j}, ~~ \delta \mcal{T}_i \rangle \right| \leq C_m \bigg( \sqrt{\frac{n\bar{d}}{d^*}} +1 \bigg) \log^{m+2}(\bar{d}) \bigg(  \rho_l^{m+2} \kappa_0^2 \mu^{\frac{m}{2}}\frac{r^* r_{m-1}}{\sqrt{r_i d^*}}  \bigg) \| \mcal{T} - \mcal{T}^* \|_F^2.
\end{equation}
\item For $j=i$, from \eqref{eq:T-T*decompose}, we know $\wtd{\mcal{Y}}_{i,i} = [\wtd{T}_1^*, \wtd{T}_2^*, \dots, \wtd{\Delta}_i, \wtd{T}_{i+1}^*,\wtd{T}_{i+2}^*, \dots, \wtd{T}_m^*]$
and  $\delta \mcal{T}_i$ can be expressed as $\delta \mcal{T}_i =[\wtd{T}_1, \wtd{T}_2, \dots, \wtd{T}_{i-1}, \wtd{X}_i, \wtd{T}_{i+1}, \dots, \wtd{T}_m]$. Further, we decompose $\delta \mcal{T}_i$ into a sum of $m+1$ terms $\delta \mcal{T}_i = \wtd{\mcal{T}}_{i, 0} + \sum_{t=1, t\neq i}^{m} \wtd{\mcal{T}}_{i,t}$, where $\wtd{\mcal{T}}_{i,0} =[\wtd{T}_1^*, \dots, \wtd{T}_{i-1}^*, \wtd{X}_i, \wtd{T}_{i+1}^*, \dots, \wtd{T}_m^*] $ and  $\wtd{\mcal{T}}_{i,t} =[\wtd{T}_1^*, \dots, \wtd{T}_{k-1}^*, \wtd{\Delta}_t, \wtd{T}_{k+1}, \dots,$ $  \wtd{X}_i, \wtd{T}_{i+1}, \dots, \wtd{T}_{m}]$. Then we know
$$
\begin{aligned}
\langle (\scr{P}_\Omega - p \scr{I})\wtd{\mcal{Y}}_{i,i}, \delta \mcal{T}_i \rangle = \langle (\scr{P}_\Omega - p \scr{I})\wtd{\mcal{Y}}_{i,j},  \wtd{\mcal{T}}_{i, 0} \rangle + \sum_{t=1, t\neq i}^{m} \langle (\scr{P}_\Omega - p \scr{I})\wtd{\mcal{Y}}_{i,j}, \wtd{\mcal{T}}_{i, t} \rangle.
\end{aligned}
$$
Next, we estimate $\langle (\scr{P}_\Omega - p \scr{I})\wtd{\mcal{Y}}_{i,j},  \wtd{\mcal{T}}_{i, 0} \rangle$ and $\langle (\scr{P}_\Omega - p \scr{I})\wtd{\mcal{Y}}_{i,j}, \wtd{\mcal{T}}_{i, t} \rangle$ separately.

{\bf (1)} First, We begin by focusing on the first term $ \langle (\scr{P}_\Omega - p \scr{I})\wtd{\mcal{Y}}_{i,i}, \wtd{\mcal{T}}_{i,0} \rangle$.  Applying Lemma \ref{coro12}, we obtain the following inequality
\begin{equation}\label{eq:bound-xi0}
\left| \left\langle (\scr{P}_\Omega - p \scr{I})\wtd{\mcal{Y}}_{i,i}, \wtd{\mcal{T}}_{i,0} \right\rangle \right|  \leq C_m \frac{\mu \bar{r}}{d^*}\sqrt{\bar{d}n\log(\bar{d})} \| \wtd{\mcal{Y}}_{i,i} \|_F \| \wtd{\mcal{T}}_{i,0}\|_F.
\end{equation}
To further refine this bound, we require estimates of the terms $\| \wtd{\mcal{Y}}_{i, i} \|_F$ and $\| \wtd{\mcal{T}}_{i,0}\|_F$. First, applying the bound on $\| L(\wht{\Delta}_{i}) \|_F$ in \eqref{eq:bound-deltai}, we arrive at
$$
\begin{aligned}
\| \wtd{\mcal{Y}}_{i,i} \|_F &\leq \nu_l^{-\frac{1}{2}} \| \wht{\mcal{Y}}_{i, i} \|_{F}\\
 &\leq \nu_l^{-\frac{1}{2}}  \sigma_{\max}(\wht{\mcal{T}}^*)\| L(\wht{\Delta}_{i}) \|_F\\
&\leq 20 \nu_l^{-\frac{1}{2}} \sigma_{\max}(\wht{\mcal{T}}^*) \sigma_{\min}(\wht{\mcal{T}}^*) \sqrt{r_i} \hat{\kappa}_0^2 \| \wht{T} - \wht{T}^* \|_F\\
&\leq 20 \rho_{l}^{2}\kappa_0^3 \sqrt{r_i} \| \mcal{T} - \mcal{T}^* \|_F.
\end{aligned}
$$
Subsequently, we obtain
$$
\begin{aligned}
\| \wtd{\mcal{T}}_{i,0}\|_F &\leq \nu_l^{-\frac{1}{2}} \| \wht{\mcal{T}}_{i,0}\|_F\\
 &\leq \nu_l^{-\frac{1}{2}} \sigma_{\max}(\wht{\mcal{T}}^*)\| L(\wht{X}_{i}) \|_F\\
&\leq 2 \rho_l  \kappa_0 \| \mcal{T} - \mcal{T}^* \|_F,
\end{aligned}
$$
where the last inequality follows from the bound on $\| L(\wht{X}_{i}) \|_F$ in \eqref{eq:bound-Xihat}. Finally, substituting the derived bounds for $\| \wtd{\mcal{Y}}_{i,i} \|_F$ and $\| \wtd{\mcal{T}}_{i,0}\|_F$ into inequality \eqref{eq:bound-xi0}, we obtain the desired result
$$
\left| \left\langle (\scr{P}_\Omega - p \scr{I})\wtd{\mcal{Y}}_{i,i}, \wtd{\mcal{T}}_{i,0} \right\rangle \right|  \leq C_m \frac{\mu \bar{r}^{\frac{3}{2}}}{d^*}\sqrt{\bar{d}n\log(\bar{d})} \rho_{l}^{3}  \kappa_0^4  \| \mcal{T} - \mcal{T}^* \|_F^2.
$$
{\bf (2)} Next, we aim to bound the term $ \left\langle (\scr{P}_\Omega - p \scr{I})\wtd{\mcal{Y}}_{i,i}, \wtd{\mcal{T}}_{i,t} \right\rangle$.

\begin{itemize}
\item For $t \neq i,m$, we leverage Lemma \ref{lemma:bound of tildet} and Lemma  \ref{lemma: tilde delta norm} to establish an upper bound
\vspace{0.1cm}
$$
\begin{aligned}
&\left(\max_{x_k} \| \wtd{\mcal{Y}}_{i,j} (:, x_k, :)\|_F \| \delta \mcal{T}_i  \|_F\right) \wedge (\max_{x_k}\| \delta \mcal{T}_i (:, x_k, :)\|_F\| \wtd{\mcal{Y}}_{i,j} \|_F) \\
&\leq \left\{\begin{array}{cc}
 \nu_l^{-\frac{1}{m}} \rho_l r_k \sqrt{\frac{\mu}{d_k}}, ~~&k \neq t, i, m,\\
 20  \nu_l^{-\frac{1}{m}} \rho_l^{\frac{5}{2}} \sigma_{\min}^{-1}(\mcal{T}^*)\kappa_0^2 r_k \sqrt{\frac{\mu }{d_k}} \| \mcal{T} - \mcal{T}^* \|_F, ~~&k =t,\\
    \vspace{0.2cm}\\
 4 \nu_l^{-\frac{1}{m}} \rho_l^{\frac{3}{2}}  \sigma_{\min}^{-1}(\mcal{T}^*) \sqrt{\frac{\mu r_k }{d_k}} \| \mcal{T} - \mcal{T}^* \|_F, ~~&k=i,\\
 \nu_l^{-\frac{1}{m}} \mu_l \rho_l \sigma_{\max}^2(\mcal{T})   r_{m-1} \sqrt{\frac{\mu}{d_m}}, ~~&k=m.
\end{array}\right.
\end{aligned}
$$
Subsequently, by Lemma \ref{Lema13}, we arrive at the following bound
$$
\begin{aligned}
&\left| \langle (\scr{P}_\Omega - p \scr{I})\wtd{\mcal{Y}}_{i,i}, ~~ \wtd{\mcal{T}}_{i,t} \rangle \right| \leq C_m \big( \sqrt{\frac{n\bar{d}}{d^*}} +1 \big) \log^{m+2}(\bar{d}) \bigg(  \rho_l^{m +3} \kappa_0^4\mu^{\frac{m}{2}} \frac{r^* r_{m-1}}{\sqrt{r_i d^*}}  \bigg) \| \mcal{T} - \mcal{T}^* \|_F^2.
\end{aligned}
$$
\item For $t = m$,  we follow a similar procedure as in the previous case. Utilizing Lemma \ref{lemma:bound of tildet} and Lemma \ref{lemma: tilde delta norm}, we derive the following bound:
$$
\begin{aligned}
&(\max_{x_k} \| \wtd{\mcal{Y}}_{i,j} (:, x_k, :)\|_F \| \delta \mcal{T}_i \|_F) \wedge (\max_{x_k}\| \delta \mcal{T}_i (:, x_k, :)\|_F\| \wtd{\mcal{Y}}_{i,j} \|_F) \\
&\leq \left\{\begin{array}{cc}
 \nu_l^{-\frac{1}{m}} \rho_l r_k  \sqrt{\frac{\mu }{d_k}}, ~~&k \neq i, m,\\
 4 \nu_l^{-\frac{1}{m}} \rho_l^{\frac{3}{2}} \sigma_{\min}^{-1}(\mcal{T}^*) \sqrt{\frac{\mu r_k }{d_k}}\|\mcal{T} - \mcal{T}^* \|_F, ~~&k =i,\\
3 \nu_l^{-\frac{1}{m}} \rho_l^{\frac{3}{2}} \mu_l \kappa_0 r_{m-1} \sigma_{\max}(\mcal{T}) \sqrt{\frac{\mu }{d_m}} \| \mcal{T} - \mcal{T}^* \|_F, ~~&k=m.
\end{array}\right.
\end{aligned}
$$
Applying Lemma \ref{Lema13} once more yields
$$
\begin{aligned}
&\left| \left\langle (\scr{P}_\Omega - p \scr{I})\wtd{\mcal{Y}}_{i,i}, ~~ \wtd{\mcal{T}}_{i,t} \right\rangle \right| \leq C_m \bigg( \sqrt{\frac{n\bar{d}}{d^*}} +1 \bigg) \log^{m+2}(\bar{d}) \bigg( \rho_l^{m+2}   \kappa_0^2 \mu^{\frac{m}{2}} \frac{r^* r_{m-1}}{\sqrt{r_i d^*}}  \bigg) \| \mcal{T} - \mcal{T}^* \|_F^2.
\end{aligned}
$$
\end{itemize}

In summary, for $j =i$, under the conditions of event $\bm{\mcal{E}}_3$, we establish the following bound
\begin{equation}\label{I2-ineq:j=i}
\left| \left\langle (\scr{P}_\Omega - p \scr{I})\wtd{\mcal{Y}}_{i,i}, ~~ \wtd{\mcal{T}}_{i,t} \right\rangle \right| \leq C_m \bigg( \sqrt{\frac{n\bar{d}}{d^*}} +1 \bigg) \log^{m+2}(\bar{d}) \bigg(  \rho_l^{m +3} \kappa_0^4\mu^{\frac{m}{2}} \frac{r^* r_{m-1}}{\sqrt{r_i d^*}}  \bigg) \| \mcal{T} - \mcal{T}^* \|_F^2.
\end{equation}
\end{itemize}
Further, for $i \in [m-1]$, combining \eqref{I2-ineq:jneqim}, \eqref{I2-ineq:j=m} and \eqref{I2-ineq:j=i},  under the event $\bm{\mcal{E}}_3$, it yields
\begin{equation}\label{I2-ineq:ineqm}
\left| \left\langle (\scr{P}_\Omega - p \scr{I})(\mcal{T}_l - \mcal{T}^*), \delta \mcal{T}_i \right\rangle \right| \leq C_m \bigg( \sqrt{\frac{n\bar{d}}{d^*}} +1 \bigg) \log^{m+2}(\bar{d}) \bigg(  \rho_l^{m +3} \kappa_0^4\mu^{\frac{m}{2}} \frac{r^* r_{m-1}}{\sqrt{r_i d^*}}  \bigg) \| \mcal{T} - \mcal{T}^* \|_F^2.
\end{equation}

\noindent{{\it Case 2.}} When $i = m$, we aim to bound the term $\left\langle (\scr{P}_\Omega - p \scr{I})(\mcal{T}_l - \mcal{T}^*), ~~ \delta \mcal{T}_m \right\rangle$. Observing that
$$
\left\langle (\scr{P}_\Omega - p \scr{I})(\mcal{T}_l - \mcal{T}^*), ~~ \delta \mcal{T}_m \right\rangle = \sum_{j=1}^{m} \left\langle (\scr{P}_\Omega - p \scr{I})\wtd{\mcal{Y}}_{m,j}, ~~ \delta \mcal{T}_m \right\rangle,
$$
it suffices to derive an upper bound for each term $\left\langle (\scr{P}_\Omega - p \scr{I})\wtd{\mcal{Y}}_{m,j}, ~\delta \mcal{T}_m \right\rangle$.  To this end, we consider the following two cases:

\begin{itemize}
\item For $j \in [m-1] $, we begin by establishing an upper bound for the quantity $\left(\max_{x_k} \| \wtd{\mcal{Y}}_{m,j} (:, x_k, :)\|_F \| \delta \mcal{T}_m \|_F\right) \wedge \left(\max_{x_k}\| \delta \mcal{T}_m (:, x_k, :)\|_F\| \wtd{\mcal{Y}}_{m,j} \|_F\right)$. Applying Lemma \ref{lemma:bound of tildet} and Lemma \ref{lemma: tilde delta norm}, we obtain
$$
\begin{aligned}
&\left(\max_{x_k} \| \wtd{\mcal{Y}}_{m,j} (:, x_k, :)\|_F \| \delta \mcal{T}_m \|_F\right) \wedge \left(\max_{x_k}\| \delta \mcal{T}_m (:, x_k, :)\|_F\| \wtd{\mcal{Y}}_{m,j} \|_F\right)\\
&\leq \left\{\begin{array}{cc}
 \nu_l^{-\frac{1}{m}} \rho_l r_k \sqrt{\frac{\mu}{d_k}}, ~~&k \neq j, m,\\
 20  \nu_l^{-\frac{1}{m}} \rho_l^{\frac{5}{2}}  \sigma_{\min}^{-1}(\mcal{T}^*)\kappa_0^2 r_k \sqrt{\frac{\mu }{d_k}} \| \mcal{T} - \mcal{T}^* \|_F, ~~&k =j,\\
  \nu_l^{-\frac{1}{m}} \mu_l \rho_l   \sigma_{\max}(\mcal{T}) \sqrt{\frac{\mu r_{m-1} }{d_{m}}}, ~~&k=m.
\end{array}\right.
\end{aligned}
$$
Subsequently, invoking Lemma \ref{Lema13} and utilizing the derived bound, we arrive at the following inequality for $j \in [m-1]$
\begin{equation}\label{I2-ineq: j in m-1}
\left| \left\langle (\scr{P}_\Omega - p \scr{I})\wtd{\mcal{Y}}_{m,j}, ~~ \delta \mcal{T}_m \right\rangle \right| \leq  C_m \bigg( \sqrt{\frac{n\bar{d}}{d^*}} +1 \bigg) \log^{m+2}(\bar{d}) \bigg(  \rho_l^{m+\frac{5}{2}} \kappa_0^3 \mu^{\frac{m}{2}} r^* \sqrt{\frac{ r_{m-1}} {d^*}}  \bigg) \| \mcal{T} - \mcal{T}^* \|_F^2.
\end{equation}

\item For $j=m$,  recall that $\wtd{\mcal{Y}}_{m,m} = [\wtd{T}_1^*, \wtd{T}_2^*,  \dots, \wtd{T}_{m-1}^*, \wtd{\Delta}_m]$  and $\delta \mcal{T}_m =[\wtd{T}_1, \wtd{T}_2,  \dots,  \wtd{T}_{m-1}, \wtd{X}_m]$. We decompose $\delta \mcal{T}_m$ as $\delta \mcal{T}_m = \wtd{\mcal{T}}_{m, 0} + \sum_{t=1}^{m-1} \wtd{\mcal{T}}_{m,t}$, where $\wtd{\mcal{T}}_{m,0} =[\wtd{T}_1^*, \dots,  \wtd{T}_{m-1}^*, \wtd{X}_m]$ and $\wtd{\mcal{T}}_{m,t} =[\wtd{T}_1^*, \dots, \wtd{T}_{t-1}^*, \wtd{\Delta}_t,$ $\wtd{T}_{t+1}, \dots,  \wtd{T}_{m-1}, \dots, \wtd{T}_{m}]$. Then we have 
$$
\begin{aligned}
&\left\langle (\scr{P}_\Omega - p \scr{I})\wtd{\mcal{Y}}_{m, m}, \delta \mcal{T}_m \right\rangle = \left\langle (\scr{P}_\Omega - p \scr{I})\wtd{\mcal{Y}}_{m,m},  \wtd{\mcal{T}}_{m, 0} \right\rangle + \sum_{t=1}^{m-1} \left\langle (\scr{P}_\Omega - p \scr{I})\wtd{\mcal{Y}}_{m, m}, \wtd{\mcal{T}}_{m, t} \right\rangle.
\end{aligned}
$$
We proceed by analyzing each term on the right-hand side separately.

{\bf (1)} First, for $ \langle (\scr{P}_\Omega - p \scr{I})\wtd{\mcal{Y}}_{m,m}, \wtd{\mcal{T}}_{m,0} \rangle$, applying Lemma \ref{coro12}, we obtain the following bound
$$
\left| \langle (\scr{P}_\Omega - p \scr{I})\wtd{\mcal{Y}}_{m, m}, \wtd{\mcal{T}}_{m,0} \rangle \right| \leq C_m \frac{\mu \bar{r}}{d^*}\sqrt{\bar{d}n\log(\bar{d})} \| \wtd{\mcal{Y}}_{m, m} \|_F \| \wtd{\mcal{T}}_{m,0}\|_F.
$$
For $\| \wtd{\mcal{Y}}_{m, m} \|_F$ and $\| \wtd{\mcal{T}}_{m,0}\|_F$, we have
$$
\| \wtd{\mcal{Y}}_{m,m}\|_F \leq  \nu_l^{-\frac{1}{2}} \| \wht{\mcal{Y}}_{m,m}\|_F \leq \nu_l^{-\frac{1}{2}} \sigma_{\max}(\wht{\mcal{T}}^*) \| L(\wht{\Delta}_m) \|_F \leq \rho_l \sqrt{r_{m-1}} \kappa_0 \|\mcal{T} - \mcal{T}^*\|_F
$$
and
$$
\|\wtd{\mcal{T}}_{m,0}\|_F \leq \nu_l^{-\frac{1}{2}} \| \wht{\mcal{T}}_{i,0} \|_F \leq \nu_l^{-\frac{1}{2}} \| L(\wht{X}_m)\|_F \leq \rho_l \|\mcal{T} - \mcal{T}^*\|_F.
$$
Therefore, we obtain
$$
\left| \left\langle (\scr{P}_\Omega - p \scr{I})\wtd{\mcal{Y}}_{m, m}, \wtd{\mcal{T}}_{m,0} \right\rangle \right| \leq C_m \frac{\mu \bar{r}^{\frac{3}{2}}}{d^*}\sqrt{\bar{d}n\log(\bar{d})} \rho_l^{2} \kappa_0 \| \mcal{T} - \mcal{T}^* \|_F^2.
$$

{\bf (2)} Next, we estimate $\left\langle (\scr{P}_\Omega - p \scr{I})\wtd{\mcal{Y}}_{m,m}, \wtd{\mcal{T}}_{m,t} \right\rangle$ with $t \in [m-1]$. Similar to the previous process, we first establish an upper bound for $\max_{x_k} \| \wtd{\mcal{Y}}_{m,m} (:, x_k, :)\|_F \| \wtd{\mcal{T}}_{m,t} \|_F) \wedge (\max_{x_k}\|  \wtd{\mcal{T}}_{m, t}(:, x_k, :)\|_F\| \wtd{\mcal{Y}}_{m,m}\|_F)$ using Lemma \ref{lemma:bound of tildet} and Lemma \ref{lemma: tilde delta norm}.
$$
\begin{aligned}
&\left(\max_{x_k} \| \wtd{\mcal{Y}}_{m,m}(:, x_k, :)\|_F \| \wtd{\mcal{T}}_{m,t} \|_F\right) \wedge \left(\max_{x_k}\|  \wtd{\mcal{T}}_{m, t} (:, x_k, :)\|_F\| \wtd{\mcal{Y}}_{m,m} \|_F\right)\\
&\leq \left\{\begin{array}{cc}
 \nu_l^{-\frac{1}{m}} \rho_l r_k \sqrt{\frac{\mu}{d_k}}, ~~&k \neq t, m,\\
20  \nu_l^{-\frac{1}{m}} \rho_l^{\frac{5}{2}} \sigma_{\min}^{-1}(\mcal{T}^*)\kappa_0^2 r_k \sqrt{\frac{\mu }{d_k}} \| \mcal{T} - \mcal{T}^* \|_F, ~~&k =t,\\
 \nu_l^{-\frac{1}{m}} \mu_l \rho_l \sigma_{\max} (\mcal{T}^*)  \sqrt{\frac{\mu r_{m-1}}{d_m}} | \mcal{T} - \mcal{T}^* \|_F, ~~&k=m.
\end{array}\right.
\end{aligned}
$$
Finally, applying Lemma \ref{Lema13} and incorporating the above bound, we obtain
$$
\left| \left\langle (\scr{P}_\Omega - p \scr{I})\wtd{\mcal{Y}}_{m, m},  \wtd{\mcal{T}}_{m,t} \right\rangle \right| \leq  C_m \bigg( \sqrt{\frac{n\bar{d}}{d^*}} +1 \bigg) \log^{m+2}(\bar{d}) \bigg(   \rho_l^{m+ \frac{5}{2}} \kappa_0^3\mu^{\frac{m}{2}} r^* \sqrt{\frac{ r_{m-1}}{ d^*}}  \bigg) \| \mcal{T} - \mcal{T}^* \|_F^2.
$$
When $j=m$, by calculations and estimations, we derive the following inequality
\begin{equation}\label{I2-ineq: j eq m}
\left| \left\langle (\scr{P}_\Omega - p \scr{I})\wtd{\mcal{Y}}_{m, m},  \delta \mcal{T}_m \right \rangle \right| \leq  C_m \bigg( \sqrt{\frac{n\bar{d}}{d^*}} +1 \bigg) \log^{m+2}(\bar{d}) \bigg(   \rho_l^{m+ \frac{5}{2}} \kappa_0^3\mu^{\frac{m}{2}} r^* \sqrt{\frac{ r_{m-1}}{ d^*}}  \bigg) \| \mcal{T} - \mcal{T}^* \|_F^2.
\end{equation}
\end{itemize}
Next,  for $i =m$, by leveraging \eqref{I2-ineq: j in m-1}, \eqref{I2-ineq: j eq m} and under the event $\bm{\mcal{E}}_3$, we reach the following result
\begin{equation}\label{I2-ineq:i=m}
\left| \left\langle (\scr{P}_\Omega - p \scr{I})(\mcal{T} - \mcal{T}^*), ~~ \delta \mcal{T}_m \right\rangle \right| \leq C_m \big( \sqrt{\frac{n\bar{d}}{d^*}} +1 \big) \log^{m+2}(\bar{d}) \bigg(  \rho_l^{m+ \frac{5}{2}} \kappa_0^3\mu^{\frac{m}{2}} r^* \sqrt{\frac{ r_{m-1}}{ d^*}}  \bigg) \| \mcal{T} - \mcal{T}^* \|_F^2.
\end{equation}

Finally, combining \eqref{I2-ineq:ineqm} and \eqref{I2-ineq:i=m}, we obtain
\begin{equation}\label{ineq:i21}
\begin{aligned}
I_{21}&\leq \left|\left\langle \tsp (\mcal{T}_l - \mcal{T}^*),   (\scr{P}_{\Omega} - p \scr{I}) (\mcal{T}_l - \mcal{T}^*) \right\rangle\right|\\
&\leq C_m \big( \sqrt{\frac{n\bar{d}}{d^*}} +1 \big) \log^{m+2}(\bar{d}) \bigg(  \rho_l^{m +3} \kappa_0^4\mu^{\frac{m}{2}} \frac{r^* r_{m-1} \sum_{i=1}^{m-1} r_i^{-\frac{1}{2}}}{\sqrt{d^*}}  \bigg) \| \mcal{T} - \mcal{T}^* \|_F^2.
\end{aligned}
\end{equation}
Furthermore, shifting our focus to the term $I_{22}$, we have
\begin{equation}\label{ineq:i22}
\begin{aligned}
I_{22} &= p\left\langle \scr{P}_{\wht{\bb{T}}_l}\scr{W}_l^{\frac{1}{2}}(\mcal{T}_l - \mcal{T}^*), ( \scr{W}_l^{-1}  -  \epsilon_l^{-\frac{1}{2}} \scr{I})\scr{W}_l^{\frac{1}{2}}(\mcal{T}_l - \mcal{T}^*) \right\rangle\\
&\leq p \| \scr{W}_l^{-1} - \epsilon_l^{-\frac{1}{2}} \scr{I} \| \| \scr{W}_l^{\frac{1}{2}} (\mcal{T}_l - \mcal{T}^*) \|_F^2\\
&\leq  p\frac{\|\mcal{G}_l\|_\vee^2}{2 \epsilon_l^{\frac{3}{2}}}\|\scr{W}_l^{\frac{1}{2}}(\mcal{T}_l - \mcal{T}^*)\|_F^2,
\end{aligned}
\end{equation}
where the last inequality holds because
\begin{equation}\label{ineq:W-minus-I}
\begin{aligned}
0&\leq \|\scr{W}_l^{-1} - \epsilon_l^{-\frac{1}{2}}\scr{I}\|_2 \leq \epsilon_l^{-\frac{1}{2}}\left(1 - \left(1+\frac{\|\mcal{G}_l\|_\vee^2}{\epsilon_l}\right)^{-\frac{1}{2}}\right) \leq \frac{\|\mcal{G}_l\|_\vee^2}{2\epsilon^{\frac{3}{2}}}.
\end{aligned}
\end{equation}
Further, taking $\epsilon_l = \|\mcal{G}_l\|_\vee^2$ and combining \eqref{ineq:i21} and \eqref{ineq:i22}, as long as
$$
n \geq C_m \rho_l^{2m+6} \bar{d} \log^{2m+4}(\bar{d}) \kappa_0^8 \mu^{m} (r^*)^2 r_{m-1}^2 \sum_{i=1}^{m-1} r_i^{-1} + C_m \rho_l^{m+3} \sqrt{d^*} \log^{m+2}(\bar{d}) \kappa_0^4 \mu^{\frac{m}{2}} r^* r_{m-1} \sum_{i=1}^{m-1} r_i^{-\frac{1}{2}},
$$
we have
\begin{equation}\label{eq:I2}
I_2= \left\langle \scr{P}_{\wht{\bb{T}}_l}\scr{W}_l^{\frac{1}{2}}(\mcal{T} - \mcal{T}^*), \scr{P}_{\wht{\bb{T}}_l}(\scr{W}_l^{-\frac{1}{2}} \scr{P}_{\Omega}\scr{W}_l^{-\frac{1}{2}} - p \epsilon_l^{-\frac{1}{2}} \scr{I})\scr{W}_l^{\frac{1}{2}}(\mcal{T} - \mcal{T}^*) \right\rangle \leq 0.71p \|\mcal{T}_l - \mcal{T}^*\|_F^2.
\end{equation}

\subsection{Estimation of $I_3= \| \scr{P}_{\wht{\bb{T}}_l}(\scr{W}_l^{-\frac{1}{2}} \scr{P}_{\Omega}\scr{W}_l^{-\frac{1}{2}} - p \epsilon_l^{-\frac{1}{2}} \scr{I})\scr{W}_l^{\frac{1}{2}}(\mcal{T}_l - \mcal{T}^*) \|_F^2$}
\label{appB3}

Finally, we estimate $\| \scr{P}_{\wht{\bb{T}}_l}(\scr{W}_l^{-\frac{1}{2}} \scr{P}_{\Omega}\scr{W}_l^{-\frac{1}{2}} - p \epsilon_l^{-\frac{1}{2}} \scr{I})\scr{W}_l^{\frac{1}{2}}(\mcal{T}_l - \mcal{T}^*) \|_F^2$.  We start by reformulating it as
$$
\begin{aligned}
&\| \scr{P}_{\wht{\bb{T}}_l}(\scr{W}_l^{-\frac{1}{2}} \scr{P}_{\Omega}\scr{W}_l^{-\frac{1}{2}} - p \epsilon_l^{-\frac{1}{2}} \scr{I})\scr{W}_l^{\frac{1}{2}}(\mcal{T}_l - \mcal{T}^*) \|_F^2\\
&= \|\scr{P}_{\wht{\bb{T}}_l} (\scr{W}_l^{-\frac{1}{2}} \scr{P}_{\Omega}\scr{W}_l^{-\frac{1}{2}} - p \scr{W}_l^{-1})\scr{W}_l^{\frac{1}{2}}(\mcal{T}_l - \mcal{T}^*) + p\scr{P}_{\wht{\bb{T}}_l} (\scr{W}_l^{-1} - \epsilon_l^{-\frac{1}{2}} \scr{I})\scr{W}_l^{\frac{1}{2}}(\mcal{T}_l - \mcal{T}^*) \|_F^2\\
&\leq \underbrace{1001 \|\scr{P}_{\wht{\bb{T}}_l} (\scr{P}_{\Omega} - p \scr{I})\scr{W}_l^{-\frac{1}{2}}(\mcal{T}_l - \mcal{T}^*)\|_F^2 }_{I_{31}} + \underbrace{ 1.001 p^2  \| \scr{P}_{\wht{\bb{T}}_l} (\scr{W}_l^{-1} - \epsilon_l^{-\frac{1}{2}} \scr{I}) \scr{W}_l^{\frac{1}{2}}(\mcal{T}_l - \mcal{T}^*)\|_F^2}_{I_{32}}.
\end{aligned}
$$
This decomposition lets us to estimate $I_{31}$ and $I_{32}$ separately.  Starting with $I_{31}$, we utilize the variational characterization of the Frobenius norm to express it as
$$
\begin{aligned}
&\left\|\scr{P}_{\wht{\bb{T}}_l} (\scr{P}_{\Omega} - p \scr{I})\scr{W}_l^{-\frac{1}{2}}(\mcal{T}_l - \mcal{T}^*)\right\|_F\\
&= \left\langle (\scr{P}_{\Omega} - p \scr{I})\scr{W}_l^{-\frac{1}{2}}(\mcal{T}_l - \mcal{T}^*), \scr{P}_{\wht{\bb{T}}_l} \mcal{X}_0 \right\rangle\\
&= \left\langle (\scr{P}_{\Omega} - p \scr{I})(\mcal{T}_l - \mcal{T}^*), \scr{W}_l^{-\frac{1}{2}}\scr{P}_{\wht{\bb{T}}_l} \mcal{X}_0 \right\rangle,
\end{aligned}
$$
for some $\mcal{X}_0$ with $\|\mcal{X}_0\|_F\leq 1$. From \eqref{eq:Aihat}, it follows that $\scr{P}_{\wht{\bb{T}}_l} \mcal{X}_0 = \wht{\delta \mcal{X}}_1 + \dots + \wht{\delta \mcal{X}}_m$, where $\wht{\delta \mcal{X}}_i = [\wht{T}_1, \dots, \wht{X}_i, \dots, \wht{T}_m]$.  This allows us to write
$$
\scr{W}_l^{-\frac{1}{2}}  \scr{P}_{\wht{\bb{T}}_l} \mcal{X}_0 =  \scr{W}_l^{-\frac{1}{2}} \wht{\delta \mcal{X}}_1 + \dots + \scr{W}_l^{-\frac{1}{2}} \wht{\delta \mcal{X}}_m,
$$ 
where $ \scr{W}_l^{-\frac{1}{2}} \wht{\delta \mcal{X}}_i =  \left[\wht{T}_1 \times_2 \mbf{G}_{t, j}^{-\frac{1}{4m}} , \dots, \wht{X}_i \times_2 \mbf{G}_{t, j}^{-\frac{1}{4m}} , \dots, \wht{T}_m \times_2 \mbf{G}_{t, j}^{-\frac{1}{4m}} \right]  = \left[\widetilde{T}_1, \dots, X_i, \dots, \widetilde{T}_m \right] := \delta \mcal{X}_i$. We then see that $\scr{W}_l^{-\frac{1}{2}}  \scr{P}_{\wht{\bb{T}}_l} \mcal{X}_0 = \delta \mcal{X}_1 + \delta \mcal{X}_2 + \dots + \delta \mcal{X}_m$. Next, we derive an upper bound for $\|X_i\|_F$. From \eqref{equ: tsp in nm}, it becomes clear that
$$
L(\widehat{X}_i) = \left\{\begin{array}{cc}
    (\mbf{I} - L(\widehat{T}_i) L(\widehat{T}_i)^\top ) (\widehat{T}^{\leq i-1}\otimes \mbf{I})^\top \mcal{X}_0^{\langle i\rangle} (\widehat{T}^{\geq i+1})^\top (\widehat{T}^{\geq i+1} (\widehat{T}^{\geq i+1})^\top )^{-1}, & i\in [m-1], \\
    (\widehat{T}^{\leq m-1}\otimes \mbf{I})^{\top} \mcal{X}_0^{\langle m\rangle}, & i = m.
\end{array}\right.
$$
Then, for $i\in [m-1]$ we find
$$
\begin{aligned}
\|L(\widehat{X}_i)\|_F&=\|\mcal{X}_0^{\langle i\rangle}(\widehat{T}^{\geq i+1})^\top (\widehat{T}^{\geq i+1} (\widehat{T}^{\geq i+1})^\top )^{-1} \|_F,\\
&\leq \|\mcal{X}_0^{\langle i\rangle}\|_F \cdot \sigma^{-1}_{\min} (\scr{W}_l^{\frac{1}{2}} \mcal{T}),\\
&\leq \|\mcal{X}_0\|_F \cdot \sigma^{-1}_{\min} (\scr{W}_l^{\frac{1}{2}} \mcal{T}),\\
&\leq \sigma^{-1}_{\min} (\scr{W}_l^{\frac{1}{2}} \mcal{T})\leq \nu_l^{-\frac{1}{2}} \sigma^{-1}_{\min} (\mcal{T})\leq 2  \nu_l^{-\frac{1}{2}} \sigma^{-1}_{\min} (\mcal{T}^*).
\end{aligned}
$$
For $i=m$, we have $\|L(\widehat{X}_m)\|_F = \|\mcal{X}_0^{\langle m\rangle}\|_F = \| \mcal{X}_0 \|_F \leq 1$. Hence, the bounds can be summarized as
\begin{equation}\label{equ:xhatbound}
\|X_i \|_F = \|\widehat{X}_i\times_2 \mbf{G}_{l, i}^{-\frac{1}{4 m}}\|_F\leq 2 \nu_l^{-\frac{1}{2m}}\nu_l^{-\frac{1}{2}} \sigma_{\min}^{-1}(\mcal{T}^*),~i\in [m-1], ~\textmd{and}~\|X_m\|_F\leq \nu_{l}^{-\frac{1}{2m}}.
\end{equation}
To support the analysis that follows, using a decomposition of $\mcal{T} - \mcal{T}^*$ \eqref{eq:T-T*decompose}, we can now express $\langle (\scr{P}_\Omega - p \scr{I})(\mcal{T}_l - \mcal{T}^*), ~~\scr{W}_l^{-\frac{1}{2}}  \scr{P}_{\wht{\bb{T}}_l} \mcal{X}_0  \rangle$ as
$$
\begin{aligned}
\langle (\scr{P}_\Omega - p \scr{I})(\mcal{T}_l - \mcal{T}^*), ~~\scr{W}_l^{-\frac{1}{2}}  \scr{P}_{\wht{\bb{T}}_l} \mcal{X}_0 \rangle&=\sum_{i=1}^m \langle (\scr{P}_\Omega - p \scr{I})(\mcal{T}_l - \mcal{T}^*), ~~ \delta \mcal{X}_i \rangle\\
&= \sum_{i=1}^m \sum_{j=1}^{m} \langle (\scr{P}_\Omega - p \scr{I})\wtd{\mcal{Y}}_{i,j}, ~~ \delta \mcal{X}_i \rangle.
\end{aligned}
$$
Thus, our estimation hinges on evaluating each term $ \langle (\scr{P}_\Omega - p \scr{I})\wtd{\mcal{Y}}_{i,j}, ~~ \delta \mcal{X}_i \rangle$. Similar to the proof of  $I_21$ in the section \ref{appB2}, we categorize our analysis into two cases.\\

\noindent {\it Case 1.} When $i \in [m-1]$, using Lemma \ref{lemma:bound of tildet}, Lemma  \ref{lemma: tilde delta norm}  Lemma \ref{Lema13} and Lemma \ref{coro12}, under the event $\bm{\mcal{E}}_3$, it yields
\begin{equation}\label{ineq:ineqm}
\left| \langle (\scr{P}_\Omega - p \scr{I})(\mcal{T}_l - \mcal{T}^*), \delta \mcal{X}_i \rangle \right| \leq C_m \big( \sqrt{\frac{n\bar{d}}{d^*}} +1 \big) \log^{m+2}(\bar{d}) \bigg(  \rho_l^{m +\frac{5}{2}} \nu_l^{-\frac{1}{2}} \kappa_0^4\mu^{\frac{m}{2}} \frac{r^* r_{m-1}}{\sqrt{r_i d^*}}  \bigg) \| \mcal{T} - \mcal{T}^* \|_F.
\end{equation}

\noindent{\it Case 2.} When $i = m$, we aim to bound the term $\langle (\scr{P}_\Omega - p \scr{I})(\mcal{T}_l - \mcal{T}^*), ~~ \delta \mcal{X}_m\rangle$. Using Lemma \ref{lemma:bound of tildet}, Lemma  \ref{lemma: tilde delta norm}  Lemma \ref{Lema13} and Lemma \ref{coro12}, under the event $\bm{\mcal{E}}_3$, we reach the following result
\begin{equation}\label{ineq:i=m}
\left| \langle (\scr{P}_\Omega - p \scr{I})(\mcal{T}_l - \mcal{T}^*), ~~ \delta \mcal{X}_m\rangle \right| \leq C_m \big( \sqrt{\frac{n\bar{d}}{d^*}} +1 \big) \log^{m+2}(\bar{d}) \bigg(  \nu_l^{\frac{-1}{2}} \rho_l^{m+ 2} \kappa_0^3\mu^{\frac{m}{2}} r^* \sqrt{\frac{ r_{m-1}}{ d^*}}  \bigg) \| \mcal{T} - \mcal{T}^* \|_F.
\end{equation}
Finally, combining \eqref{ineq:ineqm} and \eqref{ineq:i=m}, we obtain
\begin{equation}\label{ineq:i31}
\begin{aligned}
&\|\scr{W}_l^{\frac{1}{2}}\tsp (\scr{P}_\Omega - p \scr{I})\scr{W}_l^{-1} (\mcal{T} - \mcal{T}^*) \|_F\\
&\leq C_m \big( \sqrt{\frac{n\bar{d}}{d^*}} +1 \big) \log^{m+2}(\bar{d}) \bigg(  \rho_l^{m +\frac{5}{2}} \nu_l^{-\frac{1}{2}} \kappa_0^4\mu^{\frac{m}{2}} \frac{r^* r_{m-1} \sum_{i=1}^{m-1} r_i^{-\frac{1}{2}}}{\sqrt{d^*}}  \bigg) \| \mcal{T} - \mcal{T}^* \|_F.
\end{aligned}
\end{equation}
Furthermore, shifting our focus to the term $I_{32}$, we first have
$$
\begin{aligned}
& \| \scr{P}_{\wht{\bb{T}}_l} (\scr{W}_l^{-1} - \epsilon_l^{-\frac{1}{2}} \scr{I}) \scr{W}_l^{\frac{1}{2}}(\mcal{T}_l - \mcal{T}^*)\|_F\\
&\leq  \|\scr{W}_l^{-1} - \epsilon_l^{-\frac{1}{2}} \scr{I}\| \|\scr{W}_l^{\frac{1}{2}}(\mcal{T}_l - \mcal{T}^*) \|_F \leq  \frac{\|\mcal{G}_l\|_\vee^2}{2\epsilon_l^{\frac{3}{2}}} \mu_l^{\frac{1}{2}} \|\mcal{T}_l - \mcal{T}^* \|_F\\
&\leq  \rho_l^{\frac{1}{2}} \frac{\|\mcal{G}_l\|_\vee^2}{2\epsilon_l^{\frac{5}{4}}} \|\mcal{T}_l - \mcal{T}^* \|_F \leq  \frac{\rho_l^{\frac{1}{2}}}{2} \epsilon_l^{-\frac{1}{4}}\|\mcal{T}_l - \mcal{T}^* \|_F,\\
\end{aligned}
$$
thus, we obtain
\begin{equation}\label{ineq:i32}
\begin{aligned}
I_{32}&=1.001 p^2 \| \scr{P}_{\wht{\bb{T}}_l} (\scr{W}_l^{-1} - \epsilon_l^{-\frac{1}{2}} \scr{I}) \scr{W}_l^{\frac{1}{2}}(\mcal{T}_l - \mcal{T}^*)\|_F^2 \leq 1.001 p^2 \frac{\rho_l}{4} \nu_l^{-1} \|\mcal{T}_l - \mcal{T}^* \|_F^2.
\end{aligned}
\end{equation}
Combining \eqref{ineq:i31} and \eqref{ineq:i32}, as long as
$$
n \geq C_m \rho_l^{2m+5} \bar{d} \log^{2m+4}(\bar{d}) \kappa_0^8 \mu^{m} (r^*)^2 r_{m-1}^2 \sum_{i=1}^{m-1} r_i^{-1} + C_m \rho_l^{m+\frac{5}{2}} \sqrt{d^*} \log^{m+2}(\bar{d}) \kappa_0^4 \mu^{\frac{m}{2}} r^* r_{m-1} \sum_{i=1}^{m-1} r_i^{-\frac{1}{2}},
$$
 we have
 \begin{equation}\label{eq:I3}
I_3 =  \| \scr{P}_{\wht{\bb{T}}_l}(\scr{W}_l^{-\frac{1}{2}} \scr{P}_{\Omega}\scr{W}_l^{-\frac{1}{2}} - p \epsilon_l^{-\frac{1}{2}} \scr{I})\scr{W}_l^{\frac{1}{2}}(\mcal{T}_l - \mcal{T}^*) \|_F^2 \leq 0.354 p^2 \nu_l^{-1} \|\mcal{T}_l - \mcal{T}^* \|_F^2.
 \end{equation}

 \section*{Acknowledgements}
The first named author would like to thank Dr. Jingyang Li and Dr. Junyi Fan for their discussions on the preconditioned Riemannian gradient descent algorithm for low-rank matrix completion, which inspires the decomposition form (4.2) in the proof of Theorem 4.1.
 
\begin{small}
\bibliographystyle{siam}
\bibliography{references}
\end{small}

\end{document}